\documentclass[11pt]{article}

\usepackage[utf8]{inputenc} 
\usepackage[T1]{fontenc}    
\usepackage[margin=1in]{geometry}

\usepackage[hidelinks]{hyperref}       
\usepackage{url}            
\usepackage{booktabs}       
\usepackage{amsfonts}       
\usepackage{nicefrac}       
\usepackage{microtype}      
\usepackage{xcolor}         
\usepackage[shortlabels]{enumitem}
\usepackage{float}
\usepackage{wrapfig}
\usepackage{listings}

\usepackage[numbers]{natbib}

\usepackage{amsmath}
\usepackage{amssymb}
\usepackage{amsthm}
\usepackage{mathtools}
\usepackage{graphicx}
\usepackage{subcaption}
\usepackage{tikz, bbm}

\usepackage{microtype}
\usepackage{graphicx}
\usepackage{dirtytalk}
\usepackage{subfiles}
\usepackage{natbib}
\usepackage{enumitem}
\usepackage{xspace}
\usepackage{forloop}
\usepackage{multirow}
\usepackage{algorithm,algorithmic,refcount}
\usepackage{amsfonts}
\usepackage{comment}

\usepackage{algorithmic}
\newtheorem{theorem}{Theorem}[section]

\newtheorem{lemma}[theorem]{Lemma}
\newtheorem{corollary}[theorem]{Corollary}
\newtheorem{assumption}[theorem]{Assumption}

\newtheorem{remark}[theorem]{Remark}

\newtheorem{proposition}[theorem]{Proposition}

\DeclareMathOperator{\w}{\mathbf{w}}

\DeclareMathOperator{\E}{\mathbb{E}}

\let\Pr\relax\DeclareMathOperator{\Pr}{\mathbb{P}}

\DeclareMathOperator*{\argmin}{arg\,min}

\DeclareMathOperator{\erf}{erf}
\DeclareMathOperator{\sign}{sign}

\DeclareMathOperator{\diag}{diag}
\DeclareMathOperator{\poly}{poly}

\newcommand*\circled[1]{\tikz[baseline=(char.base)]{
    \node[shape=circle,draw,inner sep=1pt] (char) {#1};}}

\allowdisplaybreaks

\title{Provable Multi-Task Representation Learning by Two-Layer ReLU Neural Networks}

\author{Liam Collins\thanks{Chandra Family Department of Electrical and Computer Engineering, 
The University of Texas at Austin, Austin, TX,  USA.\qquad\{liamc@utexas.edu, mokhtari@austin.utexas.edu, sanjay.shakkottai@utexas.edu\}.}, \quad Hamed Hassani\thanks{Department of
Electrical and Systems Engineering, University of Pennsylvania, Philadelphia, PA, USA.\qquad\qquad\{hassani@seas.upenn.edu\}} , \quad Mahdi Soltanolkotabi\thanks{Ming Hsieh Department of
Electrical and Computer Engineering, University of Southern California, Los Angeles, CA, USA.\qquad\{soltanol@usc.edu\}},\\ Aryan Mokhtari$^*$,\quad Sanjay Shakkottai$^*$}

\date{}

\usepackage{setspace}

\begin{document}

\maketitle







\begin{abstract}
An increasingly popular machine learning paradigm is to pretrain a neural network (NN) on many tasks offline, then adapt it to downstream tasks, often by re-training only the last linear layer of the network. This approach yields strong downstream performance in a variety of contexts, demonstrating that multitask pretraining 
 leads to effective feature learning.
Although several recent theoretical studies have shown that 
shallow NNs learn meaningful features when  either (i) they are trained on a {\em single} task or (ii) they are {\em linear},
very little is known about the closer-to-practice case of {\em nonlinear} NNs trained on {\em multiple} tasks. 
In this work, we present the first results proving that feature learning occurs during training with a nonlinear model on multiple tasks. Our key insight is that multi-task pretraining induces a pseudo-contrastive loss that favors representations that align points that typically have the same label across tasks. Using this observation, we show that when the tasks are binary classification tasks with labels depending on the projection of the data onto an $r$-dimensional subspace within the $d\gg r$-dimensional input space, a simple gradient-based multitask learning algorithm on a two-layer ReLU NN recovers this projection, allowing for generalization to downstream tasks with sample and neuron complexity independent of $d$.
In contrast, we show that with high probability over the draw of a single task, training on this single task cannot guarantee to learn all $r$ ground-truth features.
\end{abstract}

\newpage








\section{Introduction} 
\label{sec:intro}


Recent empirical results have demonstrated huge successes in pretraining large neural networks (NNs) on many tasks with gradient-based algorithms \citep{
crawshaw2020multi,zhang2021survey,
wang2023multitask}.
These works suggest that the quality of the pretrained representation improves with the number of pretraining tasks, yet this phenomenon remains not well understood from a theoretical standpoint. 
Specifically, the natural questions of {\em why} nonlinear NNs learn effective feature representations when pretrained on multiple tasks with gradient-based methods and {\em how} the number of pretraining tasks affect the downstream performance of these representations remain largely unanswered.


Significant progress has been made in theoretically understanding the dynamics of NNs trained with gradient-based methods in recent years, especially in regards to proving that shallow NNs can learn meaningful features when trained with gradient descent and its variants \cite{damian2022neural,shi2022theoretical,abbe2020poly,abbe2022merged,allen2019learning,bai2019beyond,li2020learning,daniely2020learning,barak2022hidden,telgarsky2022feature,akiyama2022excess,zhou2021local}. 
However, these results are  limited to single-task settings, so they cannot explain the improvements in model performance seen by pretraining on many tasks. While a few studies show the representation learning benefits of multi-task pretraining with gradient-based algorithms \cite{argyriou2006multi,thekumparampil2021sample,collins2022maml,collins2022fedavg,saunshi2021representation,sun2021towards,chua2021fine,bullins2019generalize,chen2022active}, these analyses study only linear models; it is not clear whether they can generalize to even simple non-linear NNs. 

In this work, we aim to bridge this gap by 
analyzing the training dynamics of  a two-layer ReLU network pretrained with a generic gradient-based multi-task learning algorithm on  many binary classification tasks.  Following the aforementioned line of work, 
we suppose the existence of a ground-truth low-dimensional subspace that for all tasks preserves all information in the input data relevant to its label. 
We ask whether a variant of gradient descent applied to this  multi-task setting can learn a representation that projects the input data onto the ground-truth subspace. Learning such a representation entails successful pretraining, since it reduces the complexity of solving a downstream task to that of solving a classification problem in the low-dimensional space,  
rather than the potentially very high-dimensional input data  space.
Figure \ref{fig:multi} shows that gradient-based  multi-task learning with a two-layer ReLU NN with first-layer parameters (the representation) shared among all tasks and last-layer weights (the head) learned uniquely for each task indeed recovers the ground-truth subspace with error diminishing with the number of tasks. We theoretically justify this observation, providing the first known proofs of multi-task feature learning with a nonlinear model along with a new explanation for why multi-tasking aids feature learning.
 Our theoretical contributions are summarized below, and verified numerically  in Appendix \ref{sec:sims}.
\begin{itemize}[leftmargin=3mm]

\item \textbf{Proof of multi-task representation learning with two-layer ReLU network.} We consider binary classification tasks whose labels depend on only $r$ features of the input, where $r$ is much smaller than the ambient dimension $d$, and a large class of task distributions that includes, e.g., a uniform distribution over sparse parity tasks.
We prove that multi-task pretraining with a gradient-based learning algorithm on $T$ tasks drawn from such a distribution 
leads the first-layer ReLU weights to approximately project onto the ground-truth $r$-dimensional feature space, with error diminishing with $T$ and the number of samples per task $n$ as roughly $2^r\frac{\sqrt{d}}{\sqrt{T}}(1 + \frac{\sqrt{d}}{\sqrt{n}})$ (see Proposition \ref{prop:1} and Theorem \ref{thm:1}).  The key to this result is showing   that updating task-specific heads prior to the representation induces a {\em pseudo-contrastive loss function of the representation}, which encourages learning the ground-truth features  to align points likely to share a label on a randomly drawn task (see Section \ref{sec:sketch}).



\item \textbf{Generalization guarantees.} We show that we can add a random ReLU layer on top of the pretrained representation, then train a linear layer on top of this random layer with finite samples, to solve {\em any} downstream task with binary labels that are a function of the $r$ important features. 
Crucially, we prove that  the sample and neuron complexity of solving the downstream task are independent of the ambient dimension $d$ (see Theorem \ref{thm:downstream}).

\item \textbf{Negative results.} We confirm the necessity of multi-task pretraining by proving that using a random features model (no pretraining) or pretraining on only a {\em single}, randomly-selected task with high probability require neuron or sample complexity scaling polynomially in $d$ for solving a downstream task (see Theorems \ref{thm:lower} and \ref{thm:single}).


\end{itemize}

\begin{figure}
 \centering
  \includegraphics[width=0.56\linewidth]{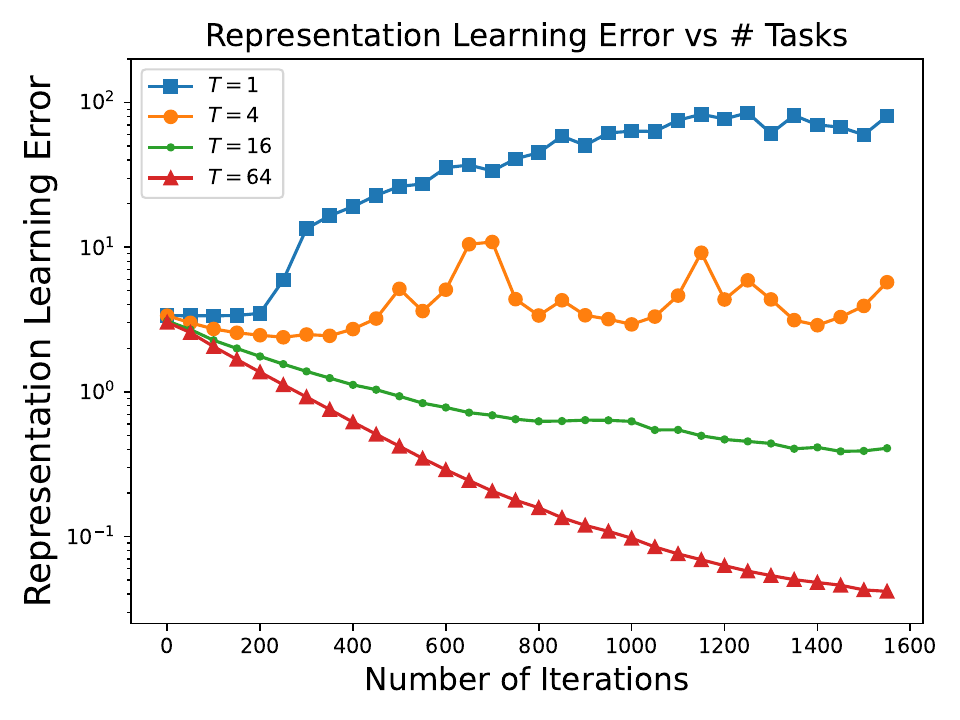}
 \caption{\textbf{Representation learning error vs  training iterations with varying numbers of tasks $T$.} 
 Here we sample tasks from the uniform distribution over sparse parity tasks on $r$ binary coordinates determined by $\sign(\mathbf{Mx})$ for some row-orthogonal matrix $\mathbf{M}$ and $d$-dimensional input $\mathbf{x}$. Here $d=32$, $r=3$,
 and the learning model  is a two-layer, $m$-neuron ReLU network with first-layer weights $\mathbf{W}\!\in\! \mathbb{R}^{m\times d}$ (the representation). All cases use the same total number of training samples, i.e. the number of samples/task is inversely proportional to the number of training tasks $T$, meaning all cases use the same total number of samples across tasks. Still, as $T$ increases, the row space of  $\mathbf{W}$ approaches that of $\mathbf{M}$ (smaller representation learning error).  Please see Appendix \ref{sec:sims} for  details.}\label{fig:multi}
\end{figure}

\textbf{Notations.} We uppercase boldface to denote matrices, lowercase boldface to denote vectors, and standard typeface to denote scalars.
We employ $\text{Unif}(\mathcal{S})$ to denote the uniform distribution over the set $\mathcal{S}$.
We  denote the zero vector in $\mathbb{R}^d$ as $\mathbf{0}_d$, the  identity matrix in $\mathbb{R}^{d\times d}$ as $\mathbf{I}_d$, the standard multivariate normal distribution over $\mathbb{R}^d$ as $\mathcal{N}(\mathbf{0}_d,\mathbf{I}_d)$, and the Rademacher hypercube in $\mathbb{R}^d$ as $\mathcal{H}^d \coloneqq \{-1,1\}^d$. We denote the space of $r \times d$ matrices with orthonormal rows as $\mathbb{O}^{r\times d}$, and use $\chi\{A\}$ as the indicator function for the event $A$. We denote the set $\{1,\dots,r\}$ as $[r]$. We use ${\Omega}(\cdot),{\Theta}(\cdot)$ and ${O}(\cdot)$  in the standard fashion, and   $\tilde{\Omega}(\cdot),\tilde{\Theta}(\cdot)$ and $\tilde{O}(\cdot)$ to denote scalings up to logarithmic factors.


\subsection{Related Work} \label{sec:related}




\textbf{Single-task learning with neural networks.} 
A plethora of works have studied the behavior of gradient-based algorithms for optimizing NNs on single tasks in recent years. 
Many of these studies consider the neural tangent kernel (NTK) regime 
\citep{jacot2018neural,arora2019exact,du2019gradient,allen2019convergence,oymak2020toward,ji2019polylogarithmic,li2018learning,du2020fewshot,zou2018stochastic,lee2019wide,chizat2019lazy},
in which a large initialization and small step size mean that early layer model weights barely change during training, so the algorithm dynamics reduce to those of linear regression on fixed features. 
We are interested in the feature learning regime of training neural networks, wherein the representation weights change significantly and the dynamics are nonlinear. 

Numerous works have studied  feature learning in NNs, but the vast majority consider optimizing only a single task from a particular class of functions \cite{allen2019learning,bai2019beyond,li2020learning,daniely2020learning,barak2022hidden,telgarsky2022feature,akiyama2022excess,zhou2021local,abbe2022merged,ba2022high,mousavi2022neural,shi2022theoretical,damian2022neural,damian2023smoothing,wang2023learning,abbe2023sgd,dandi2023two}. 
Among studies most similar to ours, \citet{abbe2022merged,abbe2023sgd,wang2023learning,dandi2023two} showed that gradient-based algorithms can learn  hierarchical features when training on single polynomial tasks.
\citet{damian2022neural} proved that a gradient-based method on a single $r$-index polynomial regression task with two-layer ReLU network can learn all $r$ relevant indices as long as this single task satisfies a Hessian lower bound assumption, and \citet{nichani2023provable} extended this line of  work to three-layer networks. 
\citet{shi2022theoretical} showed that two-layer ReLU networks with activation noise can learn functions of the sum of $r$ inputs. 
Additional works consider single-task feature learning  in the mean-field regime with infinitely-wide networks \cite{chizat2018global,mei2018mean,sirignano2020mean,nguyen2019mean}.

\textbf{Multitask feature learning.} Several works have studied whether  multitask learning algorithms recover expressive low-dimensional data representations shared across tasks, but only consider linear models \cite{argyriou2006multi,thekumparampil2021sample,collins2022maml,collins2021exploiting,collins2022fedavg,saunshi2021representation,sun2021towards,chua2021fine,bullins2019generalize,chen2022active,yuksel2023model}.
\citet{kao2021maml}  noticed a similar phenomenon as this work in that adapting task-specific heads induces a contrastive loss, but in the context of a particular meta-learning algorithm, and they did not provide feature learning results.
 Further studies including \cite{maurer2016benefit, tripuraneni2020provable, du2020fewshot, tripuraneni2020theory, xu2021representation} have provided statistical bounds on the downstream task loss for multitask-pretrained representations. However, these representations are learned
 by exactly solving an empirical risk minimization problem on the pretraining tasks, not by executing a gradient-based algorithm.









\section{Formulation} \label{sec:formulation}


In this section, we formally define the multi-task learning problem and the algorithms analyzed in Section \ref{sec:results}. Our motivation is drawn from classification problems where the input data for all tasks share a common representation, i.e. a small set of features that determine their labels. However, the mapping from these features to labels varies across tasks. Ultimately, the goal of the multi-task learner is to leverage the large set of pretraining tasks to learn a representation that captures the small set of label-relevant features, thereby enabling strong performance on downstream tasks.

\subsection{Pretraining tasks and data generating model} 

We consider pretraining on a set of $T$ binary classification tasks, each drawn independently from a distribution $\mathcal{T}$ over tasks. All these tasks share a common characteristic: their labeling function depends solely on a projection of the input features onto a low-dimensional subspace. Specifically, each task $i$ consists of a distribution $\mathcal{D}_i$ on $\mathcal{X}\times \mathcal{Y}$, where the input space is $\mathcal{X}=\mathbb{R}^d$ and the label space is $\mathcal{Y}= \{-1,1\}$.
Samples are drawn from $\mathcal{D}_i$ by first selecting a Gaussian feature vector $\mathbf{x} \in \mathcal{X}$, then computing its label $f_i(\mathbf{x}) \in \mathcal{Y}$ as follows:
\begin{equation} \label{eq1}
    \mathbf{x}\sim \mathcal{N}(\mathbf{0}_d,\mathbf{I}_d);  \;\;  f_i(\mathbf{x}) = {g}_i(\sign(\mathbf{M}\mathbf{x})) \in \{-1,1\}. 
\end{equation}
Here, $\mathbf{M} \coloneqq [\mathbf{m}_1,\dots, \mathbf{m}_r]^\top \in \mathbb{R}^{r \times d}$ is a matrix with orthonormal rows that captures the $r$ label-relevant features $(\sign(\mathbf{m}_1^\top\mathbf{x}),\dots, \sign(\mathbf{m}_r^\top\mathbf{x}))$ in $\mathbf{x}$, and  ${g}_i: \mathcal{H}^{r} \rightarrow \{-1,1\}$ is the ground-truth link function for task $i$ that maps vertices on the $r$-dimensional Rademacher hypercube to binary labels.  To model shared information across the tasks, we assume that $r\ll d$. Thus, the complexity of solving a new task can be drastically reduced by learning an appropriate low-dimensional projection onto the row space of $\mathbf{M}$. 
The question we ask here is whether gradient-based multi-task pretraining can  efficiently recover the $r$ label-relevant features expressed by $\mathbf{M}$.

This setting is similar to the sparse coding model studied by \citet{shi2022theoretical}, except  we assume the input data is continuous in $\mathbb{R}^d$, while  \citet{shi2022theoretical} assume the input data is on the hypercube $\mathcal{H}^d$. By assuming that the labelling function is a function of ``sign'' of the $r$ ground-truth features, we make our model more similar to the one in \cite{shi2022theoretical}, as they assume the label for each task is a function of the first $r$ coordinates of the integer input data. 




Critically, the set of label-relevant coordinates of the input data are shared among all tasks, while the link functions $g_i$ mapping from these coordinates to labels are specific to each task. Thus, the goal of the multi-task learner is to recover the shared label-relevant features, so that it may solve a downstream task with complexity scaling only with the number $r$ of label-relevant features, rather than the much larger ambient dimension of the data $d$.
This model draws inspiration from classification tasks in which labels are functions of the presence (or lack thereof) of a small number of features in the data. For example, whether or not a brain MRI reveals cancerous tissue depends on the presence of tumor-shaped structures in the image, indicated by a small number of features relative to the MRI dimension. 




More formally, the generative model in \eqref{eq1} implies that, for any task, the sample labels are a function of the projection of the data onto the $r$ ground-truth features in $\mathbf{M}$. We use $\mathcal{T}$ to denote the distribution over link functions $g_i:\mathcal{H}^r\rightarrow \{-1,1\}$, and $\mathcal{T}(\mathbf{M})$ to denote the distribution over functions mapping from $\mathbb{R}^d\rightarrow \{-1,1\}$, i.e. the distribution over $f_i$, where $f_i(\mathbf{x})= g_i(\sign(\mathbf{Mx}))$ and $g_i \sim \mathcal{T}$. 

However, in order to recover {\em all} the $r$ ground-truth features,  it is not sufficient that the labels simply depend on the $r$ ground-truth features -- they must depend on {\em all} of them in aggregate across tasks. For example, if the labels for all tasks can be written as functions of only the projection of the inputs onto the first $r-1$ rows of $\mathbf{M}$, there is no hope to recover the $r$-th row of $\mathbf{M}$. Thus, the tasks must be ``diverse'' in the sense that in aggregate they depend on {\em all} ground-truth features. 

To formalize this idea, we 
make the following assumption on the distribution of task link functions $\mathcal{T}$. 
Our condition entails 
that for any pair of points with different sign patterns on their $r$ label-relevant features, it is equally likely for them to have the same label as it is for them to have different labels on a task link function drawn from $\mathcal{T}$. 
\begin{assumption}\label{assump:diversity}
For any two points $\mathbf{z},\mathbf{z}'\in \mathcal{H}^r$ such that $\mathbf{z}\neq \mathbf{z}'$,
the probability that the labels of $\mathbf{z}$ and $\mathbf{z}'$ are the same for a task link function drawn from ${\mathcal{T}}$  satisfies: 
\begin{align}
       \Pr_{i \sim {\mathcal{T}}} \left[ {g}_i(\mathbf{z}) = {g}_i( \mathbf{z}')  \right]& 
    = \tfrac{1}{2}. \label{onehalff}
\end{align}
\end{assumption}
Assumption \ref{assump:diversity} is necessary to  ensure the task link functions depend equally on all inputs. To see this, suppose that all pairs of inputs $\mathbf{z},\mathbf{z}'$ with identical first $r-1$ coordinates but differing $r$-th coordinates had the same labels, i.e. $\Pr_{i \sim {\mathcal{T}}} \left[ {g}_i(\mathbf{z}) = {g}_i( \mathbf{z}')  \right]
    = 1. $ Then, all link functions in the support of $\mathcal{T}$ would in fact only depend on the first $r-1$ input coordinates, rather than all $r$ inputs. So, we need $\Pr_{i } \left[ {g}_i(\mathbf{z}) = {g}_i( \mathbf{z}')  \right] <1$ for some dependence on all inputs. We make a stronger  assumption of $\frac{1}{2}$ in the RHS of \eqref{onehalff} that  ensures perfectly  balanced  dependence  across all inputs. We note that our results do not strictly require perfect balance, rather it is useful for ease of exposition\footnote{Our results hold for {\em finitely-many} tasks drawn for training, so the empirical distribution of tasks does not assign equal importance to each of the $r$ input features, in the sense that Assumption \ref{assump:diversity} does not hold exactly on the empirical task distribution. This implies that our results can extend to cases in which the population distribution of tasks $\mathcal{T}$ does not satisfy Assumption \ref{assump:diversity} exactly, i.e. we can tolerate $\Pr_{i \sim {\mathcal{T}}} \left[ {g}_i(\sign(\mathbf{Mx})) = {g}_i( \sign(\mathbf{Mx}'))  \right] \leq \frac{1}{2} + \epsilon$ for some small $\epsilon>0$.}.
    
    Another interpretation of Assumption \ref{assump:diversity} is that it enforces that the correlation of the labels of $\mathbf{z}$ and $\mathbf{z}'$ across tasks, i.e. $\E_{i \sim {\mathcal{T}}} \left[ {g}_i(\mathbf{z}) {g}_i( \mathbf{z}')  \right]$, is 1 if $\mathbf{z}=\mathbf{z}'$, and 0 otherwise. In other words, the label correlation of $\mathbf{x}$ and $\mathbf{x}'$ across tasks is 1 if the ground-truth features of $\mathbf{x}$ and $\mathbf{x}'$ are the same, and 0 otherwise. 
    We will show in  Section \ref{sec:sketch} intuitively why the correlation of the labels  of $\mathbf{x}$ and $\mathbf{x}'$  need only be ``roughly''
 increasing with the similarity of the ground-truth features in   $\mathbf{x}$ and $\mathbf{x}'$, a very natural condition,  for gradient-based multi-task training to recover $\text{row}(\mathbf{M})$.
For now, we describe two examples of task distributions that satisfy Assumption \ref{assump:diversity}.

    
    \textbf{Example 1: Uniform distribution over all tasks.} Here we have $\mathcal{T}=\mathcal{T}_{\text{all}}$, where 
    \begin{align}
        \mathcal{T}_{\text{all}} \coloneqq \text{Unif}(\{ g_i: \mathcal{H}^r \rightarrow \{-1,1\} \}),
    \end{align}
    i.e.  $\mathcal{T}_{\text{all}}$ is the uniform distribution over all possible mappings from the $r$-dimensional $\pm 1$ hypercube to $\{-1,1\}.$

    \textbf{Example 2: Uniform distribution over all sparse parity tasks.} {\em Sparse parity} tasks are a well-studied class of tasks in which the label is the parity of a subset of the number of $-1$'s among a particular subset of input bits \cite{kearns1998efficient}. In this case we have $\mathcal{T}= \mathcal{T}_{\text{s.p.}}$, where $ \mathcal{T}_{\text{s.p.}}$ is the uniform distribution over parity functions on $r$ input bits, formally defined as follows:
    \begin{align}
        \mathcal{T}_{\text{s.p.}} \coloneqq &\;\text{Unif}(\{ g_i: g_i(\mathbf{z}) = (-1)^{\sum_{j \in \mathcal{S}_i} \chi\{{z}_j = -1\}},
        \; \mathcal{S}_i \subseteq [r],\; \forall \; \mathbf{z}\in \mathcal{H}^r  \}).\label{sparsepar}
    \end{align}
    Both $\mathcal{T}_{\text{all}}$ and $\mathcal{T}_{\text{s.p.}}$ effectively assign equal importance to all $r$ inputs, so they naturally satisfy Assumption \ref{assump:diversity} (please see Appendix \ref{app:sec:dists} for proofs). 

\subsection{Learning model and loss}  We consider multi-task pretraining of a two-layer neural network $\hat{y}(\cdot) = \hat{y}(\cdot; \mathbf{W},\mathbf{b},\mathbf{a}):\mathbb{R}^d\rightarrow \mathbb{R}$ with $m$ ReLU  neurons in the hidden layer, namely
\begin{align}
  \hat{y}(\mathbf{x}) = \hat{y}(\mathbf{x}; \mathbf{W},\mathbf{b},\mathbf{a}) := \sum_{j=1}^m a_j \sigma (\mathbf{w}_j^\top \mathbf{x} + b_j)
\end{align}
where $\sigma(\mathbf{x})=\max(\mathbf{x},\mathbf{0})$ element-wise, $\mathbf{w}_j \in \mathbb{R}^d$ and $b_j \in \mathbb{R}$ are the weight vector and bias for the $j$-th neuron, respectively, and $a_j\in \mathbb{R}$ is the last-layer weight for the $j$-th neuron. We let $\mathbf{W} = [\w_1,\dots,\w_m]\in \mathbb{R}^{d \times m}$ denote the matrix of concatenated weight vectors, $\mathbf{b}=[b_1,\dots,b_m]\in \mathbb{R}^m$ denote the vector of biases, and $\mathbf{a}= [a_1,\dots,a_m] \in \mathbb{R}^m$ denote the vector of last-layer weights, which we call the head. 
We use the hinge loss to measure the accuracy of the predictions of this model:
\begin{align} 
    \ell(\hat{y}(\mathbf{x}), f_i(\mathbf{x})):= \max\big(1 - f_i(\mathbf{x}) \hat{y}(\mathbf{x}), 0\big), \nonumber 
\end{align}
and for each task $i$, we define 
\begin{align}
     \mathcal{L}_i(\mathbf{W},\mathbf{b},\mathbf{a}) &:=  \mathbb{E}_{(\mathbf{x},f_i(\mathbf{x}))\sim \mathcal{D}_i}\left[ \ell(
     \hat{y}(\mathbf{x}), 
     f_i(\mathbf{x})) \right]  \nonumber \\
     \hat{\mathcal{L}}_i(\mathbf{W},\mathbf{b},\mathbf{a}; \hat{\mathcal{D}}_i) &:= \frac{1}{|\hat{\mathcal{D}}_i|}\sum_{(\mathbf{x},f_i(\mathbf{x}))\in \hat{\mathcal{D}}_i}  \ell(
     \hat{y}(\mathbf{x}), 
     f_i(\mathbf{x})) \nonumber 
\end{align}
as the population loss on $\mathcal{D}_i$ and empirical loss on a finite dataset $\hat{\mathcal{D}}_i$ drawn from $\mathcal{D}_i$, respectively.
Ultimately, the goal of multi-task pretraining is to learn a first-layer representation that generalizes to downstream tasks, in the sense that we can easily train a new classifier
on top of the first layer in order to achieve small task-specific loss.
To this end, we consider optimizing the following multi-task objective:
\begin{align}&\min_{\mathbf{W},\mathbf{b}, \{\mathbf{a}_1, \dots, \mathbf{a}_T\}}
 {\mathcal{L}}(\mathbf{W},\mathbf{b},\{\mathbf{a}_i\}_{i=1}^T ) := \frac{1}{T}\sum_{i=1}^T  {\mathcal{L}}_i(\mathbf{W},\mathbf{b},\mathbf{a}_i)  +   \frac{\lambda_{\mathbf{a}}}{2} \|\mathbf{a}_i\|_2^2 +  \frac{\lambda_{\mathbf{w}}}{2} \|\mathbf{W}\|_F^2  \label{obj_glob_i},
\end{align}
where $\lambda_{\mathbf{a}}$ and $\lambda_{\w}$ are regularization parameters. 
Optimizing $\mathcal{L}$ entails learning task-specific heads on top of a shared representation, a widely used and empirically successful approach to multi-task learning \citep{zhang2021survey,crawshaw2020multi,ruder2017overview}. By optimizing the above problem, we hope to find a $\mathbf{W}$ that projects input data onto the row space of $\mathbf{M}$ and thus captures all $r$ label-relevant features, while disregarding all other spurious features.
However, we cannot access $\mathcal{L}$ directly, and instead must approximate it via stochastic queries of finite samples from each $\mathcal{D}_i$.
So,  we will use the gradient of $\hat{\mathcal{L}}_i$ instead of $\mathcal{L}$ to update the variables, 
as we discuss next. 

\subsection{Algorithm} \label{sec:alg}

We consider a two-stage learning process: (1) Representation learning, in which we aim to learn effective features using $T$ available tasks, and (2) Downstream evaluation, in which we encounter a new task and aim to efficiently learn an accurate classifier on the pre-trained features. 



\textbf{Representation learning phase.} 
The multi-task  learning algorithm we consider aims to solve the global objective \eqref{obj_glob_i} with task-specific heads. 
We denote  the $j$-th neuron weights at initialization as $ \mathbf{w}_{j}^0  \in \mathbb{R}^d$,  and  
 the global bias and 
head corresponding to task $i$ at time $0$ as $\mathbf{b}^0\in \mathbb{R}^m$ and $\mathbf{a}_i^0\in \mathbb{R}^m$, respectively. We initialize these parameters as: 
\begin{align}
    \mathbf{w}_{j}^{0} \sim \mathcal{N}(\mathbf{0}_d,\nu_{\mathbf{w}}^2\mathbf{I}_d),  \quad 
    \mathbf{a}_i^0 =\mathbf{0}_m,
    \quad \mathbf{b}^0=\mathbf{0}_m
\end{align}
where $\nu_{\mathbf{w}} \in \mathbb{R}_{\geq 0}$.
After initialization, we execute an alternating gradient descent-based algorithm.
We first optimize the heads, i.e., $\mathbf{a}_{1},\dots, \mathbf{a}_{T}$, with one step of stochastic gradient descent (SGD) on the corresponding task-specific empirical loss on a batch of samples $\hat{\mathcal{D}}_{i,\mathbf{a}}$ for each task $i$:
\begin{equation}
    \mathbf{a}_{i}^1 =  (1 - \eta \lambda_{\mathbf{a}})\mathbf{a}_{i}^0 -  \eta  \nabla_{\mathbf{a}} \hat{\mathcal{L}}_i(\mathbf{W}^0,\mathbf{b}^0,\mathbf{a}_i^{0}; \hat{\mathcal{D}}_{i,\mathbf{a}})
    \quad \forall \; i \in [T]. \nonumber
\end{equation}
The same number of samples is used for each task, denoted by $n_1\coloneqq |\hat{\mathcal{D}}_{i,\mathbf{a}}|$.
Next, we update the model weights $\mathbf{W}$ with one step of SGD  on the global empirical loss induced by the updated heads, with a fresh batch of samples $\hat{\mathcal{D}}_{i, \mathbf{W}}$ for each task $i$: 
\begin{align}
    \mathbf{W}^1 &= \mathbf{W}^{0} -   \frac{\eta}{T} \sum_{i=1}^T \nabla_{\mathbf{W}}   \hat{\mathcal{L}}_i(\mathbf{W}^0, \mathbf{b}^0,\mathbf{a}_{i}^1; \hat{\mathcal{D}}_{i,\mathbf{W}})  \nonumber
\end{align}
Again all tasks use the same number of samples, denoted by $n_2\coloneqq |\hat{\mathcal{D}}_{i,\mathbf{W}}|$.
In Theorem \ref{thm:1}, we show that this single iteration of alternating stochastic gradient descent with respect to $\{\mathbf{a}_1,\dots,\mathbf{a}_T\}$ and $\mathbf{W}$ is sufficient to learn meaningful features. Notably, it is standard practice in the feature learning theory literature to consider only one gradient descent  step for the first layer weights
\citep{daniely2020learning,abbe2022merged,barak2022hidden,damian2022neural,ba2022high}. 
{We later show empirically that in our multi-task setting, it is necessary to first optimize the heads before updating the first-layer weights in order to recover the ground-truth features.} 
In any case, following \citet{damian2022neural}, we do not update the biases during pretraining. 
Next we describe how we leverage the pre-trained weights $\mathbf{W}^1$ for learning a downstream task.

\textbf{Downstream evaluation phase.}  After the representation learning phase, we consider learning a prediction function to fit a downstream task that may have {\em any} link function  on the $r$ ground-truth features, i.e. any function in the support of $\mathcal{T}_{\text{all}}$. Since we consider such a wide range of possible downstream tasks, we need to increase the model complexity to allow for solving such tasks.
Thus, we use prediction functions  with two hidden layers with first layer weights determined by the output of the representation learning phase, and second hidden layer parameters set randomly.
This random second layer is necessary  to linearly separate the classes induced by {\em any} binary function on the $r$ coordinates with high probability, without having to use a very wide first layer;
please see Remark \ref{rem:2} for more details.

In other words, the first hidden layer has $m$ neurons and the weights are a scaled version of $ \mathbf{W}^1$ denoted by $\alpha \mathbf{W}^1$, and the bias term is $ \mathbf{b}$. Note that here $\alpha>0$ is a re-scaling factor (see  Appendix \ref{app:down} for more details). The second hidden layer of the classifier has $\hat{m}$ neurons with weights denoted by $\hat{\mathbf{W}} \coloneqq [\hat{\mathbf{w}}_1,\dots,\hat{\mathbf{w}}_{\hat{m}}]^\top\in \mathbb
{R}^{\hat{m}\times m}$ and bias by $\hat{\mathbf{b}}\in  \mathbb{R}^{\hat{m}}$. Hence, the embedding of these two layers for input $\mathbf{x}$, which we denote by $ \phi(\mathbf{x})\in \mathbb{R}^{\hat{m}}$, is given by 
\begin{align}
    \phi(\mathbf{x}) =  \sigma\left( \hat{\mathbf{W}} \  \sigma\left(\alpha{}\mathbf{W}^1\mathbf{x} + \mathbf{b}\right) + \hat{\mathbf{b}}\right) \label{F}
\end{align}
Again note that $\mathbf{W}^1$ is fixed from the previous phase; it remains to set  ${\mathbf{b}}$, $ \hat{\mathbf{W}}$, and $\hat{\mathbf{b}}$ to create an effective embedding for the downstream task. We do this by sampling $\hat{\mathbf{W}}$ and $(\mathbf{b},\hat{\mathbf{b}})$ from mean-zero Gaussian and uniform distributions, respectively, with variances that depend only on $m$; 
see Appendix \ref{app:down} for more details.


Next, given a dataset $\hat{\mathcal{D}}_{T+1} \coloneqq \{(\mathbf{x}_l,f_{T+1}(\mathbf{x}_l))\}_{\ell=1}^N$ of $N$ i.i.d. samples from a distribution $\mathcal{D}_{T+1}$ corresponding to a downstream task, we learn a task-specific head $\mathbf{a}$ and a bias term $\tau$ by solving the following problem:
\begin{align}
    \min_{ \mathbf{a}\in \mathbb{R}^{\hat{m}}, \tau \in \mathbb{R}} \ \frac{1}{N}\sum_{l=1}^N \ell(\mathbf{a}^\top \phi(\mathbf{x}_l )\!+\!\tau\! , f_{T+1}(\mathbf{x}_l)  ) + \frac{\hat{\lambda}_{\mathbf{a}}}{2}\|\mathbf{a}\|_2^2.  \nonumber
\end{align}
We use the resulting head, i.e.,  $\mathbf{a}_{T+1}$, and bias term, i.e. $\tau_{T+1}$, to define the  prediction function for task $T+1$ as
\begin{align}
&F(\mathbf{x}; \mathbf{a}_{T+1}, \tau_{T+1}, \mathbf{W}^1,\mathbf{b},\hat{\mathbf{W}},\hat{\mathbf{b}} )
\ :=\ \mathbf{a}_{T+1}^\top \sigma\left( \hat{\mathbf{W}} \  \sigma\left(\alpha{}\mathbf{W}^1\mathbf{x} + \mathbf{b}\right) + \hat{\mathbf{b}}\right) +\tau_{T+1}. \nonumber
\end{align}
For ease of notation, we denote the above function by $F(\mathbf{x})$. We evaluate the performance of the prediction function on the task population loss:
\begin{align}
    \mathcal{L}_{T+1}^{\text{eval}}(F) \coloneqq \mathbb{E}_{(\mathbf{x},f_{T+1}(\mathbf{x})) \sim \mathcal{D}_{T+1}  } [ \ell(F(\mathbf{x} ), f_{T+1}(\mathbf{x})  ) ]. \nonumber
\end{align}
Note that $\mathcal{L}_{T+1}^{\text{eval}}$ is a random function of $\mathbf{b},\hat{\mathbf{W}},\hat{\mathbf{b}}$, and $\hat{\mathcal{D}}_{T+1}$ in addition to the randomness from pretraining. We  upper bound $\mathcal{L}_{T+1}^{\text{eval}}$ with high probability 
in Theorem \ref{thm:downstream}.

\section{Theoretical Results} \label{sec:results}

\textbf{Feature learning guarantees.}
We start by showing that the gradient-based multi-task learning algorithm described in the previous section recovers the ground-truth features.
To do this, we first need the following proposition, which shows that the projection of the initial features $\mathbf{W}_0$ onto the subspace spanned by the label-relevant, or ground-truth, features stays roughly the same after one step, while their projection onto the subspace spanned by the spurious features becomes very small. 

Here, we let $\Pi_{\parallel}(\mathbf{W})\coloneqq \mathbf{W}\mathbf{M}^\top\mathbf{M}$ denote the projection of the rows of the matrix $\mathbf{W}$ onto the ground-truth subspace, and $\Pi_{\perp}(\mathbf{W})\coloneqq \mathbf{W}\mathbf{M}_{\perp}^\top \mathbf{M}_{\perp}$ denote the projection onto the spurious subspace. 
For brevity, we abbreviate the statements of the theoretical results in this section and defer the full versions, along with their proofs, to the Appendix.
\begin{proposition} \label{prop:1}
    Consider the gradient-based multi-task algorithm described in Section~\ref{sec:alg} that uses $T$ tasks and $(n_1,n_2)$ samples per task to update the (head, representation), respectively,  and suppose Assumption \ref{assump:diversity} holds. Further assume\footnote{The $m=O(d)$ condition in Proposition \ref{prop:1} and Theorem \ref{thm:1} is purely for ease of presentation; please see Lemma \ref{lem:final} for a complete statement of the errors for arbitrary $m$.}  
    $m= O(d)$
and define \\
$\epsilon \coloneqq {O}\left(
\frac{d\; {\log(dTn_2/\delta)} }{\sqrt{Tn_2}} \left( 1+ \frac{\sqrt{\log(T/\delta)}}{\sqrt{n_1}} \right) + {\frac{\sqrt{dr} \log(dm/\delta) }{\sqrt{T}}} 
\right)$ for $\delta<1$ and $\delta = \Omega(e^{-d})$.
Then there is a setting of the parameters $\eta,   \lambda_{\mathbf{w}}$ and $\nu_{\mathbf{w}}$
such that  
with probability at least $1-\delta$,
\begin{align}
    &1. \quad  \tfrac{1}{\nu_{\mathbf{w}}\sqrt{m}}\|\Pi_{\parallel}(\mathbf{W}^1) -  \tfrac{1}{2^{r+1}\pi} 
\Pi_{\parallel}(\mathbf{W}^{0}) \|_2  =  O\left(  
 \epsilon +\tfrac{r^4 + \log^4(m/\delta)}{2^r d} \right), \nonumber \\
&2. \quad \tfrac{1}{\nu_{\mathbf{w}}\sqrt{m}}\|\Pi_{\perp}(\mathbf{W}^{1}) \|_2  = O\left(
\epsilon +\tfrac{r^{3.5} + \log^{3.5}(m/\delta)}{2^r d^{1.5}} 
\right). \nonumber
\end{align} 

\end{proposition}
Proposition \ref{prop:1} shows that with high probability (w.h.p.) over the random initialization, the weights learned by the gradient-based multi-task learning algorithm satisfy two properties, for sufficiently large $T, n_1,$ and $ n_2$:
(1) the projection of these weights onto the ground-truth subspace is close to a slightly scaled down (by a factor of $2^{-r}$) version of their projection at initialization, and (2) their projection onto the spurious subspace is negligible. These two observations, combined with the fact that the neuron weights have independent standard Gaussian initializations, imply that  the projection of the neuron weights onto the ground-truth subspace dominates their projection onto the spurious subspace. We formalize this observation below.

\begin{theorem}[Representation Learning]\label{thm:1}
    Consider the setting in Proposition \ref{prop:1}  with $d =\tilde{\Omega}( r^4 )$, $m = O(d)$, and $\epsilon$ defined the same way.
    Further suppose $m = \tilde{\Omega}(r)$, $ T = \tilde{\Omega}( 2^{2r} d r ) $ and $Tn_2 = \tilde{\Omega}(2^{2r} d^2)$. Let $\sigma_r(\mathbf{B})$ denote the $r$-th singular value of the matrix $\mathbf{B}$.
Then with probability at least $1-\delta$,
\begin{align}  
    \frac{\sigma_1( \Pi_{\perp}(\mathbf{W}^1)) }{\sigma_{r}( \Pi_{\parallel}(\mathbf{W}^1) )} = O\left( \tfrac{ {r^{3.5}} + \log^{3.5}(m/\delta) }{ d^{1.5} } +  2^r \epsilon \right). \nonumber
\end{align}
\end{theorem}
Theorem \ref{thm:1} characterizes the representation learned by multi-task pretraining in an intuitive manner.  For $d \gg r^4$, $ T \gg 2^{2r} d $ and $Tn_2 \gg 2^{2r} d^2$, we have ${\sigma_{r}( \Pi_{\parallel}(\mathbf{W}^1) )} \gg \sigma_1( \Pi_{\perp}(\mathbf{W}^1) )$, meaning that most of the energy in each neuron weight is in the column space of the ground-truth subspace. This is equivalent to saying that applying $\mathbf{W}^1$ to an input $\mathbf{x}$ essentially projects it onto the ground-truth subspace spanned by the row space of $\mathbf{M}$, as desired.

\textbf{Downstream performance.} Now that we have shown that the learned representation recovers the ground-truth subspace, we use this result to show that the representation generalizes to downstream tasks. We consider tasks with input data sharing the same label-relevant $r$ features as the pretraining tasks, but here the input data is discrete. In particular, each $\mathbf{v} \in \mathbb{R}^d$ is generated as:
\begin{align}
    \mathbf{v} = \mathbf{M}^\top \mathbf{z} + \mathbf{M}_{\perp}^\top \boldsymbol{\xi}, \; \mathbf{z} \sim \text{Unif}(\mathcal{H}^r), \; \boldsymbol{\xi} \sim  \text{Unif}(\mathcal{H}^{d-r})\nonumber
\end{align}
where $\mathbf{z}$ is a latent vector whose coordinates indicate the activations of the ground-truth features in the input $\mathbf{x}$ and $\boldsymbol{\xi}$ is a noise vector whose coordinates indicate the activation of the spurious features in the input.
Again, labels for the downstream task $T+1$ are generated by projecting the input onto the row space of $\mathbf{M}$ as follows:
\begin{align} 
 {f}_{T+1}(\mathbf{x}) = g_{T+1}(\sign(\mathbf{Mv})) = g_{T+1}(\mathbf{z}) \in \{-1,1\} \nonumber
\end{align}
We formally show below that the features learned during pretraining generalize to {\em any} such link function $g_{T+1}$.
\begin{theorem}[End-to-End Guarantee] \label{thm:downstream}
Let $\mathbf{W}^1$ be the outcome of the multi-task representation learning algorithm described in Section \ref{sec:alg} on the task distribution $\mathcal{T}(\mathbf{M})$, where $\mathcal{T}$ satisfies Assumption \ref{assump:diversity}. 
Consider a downstream task in the support of $\mathcal{T}(\mathbf{M})$ with link function $g_{T+1}$.
Construct the two-layer ReLU embedding $\phi$ using the rescaled $\mathbf{W}^1$ for first layer weights as in  \eqref{F}, and train the task-adapted head $(\mathbf{a}_{T+1},\tau_{T+1})$ using $N$ i.i.d. samples from the downstream task.
%
 Further, suppose $ d = \exp(\tilde{\Omega}( r^5  )) $, $T = d^2 r \exp(\tilde{\Omega}( r^5  )  )$, $Tn_2 = d^3 \exp(\tilde{\Omega}( r^5  )  )$, $n_1 = \Omega(\log(T))$ 
$m = \tilde{\Theta}(r^5)$,  and $\hat{m}= \exp\left( \tilde{\Omega}(r^5)  \right)$.
Then there is a setting of the parameters $\eta,   \lambda_{\mathbf{w}}$ and $\nu_{\mathbf{w}}$ such that for any $\delta \in (e^{-d},0.05]$, with probability at least $1-\delta$,
\begin{align}
        \mathcal{L}^{\text{eval}}_{T+1} 
        &= 
       \frac{ \exp(\tilde{O}(r^5)) }{ \sqrt{N}}. 
    \end{align}
\end{theorem}
Theorem \ref{thm:downstream} shows that the features learned by multi-task pretraining generalize to {\em any} downstream task 
that has the same representation as the pretraining tasks, i.e., its labels are a function of the input's projection onto the row space of $\mathbf{M}$. Specifically,
if we compose the learned representation with  a random ReLU layer, then learn a linear head using $\exp(\tilde{O}(r^5))$ samples from the task, we solve the task w.h.p.
Crucially, the number of samples and neurons  needed to solve the downstream task do not depend on the ambient dimension  $d$. 



The proof of Theorem \ref{thm:downstream} leverages Proposition \ref{prop:1} to show that the embedding  generated by multi-task learning is close to the embedding of a coupled, ``purified'' two-hidden layer random ReLU network whose first layer weights project the input {\em exactly} onto the row space of $\mathbf{M}$. Then, the proof applies  Theorem 2 from \citet{dirksen2022separation} which implies that w.h.p. the purified network linearly separates two classes of points on $\mathcal{H}^r$ with margin and neuron complexity scaling as functions of the input dimension $r$. The representation learning error from Proposition \ref{prop:1} is  smaller than this margin due to the lower bounds on $T$, $n_1, n_2$,  and $d$ in Theorem \ref{thm:downstream}, so the learned network also linearly separates the two classes w.h.p. Then, the proof invokes a standard linear classification generalization bound to control the final error in learning the head \cite{Livni_2017}. Note that Theorem \ref{thm:downstream} requires $d^3$ training sample complexity rather than the $d^2$ complexity of Theorem \ref{thm:1} because Theorem \ref{thm:1} concerns the spectral norm of the representation learning error, whereas for the generalization result, we require a Frobenius norm bound, which induces an extra $d$ factor. Please see the proof of Lemma C.3 for details.

{

\begin{remark}[Necessity of second layer] \label{rem:2}
 In the ideal representation learning scenario, the  first-layer weights are i.i.d. isotropic Gaussians in the ground-truth subspace $\text{row}(\mathbf{M})$. In this scenario we can think of the network as taking an $r$-dimensional input (corresponding to the $r$ ground-truth features of the input) and having first-layer weights that are i.i.d.isotropic Gaussians in $\mathbb{R}^r$. Even in this ideal scenario, existing results have not shown whether such a representation is sufficiently expressive or generalization to all downstream tasks w.h.p.
 There are several positive results for the expressivity of a  random, finite-width ReLU layer, but these concern approximating low-degree polynomials under the squared loss \cite{hsu2021approximation,ji2019neural,yehudai2019power,bach2017breaking}. 
However, to our knowledge, there are no analogous positive results showing that one layer of random ReLU neurons can linearly separate  two arbitrary classes of points on the Boolean hypercube  w.h.p., even with  exponentially-many neurons or exponentially-small margin. 
\end{remark}

\begin{remark}[Tightness of exponential complexity in $r$ in positive results]
Replacing $d$ with $r$, Theorem \ref{thm:lower} implies  that at least $r^{\Omega(k)}$ samples or width is necessary to express all $k$-sparse parity tasks on $r$ inputs. So, even if we learn exactly the correct representation, we require $r^{\Omega(r)}$ samples or width to solve all $\frac{r}{2}$-sparse parity tasks on the $r$ ground-truth features. Please see \cite{malach2022hardness,abbe2020poly,abbe2022merged,abbe2023sgd,shalev2017failures,hsu2021approximation,kamath2020approximate,ghorbani2020neural} 
for similar lower bounds. 
Nevertheless, our complexity of $\exp(\text{poly}(r))$ is larger than such lower bounds. We leave to future work to investigate whether the $\text{poly}(r)$ complexity in the exponent can be reduced.
\end{remark}

}

\subsection{Negative Results}
Next, we present two negative results that underscore the tightness of our findings in the previous section. The first result emphasizes the significance of representation learning in achieving strong generalization guarantees. The second result highlights the importance of multi-task learning by demonstrating that single-task learning may fail to capture all critical features.


\textbf{Random features do not generalize.} A consequence of Theorem \ref{thm:downstream} is that multi-task pretraining improves the sample and neuron complexity of solving downstream tasks by an exponential factor in $d$. 
To show this, we consider sparse parity tasks,  and show that learning a linear classifier on top of random features
entails exponential complexity in $d$ to solve the task.
Now, there is no feature learning, so the learner has no knowledge of which few features are relevant and needs to consider tasks on all $d$ inputs.
We model this by considering a set of tasks sharing a single link function but having many distinct representations.
We consider a smaller class of representations than in our positive results: here $\mathbf{M}$ belongs to $\mathbb{O}^{r\times d}_{\{0,1\}}\coloneqq \{\mathbf{M}: \mathbf{M}\in \mathbb{O}^{r\times d}, \mathbf{M}\in \{0,1\}^{r\times d} \}$, that is, the rows of $\mathbf{M}$ are standard basis elements. The single link function we consider is the parity function on $r$ inputs, namely $g^{(r)}(\mathbf{v})\coloneqq (-1)^{\sum_{j =1}^r \chi\{{v}_j = -1\}}$.

While there is a large literature demonstrating the hardness of learning sparse parities in various settings \cite{kearns1998efficient,abbe2020poly,telgarsky2022feature,barak2022hidden,malach2022hardness,kamath2020approximate,goel2019time}, the most relevant results to our setting show that any data-independent, $\hat{m}$-dimensional embedding of $d$ inputs can admit linear classifiers that solve all sparse parity tasks on subset only if the dimension $\hat{m}$ and/or the classification margin is exponentially large (small, respectively) in $d$.
In particular, the following result adapts Theorem 5 in \cite{barak2022hidden}, which in turn draws on the works of \citet{kamath2020approximate} and \citet{malach2022hardness}.
\begin{theorem}[] \label{thm:lower}
Consider any embedding $\Psi:\mathcal{H}^d \rightarrow \mathbb{R}^{\hat{m}}$ such that $\|\Psi(\mathbf{v})\|_2 \leq 1 $ for all $\mathbf{v}\in \mathcal{H}^d$. 
For 
any $\epsilon >0$, if 
$ \hat{m}B^2 \leq \epsilon^2 \binom{d}{r}, $
then there exists a 
representation $\mathbf{M} \in \mathbb{O}^{r\times d}_{\{0,1\}}$
such that: 
    \begin{align}
        \inf_{\mathbf{a}: \|\mathbf{a}\|_2 \leq B } \mathbb{E}_{\mathbf{v}\sim \text{Unif}(\mathcal{H}^d)}
        \left[ \ell \left(\mathbf{a}^\top \Psi(\mathbf{v}), g^{(r)}(\mathbf{Mv})\right)   \right]  &\geq 1 - \epsilon.\nonumber
    \end{align}
\end{theorem}
Theorem \ref{thm:lower} implies that any random feature model requires a number of neurons and/or  inverse margin that is polynomially large in $d^r$ in order to solve a downstream sparse parity task with a linear classifier. Note that the margin (i.e. inverse of $B$ in Theorem \ref{thm:lower}) is inversely proportional to the number of samples that are required to learn the classifier \cite{shamir2015sample}. On the other hand, Theorem~\ref{thm:1} guarantees that  after multi-task pretraining, the output embedding admits a linear classifier that solves any sparse parity task  on the extracted $r$ features, with the number of neurons and samples of the downstream task of the order of $\exp(\poly(r))$.



\textbf{Single task does not suffice for  feature learning.} Although Theorem \ref{thm:lower} shows that feature learning is essential for generalization in our setting, we have not yet shown that effective feature learning necessitates pretraining on {\em multiple} tasks. We address this issue next.  
\begin{theorem}\label{thm:single}
    Consider any algorithm $\mathcal{A}$ that takes as input infinite samples from any {\em single} task in  $\mathcal{T}_{\text{s.p.}}(\mathbf{M})$ and returns an $\hat{m}$-dimensional representation $\Psi:\mathcal{H}^d \rightarrow \mathbb{R}^{\hat{m}}$. 
Then
 there exists an $\mathbf{M}\in \mathbb{O}^{r\times d}_{\{0,1\}}$ such that for any $k\in[r]$, with probability at least $1 - 2^{-r}\sum_{j=k}^r\binom{r}{j} $ over the draw of a single training task $f_1 \sim \mathcal{T}_{\text{s.p.}}(\mathbf{M})$, the representation $\Psi_{f_1}:= \mathcal{A}(f_1)$ satisfies that
for any $\epsilon >0$, 
$ \hat{m}B^2 > \epsilon^2 \binom{d-k+1}{r-k+1}$ is necessary to obtain
    \begin{align}
   \min_{\mathbf{a}_2: \|\mathbf{a}_2\|_2 \leq B}      \mathbb{E}_{\mathbf{v} \sim \text{Unif}(\mathcal{H}^d)}[\ell( \mathbf{a}_2^\top \Psi_{f_1}(\mathbf{v}), \; f_{2}(\mathbf{v}) )] \geq 1- \epsilon. \nonumber
    \end{align}
\end{theorem}
Theorem \ref{thm:single} shows that w.h.p., a single task drawn from the task distribution $\mathcal{T}_{\text{s.p.}}(\mathbf{M})$ cannot be used to guarantee generalization with downstream neuron and margin complexity smaller than the ambient dimension for all ground-truth representations $\mathbf{M}$. For example, if $k=r$, then with probability at least $1-2^{-r}$, the number of neurons must be $\Omega(d)$ and/or the margin must be $O(d^{-1/2})$ to allow for non-trivial error. The underlying reason is that most tasks in $\mathcal{T}_{\text{s.p.}}(\mathbf{M})$ are ``simple'' in that they only depend on a strict subset of the  $r$ ground-truth features,  thus do not contain information about all the important features (although they are still ``hard'' by virtue of being sparse parity tasks), so single-task pretraining cannot improve upon random features in terms of recovering the remaining important features. Nevertheless, Theorem \ref{thm:1} shows that multi-task pretraining aggregates information across the tasks to learn a generalizable model.


\begin{remark}[Single-task training with highly informative task]
     Theorem \ref{thm:1}  leaves open the possibility that training on a highly-informative task could perform as well as multi-tasking. Note that  there is one task supported by $\mathcal{T}_{\text{s.p.}}(\mathbf{M})$, the full parity task, that provides information abut all $r$ ground-truth features in $\mathbf{M}$. While gradient-based training on this task may allow for efficient generalization to any downstream task on the $r$ features \cite{barak2022hidden}\footnote{\citet{barak2022hidden} show that SGD on a two-layer ReLU NN with batch size $\tilde{\Omega}(d^r)$ can solve the parity task on $r$ features with unknown $\mathbf{M} \in \mathbb{O}_{{0,1}}^{d\times r}$, which suggests that the representation learned during this process generalizes to simpler tasks on the features in $\mathbf{M}$.}, the sample complexity of this training may be much larger than multi-tasking.
    Additional prior results show that 
    gradient-based algorithms require at least $\Omega(d^r)$ samples to solve the full parity task on $r$ inputs \cite{abbe2023sgd,abbe2020poly,shalev2017failures}. This complexity has worse dependence on $d$ and $T$ than the $n_1+n_2 = \tilde{O}(\frac{d^3}{T})$ training samples per task required by Theorem \ref{thm:downstream} for downstream generalization with the number of downstream samples independent of $d$. In fact,  it is even worse complexity in $d$ than the $T(n_1+n_2) + N = \tilde{O}(d^3)$ total samples across tasks that Theorem \ref{thm:downstream} requires for multi-task pretraining followed by downstream adaptation. Thus multi-tasking reduces feature learning  sample complexity compared to training with {\em any} single task. 

\end{remark}

\section{Proof Sketch} \label{sec:sketch}

In this section, we sketch the proof of Proposition \ref{prop:1}, which is the key feature learning result that enables downstream guarantees.
The proof heavily leverages the fact that multi-task pretraining entails updating the first-layer weights {\em after} fitting a unique head to each task.
{{Surprisingly, we show that making one gradient-based update of the head for each task induces a pseudo-contrastive loss that encourages representations of two points to be similar if and only if they are likely to share a label on a randomly drawn task.}} Since two points are likely to share a label on a drawn task if and only if they share the same sign pattern on their $r$ ground-truth features (by Assumption \ref{assump:diversity}), the pseudo-contrastive loss inclines the representation to extract these $r$ latent features\footnote{We also show in Appendix \ref{app:regression} that these intuitions can be extended to the regression setting.}. For ease of exposition, in this setting we focus on the population setting with infinite tasks and samples per task, and defer the finite-task and samples proof to Appendix \ref{app:thrm1}.

\textbf{Step 1: Derive pseudo-contrastive loss after head updates.} We first update the task-specific head with one gradient step for each task $i$ given the initial parameters $(\mathbf{W}^0,\mathbf{b}^0,\mathbf{a}_{i}^0)$. 
Due to the choice of $\mathbf{a}_i^0=\mathbf{0}_m$ for all $i\in[T]$, $f(\mathbf{x};\mathbf{W}^0,\mathbf{b}^0,\mathbf{a}^0_i)=0$ for all $\mathbf{x}\in \mathbb{R}^d$ and $i\in[T]$, so the hinge loss is affine in $\mathbf{a}_i^0$ for all tasks (the $\max(\cdot,0)$ threshold is inactive). Therefore, using also the choice of $\mathbf{b}^0=\mathbf{0}_m$, we have
\begin{align}
   \mathbf{a}_{i}^1 &= (1 - \eta \lambda_{\mathbf{a}})\mathbf{a}_{i}^0 - \eta\nabla \mathcal{L}_i(\mathbf{W}^0,\mathbf{b}^0,\mathbf{a}_i^0) = \eta \mathbb{E}_{\mathbf{x}} [ f_{i}(\mathbf{x}) \sigma(\bar{\mathbf{W}}^0\mathbf{x}) ]  \label{head1}
\end{align} 
where we use $\bar{\mathbf{W}}^0$ to denote a stop-gradient on ${\mathbf{W}}^0$.
As a result, the updated head for task $i$, $\mathbf{a}_{i}^1$, is proportional to the average label-weighted neuron output over the dataset for task $i$. Now we can insert this value of $\mathbf{a}_{i}^1$ back into the loss, to obtain $\mathcal{L}_i(\mathbf{W}^0, \mathbf{b}^0,\mathbf{a}_{i}^1)$. For ease of notation we define 
$\beta({\mathbf{x},\mathbf{x}'})\coloneqq \mathbb{E}_{i}[f_{i}(\mathbf{x})  f_{i}(\mathbf{x}')]$ for all pairs of inputs $\mathbf{x},\mathbf{x}'$, and replace $\max(\cdot,0)$ with $\sigma(\cdot)$ in the hinge loss (recall $\sigma(\cdot)$ is the ReLU).
Taking the average over all tasks yields
\begin{align}
    \mathcal{L}(\mathbf{W}^0,\mathbf{b}^0,\{\mathbf{a}^{1}_i\}_{i})
    &= \E_{i,\mathbf{x} }  \left[ \sigma\left( 1- \eta f_{i}(\mathbf{x}) \mathbb{E}_{\mathbf{x}'} [ f_{i}(\mathbf{x}') \sigma(\bar{\mathbf{W}}^0 \mathbf{x}') ]^\top \sigma( \mathbf{W}^0\mathbf{x})   \right) 
\right] \nonumber \\
&\approx 1 -\eta \E_{\mathbf{x}, \mathbf{x}' }  \left[  \beta({\mathbf{x},\mathbf{x}'}) \sigma(\bar{\mathbf{W}}^0 \mathbf{x}')^\top \sigma( \mathbf{W}^0\mathbf{x}) \right] \label{tjjt}
\end{align}
where 
the approximation holds as $|\eta \mathbb{E}_{\mathbf{x}'} [ f_{i}(\mathbf{x}') \sigma(\mathbf{W}^0\mathbf{x}') ]^\top \sigma(\mathbf{W}^0\mathbf{x}) | < 1$ w.h.p. over $\mathbf{x}$ and $\mathbf{W}^0$. The resulting loss in \eqref{tjjt} encourages the first-layer representation to align sample pairs $(\mathbf{x},{\mathbf{x}'})$ that  have the same label for most tasks ($ \beta({\mathbf{x},\mathbf{x}'}) \approx 1$) and penalizes the representation for aligning pairs of samples that do not have the same label on most tasks ($ \beta({\mathbf{x},\mathbf{x}'}) \ll 1$). In this way, \eqref{tjjt} is reminiscent of a constrastive loss\footnote{This analysis suggests that any $\beta(\mathbf{x},\mathbf{x}')$ that is ``roughly'' increasing with the similarity of $\mathbf{Mx}$ and $\mathbf{Mx}'$ results in a pseudo-contrastive loss, that, as we later show, results in representation learning. As such, our observations suggest that
 Assumption \ref{assump:diversity} can be relaxed to the very natural condition that the correlation of the labels of $\mathbf{x}$ and $\mathbf{x}'$ across tasks is roughly increasing with the similarity of their ground-truth features. We verify this conjecture empirically in Appendix \ref{sec:sims}.} \cite{chen2020simple} in which positive pairs are pairs with large $\beta$.



To translate these intuitive connections with contrastive learning to feature learning, we must leverage Assumption \ref{assump:diversity}, which implies that
$\beta({\mathbf{x},\mathbf{x}'})$ encodes information about the ground-truth features  $\mathbf{Mx}$ and $\mathbf{Mx}'$.
In particular, pairs of points with the {\em same} sign patterns among the ground-truth features have the same label, so $\beta({\mathbf{x},\mathbf{x}'})=1$ almost surely, while pairs of points with {\em different} sign patterns on the ground-truth features have the same label for only half of the tasks in the universe of tasks $\mathcal{T}$, meaning $\beta({\mathbf{x},\mathbf{x}'})=0$.
As a result, the loss can now be approximated as:
\begin{align}
    &\tfrac{1}{\eta}\mathcal{L}(\mathbf{W}^0,\mathbf{b}^0,\{\mathbf{a}_{i}^1\}_{i}) \approx - \E_{\mathbf{x}, \mathbf{x}' } \! \left[ \chi\{\sign(\mathbf{Mx})\!= \!\sign(\mathbf{Mx}')\}\sigma(\bar{\mathbf{W}}^0\mathbf{x}') ^\top \sigma(\mathbf{W}^0\mathbf{x}) \right] \label{15}
\end{align}
We next show that a gradient descent step on \eqref{15}
results in $\mathbf{W}^1$ essentially projecting onto the row space of $\mathbf{M}$.


\textbf{Step 2: Update neuron weights.} The proof of this step requires computing the gradient of $\mathcal{L}(\mathbf{W}^0,\mathbf{b}^0,\{\mathbf{a}_{i}^1\}_{i})$  with respect to each vector of neuron weights. For ease of notation, we from here onwards denote $\mathbf{w}_{j}=\mathbf{w}_j^0$. Using \eqref{15}, this gradient can be approximated by $\mathbf{A}(\mathbf{w}_j)\mathbf{w}_{j}$, where  $\mathbf{A}(\mathbf{w}) \in \mathbb{R}^{d\times d}$ is defined as:
\begin{align}
    &\mathbf{A}(\mathbf{w}) := \!-\! \E_{\mathbf{x}, \mathbf{x}' }  \big[ \chi\{\sign(\mathbf{Mx}) = \sign(\mathbf{Mx}')\}  \sigma'(\mathbf{w}^\top \mathbf{x})\sigma'(\bar{\mathbf{w}}^\top \mathbf{x}')\mathbf{x}(\mathbf{x}')^\top \big]\label{16}  
\end{align}
where $\sigma'(z) = 1$ if $z>0$ and $\sigma'(z) = 0$ otherwise.
The crucial reason why $\mathbf{A}(\mathbf{w}_{j})$ has favorable structure is due to the indicator $ \chi\{\sign(\mathbf{Mx})= \sign(\mathbf{Mx}')\}$ in the RHS of \eqref{15}.
Intuitively, this indicator encourages the first-layer weights to align {\em only} the representations of points with the same sign pattern on the label-relevant coordinates, by ensuring that only these pairs of points appear in the gradient. 
With this indicator removed, we would have $\mathbf{A}(\mathbf{w}_{j}) = \frac{2}{\pi \|\mathbf{w}_{j}\|_2^2}\mathbf{w}_{j} \mathbf{w}_{j}^\top $, meaning the gradient would not put any emphasis on the ground-truth projection. However, 
the indicator means that $ \mathbf{A}(\mathbf{w}_{j})$ is an average outer product over vectors whose signs agree on the $r$ important features and may disagree on all other features. This disagreement results in cancellation during averaging, unlike the important $r$ features, leading to:
\begin{align}
    \mathbf{M}\mathbf{A}(\mathbf{w}_{j}) \mathbf{w}_j 
    &\approx -\tfrac{1}{2^{r}\pi }\mathbf{Mw}_j,\quad \quad
   \mathbf{M}_\perp \mathbf{A}(\mathbf{w}_{j}) \mathbf{w}_j  \approx -\tfrac{1}{2^{r+1}\pi }\mathbf{M}_\perp \mathbf{w}_j, \label{rowp2}
\end{align}
meaning that the gradient up-weights the energy of $\mathbf{w}_j$ in the ground-truth subspace by a factor of roughly 2 compared to the the energy in the spurious subspace.
Moreover, applying $ \mathbf{A}(\mathbf{w}_{j})$ to $\mathbf{w}_{j}$ does not change the direction of $\mathbf{Mw}_{j}$, meaning $\mathbf{Mw}_{j}^1$ remains isotropic in $\mathbb{R}^r$. These observations are the crux of the proof; please see Appendix \ref{app:thrm1} for full details.

\section{Conclusion}
We have provided the {\em first results} showing that multi-task pretraining with a gradient-based algorithm on a non-linear neural network learns generalizable features. Moreover, our analysis reveals that updating the task-specific heads prior to updating the first-layer weights induces a supervised contrastive loss that encourages recovering the features indicative of whether two points share a label. As a result, this work suggests  further 
exploring the role of adapting the head to each task in order to learn more expressive features. 

\section*{Acknowledgements}
L.C., A.M. and S.S. are supported in part by NSF Grants 2127697, 2019844, 2107037, and 2112471, ARO Grant W911NF2110226, ONR Grant N00014-19-1-2566, the Machine Learning Lab (MLL) at UT Austin, and the Wireless Networking and Communications Group (WNCG) Industrial Affiliates Program. M.S. is supported by an NIH Director's new innovator award \#1DP2LM014564-01, a Packard Fellowship in Science and Engineering, a Sloan Research Fellowship in Mathematics, an NSF-CAREER under award \#1846369, DARPA FastNICS program, and NSF-CIF award \#2008443. H.H. is supported by the NSF Institute for CORE Emerging Methods in Data Science (EnCORE) as well as The Institute for Learning-enabled Optimization at Scale (TILOS).


\bibliography{refs}
\bibliographystyle{icml2024}

\newpage
\appendix
\newpage
\tableofcontents
\newpage

\newpage

\section{Proofs of Proposition \ref{prop:1} and Theorem \ref{thm:1}} \label{app:thrm1}

In this section we prove Proposition \ref{prop:1} and Theorem \ref{thm:1}. Throughout, we will slightly abuse notation by reusing $c,c',c''$ and $C$ as absolute constants independent of all other parameters. The notations $O(\cdot)$, $\Theta(\cdot)$, and $\Omega(\cdot)$ describe scalings up to absolute constants independent of all other parameters. 

\subsection{General lemmas}

\begin{lemma} \label{lem:exp1}
Suppose $\mathbf{x} \sim \mathcal{N}(\mathbf{0}_d,\mathbf{I}_d)$. Then for any $\mathbf{w}\in \mathbb{R}^d$,
    $$\mathbb{E}_{\mathbf{x}}\left[
\sigma'(\mathbf{w}^\top \mathbf{x}) \mathbf{x}\right] = \sqrt{\frac{2}{\pi}} \frac{\mathbf{w}}{\|\mathbf{w}\|_2}.$$
\end{lemma}

\begin{proof}
 For any $\mathbf{u}: \mathbf{w}^\top\mathbf{u}=0$, we have
 \begin{align}
     \mathbf{u}^\top (\mathbb{E}_{\mathbf{x}}\left[
\sigma'(\mathbf{w}^\top \mathbf{x}) \mathbf{x}\right]) &= \mathbb{E}_{\mathbf{x}}\left[
\sigma'(\mathbf{w}^\top \mathbf{x}) \mathbf{u}^\top\mathbf{x}\right]\nonumber \\
&= \mathbb{E}_{\mathbf{x}}\left[
\sigma'(\mathbf{w}^\top \mathbf{x})\right]  \mathbb{E}_{\mathbf{x}}\left[\mathbf{u}^\top\mathbf{x}\right]\nonumber \\
&= 0
 \end{align}
 by the independence of orthogonal projections of isotropic Gaussian vectors. So, $\mathbb{E}_{\mathbf{x}}\left[
\sigma'(\mathbf{w}^\top \mathbf{x}) \mathbf{x}\right]$ is parallel to $\mathbf{w}$. Thus,
\begin{align}
    \mathbb{E}_{\mathbf{x}}\left[
\sigma'(\mathbf{w}^\top \mathbf{x}) \mathbf{x}\right] &= \frac{\mathbf{ww}^\top}{\|\mathbf{w}\|_2^2}  \mathbb{E}_{\mathbf{x}}\left[
\sigma'(\mathbf{w}^\top \mathbf{x}) \mathbf{x}\right]\nonumber \\
&= \frac{\mathbf{w}}{\|\mathbf{w}\|_2^2}  \mathbb{E}_{\mathbf{x}}\left[
\sigma'(\mathbf{w}^\top \mathbf{x}) \mathbf{w}^\top\mathbf{x}\right]\nonumber \\
&= \frac{\mathbf{w}}{\|\mathbf{w}\|_2^2}  \mathbb{E}_{\mathbf{x}}\left[
\sigma(\mathbf{w}^\top \mathbf{x}) \right]
\end{align}
where $\sigma(\mathbf{w}^\top \mathbf{x})$ is a half-normal random variable with parameter $\|\mathbf{w}\|_2$, so it has mean $\|\mathbf{w}\|_2\sqrt{\frac{2}{\pi}}$, completing the proof.
\end{proof}

\begin{lemma}
    \label{lem:exp-perp}
Suppose $\mathbf{x} \sim \mathcal{N}(\mathbf{0}_d,\mathbf{I}_d)$ and $\mathbf{M}_\perp \in \mathbb{O}^{(d-r)\times d}$ and $\mathbf{M} \in \mathbb{O}^{r \times d}$ such that $\mathbf{M}_{\perp}\mathbf{M}^\top = \mathbf{0}_{(d-r)\times r}$, i.e.  the rowspaces of $\mathbf{M}$ and $\mathbf{M}_{\perp}$ are orthogonal. Then for any $\mathbf{w}\in \mathbb{R}^d$,
\begin{align}
    \mathbb{E}_{\mathbf{M}_\perp \mathbf{x}} \left[ \mathbf{M}_\perp \mathbf{x} \; \sigma'(\mathbf{w}^\top \mathbf{x})  \right] = {\frac{1}{\sqrt{2\pi}} \exp\left(-\frac{(\mathbf{w}^\top \mathbf{M}^\top\mathbf{M}\mathbf{x})^2}{2} \right) }\frac{\mathbf{M}_\perp \mathbf{w} }{ \| \mathbf{M}_\perp \mathbf{w} \|_2} 
\end{align}
\end{lemma}

\begin{proof}
  We have
\begin{align}
    &\mathbb{E}_{\mathbf{M}_\perp \mathbf{x}} \left[ \mathbf{M}_\perp \mathbf{x} \; \sigma'(\mathbf{w}^\top \mathbf{x})  \right]\nonumber \\
    &= \mathbb{E}_{\mathbf{M}_\perp \mathbf{x}} \left[ \mathbf{M}_\perp \mathbf{x} \; \sigma'(\mathbf{w}^\top \mathbf{M}^\top\mathbf{M}\mathbf{x} + \mathbf{w}^\top \mathbf{M}_\perp^\top\mathbf{M}_\perp \mathbf{x})  \right] \nonumber \\
    &= \mathbb{E}_{\mathbf{M}_\perp \mathbf{x}} \left[ \mathbf{M}_\perp \mathbf{x} |\mathbf{w}^\top \mathbf{M}^\top\mathbf{M}\mathbf{x} + \mathbf{w}^\top \mathbf{M}_\perp^\top\mathbf{M}_\perp \mathbf{x} > 0 \right] \mathbb{P}_{\mathbf{M}_\perp \mathbf{x}}[\mathbf{w}^\top \mathbf{M}^\top\mathbf{M}\mathbf{x} + \mathbf{w}^\top \mathbf{M}_\perp^\top\mathbf{M}_\perp \mathbf{x} > 0] \nonumber \\
    &= \mathbb{E}_{\mathbf{M}_\perp \mathbf{x}} \left[ \mathbf{M}_\perp \mathbf{x} \bigg|\mathbf{w}^\top \mathbf{M}^\top\mathbf{M}\mathbf{x} > | \mathbf{w}^\top \mathbf{M}_\perp^\top\mathbf{M}_\perp \mathbf{x} |\right] \mathbb{P}_{\mathbf{M}_\perp \mathbf{x}}[\mathbf{w}^\top \mathbf{M}^\top\mathbf{M}\mathbf{x} >| \mathbf{w}^\top \mathbf{M}_\perp^\top\mathbf{M}_\perp \mathbf{x}| ] \label{bh} 
\end{align}
where the last line follows by considering two cases: (i) $\mathbf{w}^\top \mathbf{M}_\perp^\top\mathbf{M}_\perp \mathbf{x} <0$ and (ii) $\mathbf{w}^\top \mathbf{M}_\perp^\top\mathbf{M}_\perp \mathbf{x} >0$.
If case (i) holds, then $ - \mathbf{w}^\top \mathbf{M}_\perp^\top\mathbf{M}_\perp \mathbf{x} = |\mathbf{w}^\top \mathbf{M}_\perp^\top\mathbf{M}_\perp \mathbf{x}|$ so $$\mathbb{E}_{\mathbf{M}_\perp \mathbf{x}} \left[ \mathbf{M}_\perp \mathbf{x} |\mathbf{w}^\top \mathbf{M}^\top\mathbf{M}\mathbf{x} + \mathbf{w}^\top \mathbf{M}_\perp^\top\mathbf{M}_\perp \mathbf{x} > 0 \right]= \mathbb{E}_{\mathbf{M}_\perp \mathbf{x}} \left[ \mathbf{M}_\perp \mathbf{x} \bigg|\mathbf{w}^\top \mathbf{M}^\top\mathbf{M}\mathbf{x} > | \mathbf{w}^\top \mathbf{M}_\perp^\top\mathbf{M}_\perp \mathbf{x} |\right]$$ and $$\mathbb{P}_{\mathbf{M}_\perp \mathbf{x}}[\mathbf{w}^\top \mathbf{M}^\top\mathbf{M}\mathbf{x} + \mathbf{w}^\top \mathbf{M}_\perp^\top\mathbf{M}_\perp \mathbf{x} > 0]=\mathbb{P}_{\mathbf{M}_\perp \mathbf{x}}[\mathbf{w}^\top \mathbf{M}^\top\mathbf{M}\mathbf{x} >| \mathbf{w}^\top \mathbf{M}_\perp^\top\mathbf{M}_\perp \mathbf{x}| ].$$
Alternatively, if case (ii) holds,  then by the law of total expectation,
\begin{align}
   &\mathbb{E}_{\mathbf{M}_\perp \mathbf{x}} \left[ \mathbf{M}_\perp \mathbf{x} |\mathbf{w}^\top \mathbf{M}^\top\mathbf{M}\mathbf{x} + \mathbf{w}^\top \mathbf{M}_\perp^\top\mathbf{M}_\perp \mathbf{x} > 0 \right] \nonumber \\
   &= \mathbb{E}_{\mathbf{M}_\perp \mathbf{x}} \left[ \mathbf{M}_\perp \mathbf{x} |\mathbf{w}^\top \mathbf{M}^\top\mathbf{M}\mathbf{x} > \mathbf{w}^\top \mathbf{M}_\perp^\top\mathbf{M}_\perp \mathbf{x}  \right] \nonumber \\
   &\quad \quad \times \mathbb{P}_{\mathbf{M}_\perp \mathbf{x}} \left[ \mathbf{w}^\top \mathbf{M}^\top\mathbf{M}\mathbf{x} > \mathbf{w}^\top \mathbf{M}_\perp^\top\mathbf{M}_\perp \mathbf{x}  |\mathbf{w}^\top \mathbf{M}^\top\mathbf{M}\mathbf{x} > -\mathbf{w}^\top \mathbf{M}_\perp^\top\mathbf{M}_\perp \mathbf{x} \right] \nonumber \\
   &\quad + \mathbb{E}_{\mathbf{M}_\perp \mathbf{x}} \left[ \mathbf{M}_\perp \mathbf{x} |\mathbf{w}^\top \mathbf{M}_\perp^\top\mathbf{M}_\perp \mathbf{x}> \mathbf{w}^\top \mathbf{M}^\top\mathbf{M}\mathbf{x} > -\mathbf{w}^\top \mathbf{M}_\perp^\top\mathbf{M}_\perp \mathbf{x}  \right] \nonumber \\
   &\quad \quad \times \mathbb{P}_{\mathbf{M}_\perp \mathbf{x}} \left[ \mathbf{w}^\top \mathbf{M}_\perp^\top\mathbf{M}_\perp \mathbf{x}> \mathbf{w}^\top \mathbf{M}^\top\mathbf{M}\mathbf{x} > -\mathbf{w}^\top \mathbf{M}_\perp^\top\mathbf{M}_\perp \mathbf{x}  |\mathbf{w}^\top \mathbf{M}^\top\mathbf{M}\mathbf{x} > -\mathbf{w}^\top \mathbf{M}_\perp^\top\mathbf{M}_\perp \mathbf{x} \right]  \nonumber \\
   &= \mathbb{E}_{\mathbf{M}_\perp \mathbf{x}} \left[ \mathbf{M}_\perp \mathbf{x} |\mathbf{w}^\top \mathbf{M}^\top\mathbf{M}\mathbf{x} > \mathbf{w}^\top \mathbf{M}_\perp^\top\mathbf{M}_\perp \mathbf{x}  \right] \nonumber\\
   &\quad \quad \times \mathbb{P}_{\mathbf{M}_\perp \mathbf{x}} \left[ \mathbf{w}^\top \mathbf{M}^\top\mathbf{M}\mathbf{x} > \mathbf{w}^\top \mathbf{M}_\perp^\top\mathbf{M}_\perp \mathbf{x}  |\mathbf{w}^\top \mathbf{M}^\top\mathbf{M}\mathbf{x} > -\mathbf{w}^\top \mathbf{M}_\perp^\top\mathbf{M}_\perp \mathbf{x} \right] \nonumber  \\
   &= \mathbb{E}_{\mathbf{M}_\perp \mathbf{x}} \left[ \mathbf{M}_\perp \mathbf{x} |\mathbf{w}^\top \mathbf{M}^\top\mathbf{M}\mathbf{x} > |\mathbf{w}^\top \mathbf{M}_\perp^\top\mathbf{M}_\perp \mathbf{x}|  \right]\nonumber \\
   &\quad\quad \times \mathbb{P}_{\mathbf{M}_\perp \mathbf{x}} \left[ \mathbf{w}^\top \mathbf{M}^\top\mathbf{M}\mathbf{x} > \mathbf{w}^\top \mathbf{M}_\perp^\top\mathbf{M}_\perp \mathbf{x}  |\mathbf{w}^\top \mathbf{M}^\top\mathbf{M}\mathbf{x} > -\mathbf{w}^\top \mathbf{M}_\perp^\top\mathbf{M}_\perp \mathbf{x} \right] \label{bnv}  \\
   &= \mathbb{E}_{\mathbf{M}_\perp \mathbf{x}} \left[ \mathbf{M}_\perp \mathbf{x} |\mathbf{w}^\top \mathbf{M}^\top\mathbf{M}\mathbf{x} > |\mathbf{w}^\top \mathbf{M}_\perp^\top\mathbf{M}_\perp \mathbf{x}|  \right]\frac{ \mathbb{P}_{\mathbf{M}_\perp \mathbf{x}} \left[ \mathbf{w}^\top \mathbf{M}^\top\mathbf{M}\mathbf{x} > |\mathbf{w}^\top \mathbf{M}_\perp^\top\mathbf{M}_\perp \mathbf{x}|\right] }{ \mathbb{P}_{\mathbf{M}_\perp \mathbf{x}} \left[\mathbf{w}^\top \mathbf{M}^\top\mathbf{M}\mathbf{x} > -\mathbf{w}^\top \mathbf{M}_\perp^\top\mathbf{M}_\perp \mathbf{x} \right]}  \nonumber
\end{align}
where \eqref{bnv} follows since $\mathbf{w}^\top \mathbf{M}_\perp^\top\mathbf{M}_\perp \mathbf{x} = |\mathbf{w}^\top \mathbf{M}_\perp^\top\mathbf{M}_\perp \mathbf{x}| $. 
Now we return to \eqref{bh}. 
Note that
\begin{align}
&\mathbb{E}_{\mathbf{M}_\perp \mathbf{x}} \left[ \mathbf{M}_\perp \mathbf{x} \bigg|\mathbf{w}^\top \mathbf{M}^\top\mathbf{M}\mathbf{x} > | \mathbf{w}^\top \mathbf{M}_\perp^\top\mathbf{M}_\perp \mathbf{x} |\right] \nonumber \\
&=\mathbb{E}_{\mathbf{M}_\perp \mathbf{x}} \left[ \mathbf{M}_\perp \mathbf{x} - 2 \sigma\left( -\frac{\mathbf{w}^\top \mathbf{M}_\perp^\top  }{ \| \mathbf{M}_\perp \mathbf{w} \|_2}\mathbf{M}_\perp \mathbf{x}\right)\frac{\mathbf{M}_\perp \mathbf{w} }{ \| \mathbf{M}_\perp \mathbf{w} \|_2} \bigg||\mathbf{w}^\top \mathbf{M}^\top\mathbf{M}\mathbf{x}| > | \mathbf{w}^\top \mathbf{M}_\perp^\top\mathbf{M}_\perp \mathbf{x} |\right]
\end{align}
by the symmetry of the Gaussian distribution and the fact that $$\mathbf{M}_\perp \mathbf{x} - 2 \sigma\left( -\frac{\mathbf{w}^\top \mathbf{M}_\perp^\top  }{ \| \mathbf{M}_\perp \mathbf{w} \|_2}\mathbf{M}_\perp \mathbf{x}\right)\frac{\mathbf{M}_\perp \mathbf{w} }{ \| \mathbf{M}_\perp \mathbf{w} \|_2} $$ is the flip of $\mathbf{M}_\perp \mathbf{x}$ across the hyperplane with normal vector $\frac{\mathbf{w}^\top \mathbf{M}_\perp^\top  }{ \| \mathbf{M}_\perp \mathbf{w} \|_2}$ when $\mathbf{w}^\top \mathbf{M}_\perp^\top \mathbf{M}_\perp \mathbf{x} < - | \mathbf{w}^\top \mathbf{M}_\perp^\top\mathbf{M}_\perp \mathbf{x} |$.
Using this, we obtain
\begin{align}
    &\mathbb{E}_{\mathbf{M}_\perp \mathbf{x}} \left[ \mathbf{M}_\perp \mathbf{x} \sigma'(\mathbf{w}^\top \mathbf{x})  \right]\nonumber \\
    &= \mathbb{E}_{\mathbf{M}_\perp \mathbf{x}} \left[ \mathbf{M}_\perp \mathbf{x} - 2  \sigma\left( -\frac{\mathbf{w}^\top \mathbf{M}_\perp^\top  }{ \| \mathbf{M}_\perp \mathbf{w} \|_2}\mathbf{M}_\perp \mathbf{x}\right)\frac{\mathbf{M}_\perp \mathbf{w} }{ \| \mathbf{M}_\perp \mathbf{w} \|_2} \bigg| |\mathbf{w}^\top \mathbf{M}_\perp^\top\mathbf{M}_\perp \mathbf{x}|  > |\mathbf{w}^\top \mathbf{M}^\top\mathbf{M}\mathbf{x}| \right] \nonumber \\
    &\quad \quad \times \mathbb{P}_{\mathbf{M}_\perp \mathbf{x}}\left[\mathbf{w}^\top \mathbf{M}_\perp^\top\mathbf{M}_\perp \mathbf{x} > |\mathbf{w}^\top \mathbf{M}^\top\mathbf{M}\mathbf{x}| \right]\nonumber \\
    &= 2\mathbb{E}_{\mathbf{M}_\perp \mathbf{x}} \left[ \sigma\left( -\frac{\mathbf{w}^\top \mathbf{M}_\perp^\top \mathbf{M}_\perp \mathbf{x} }{ \| \mathbf{M}_\perp \mathbf{w} \|_2}\right)\bigg| |\mathbf{w}^\top \mathbf{M}_\perp^\top\mathbf{M}_\perp \mathbf{x} | > |\mathbf{w}^\top \mathbf{M}^\top\mathbf{M}\mathbf{x}| \right]  \nonumber \\
    &\quad \quad \times \mathbb{P}_{\mathbf{M}_\perp \mathbf{x}}\left[\mathbf{w}^\top \mathbf{M}_\perp^\top\mathbf{M}_\perp \mathbf{x} > |\mathbf{w}^\top \mathbf{M}^\top\mathbf{M}\mathbf{x}| \right] \frac{\mathbf{M}_\perp \mathbf{w} }{ \| \mathbf{M}_\perp \mathbf{w} \|_2}  \nonumber \\
    &=\mathbb{E}_{\mathbf{M}_\perp \mathbf{x}} \left[ \sigma\left( -\frac{\mathbf{w}^\top \mathbf{M}_\perp^\top \mathbf{M}_\perp \mathbf{x} }{ \| \mathbf{M}_\perp \mathbf{w} \|_2}\right)\bigg| \mathbf{w}^\top \mathbf{M}_\perp^\top\mathbf{M}_\perp \mathbf{x}  < - |\mathbf{w}^\top \mathbf{M}^\top\mathbf{M}\mathbf{x}| \right]   \nonumber \\
    &\quad \quad \times \mathbb{P}_{\mathbf{M}_\perp \mathbf{x}}\left[\mathbf{w}^\top \mathbf{M}_\perp^\top\mathbf{M}_\perp \mathbf{x} > |\mathbf{w}^\top \mathbf{M}^\top\mathbf{M}\mathbf{x}| \right] \frac{\mathbf{M}_\perp \mathbf{w} }{ \| \mathbf{M}_\perp \mathbf{w} \|_2}  \nonumber \\
    &= \frac{\frac{1}{\sqrt{2\pi}} \exp(-(\mathbf{w}^\top \mathbf{M}^\top\mathbf{M}\mathbf{x})^2/2 ) }{\mathbb{P}_{\mathbf{M}_\perp \mathbf{x}}\left[\mathbf{w}^\top \mathbf{M}_\perp^\top\mathbf{M}_\perp \mathbf{x} <- |\mathbf{w}^\top \mathbf{M}^\top\mathbf{M}\mathbf{x}| \right]} \mathbb{P}_{\mathbf{M}_\perp \mathbf{x}}\left[\mathbf{w}^\top \mathbf{M}_\perp^\top\mathbf{M}_\perp \mathbf{x} > |\mathbf{w}^\top \mathbf{M}^\top\mathbf{M}\mathbf{x}| \right] \frac{\mathbf{M}_\perp \mathbf{w} }{ \| \mathbf{M}_\perp \mathbf{w} \|_2}  \label{ddw} \\
    &= {\frac{1}{\sqrt{2\pi}} \exp(-(\mathbf{w}^\top \mathbf{M}^\top\mathbf{M}\mathbf{x})^2/2 ) }\frac{\mathbf{M}_\perp \mathbf{w} }{ \| \mathbf{M}_\perp \mathbf{w} \|_2}  \label{mn}
\end{align}
where \eqref{ddw} follows by the definition of the inverse Mills ratio.
\end{proof}

\begin{lemma} \label{lem:exp2}
For any function $f:\mathbb{R}^d\rightarrow \{-1,1\}$ such that $f(\mathbf{x}) = g(\mathbf{Mx})$ for some row-orthonormal matrix $\mathbf{M}\in \mathbb{O}^{r\times d}$ and some function $g: \mathbb{R}^r \rightarrow \{-1,1\}$ for all $\mathbf{x}\in \mathbb{R}^d$. Then for any vector $\mathbf{w} \in \mathbb{R}^d$,
\begin{align}
    \|\mathbb{E}_{\mathbf{x}}\left[
f(\mathbf{x}) \sigma'(\mathbf{w}^\top \mathbf{x}) \mathbf{x}\right] \|_2 \leq \frac{\sqrt{r}}{2}\left( 1 + \frac{\sqrt{\pi}}{2} \frac{\|\mathbf{Mw}\|_2}{\|\mathbf{M}_\perp\mathbf{w}\|_2} \right) +  \left( \frac{1 + \|\mathbf{Mw}\|_2^2 }{2\pi } \right)^{1/2}.
\end{align}
\end{lemma}

\begin{proof}
Let $\mathbf{M}_{\perp}\in \mathbb{O}^{(d-r)\times d}$ be a row-orthonormal matrix whose rowspace is orthogonal to that of $\mathbf{M}$. Using that $\mathbf{M}^\top\mathbf{M} + \mathbf{M}_\perp^\top\mathbf{M}_\perp = \mathbf{I}_d$ and $\mathbf{Mx}$ and $\mathbf{M}_{\perp}\mathbf{x}$ are independent standard normal multivariate random vectors,  we have 
\begin{align}
    \mathbb{E}_{\mathbf{x}}\left[f(\mathbf{x})
\sigma'(\mathbf{w}^\top \mathbf{x}) \mathbf{x}\right] 
&= \mathbb{E}_{\mathbf{x}}\left[g(\mathbf{Mx})
\sigma'(\mathbf{w}^\top \mathbf{M}^\top  \mathbf{Mx} + \mathbf{w}^\top \mathbf{M}_\perp^\top \mathbf{M}_\perp \mathbf{x}) \mathbf{M}^\top\mathbf{M}\mathbf{x} \right] \nonumber \\
&\quad + \mathbb{E}_{\mathbf{x}}\left[g(\mathbf{Mx})
\sigma'(\mathbf{w}^\top \mathbf{M}^\top  \mathbf{Mx} + \mathbf{w}^\top \mathbf{M}_\perp^\top \mathbf{M}_\perp \mathbf{x})  \mathbf{M}_\perp^\top\mathbf{M}_\perp\mathbf{x})\right] \nonumber \\
&= \underbrace{\mathbb{E}_{\mathbf{Mx}}\left[g(\mathbf{Mx})\mathbb{E}_{\mathbf{M}_\perp\mathbf{x}}\left[
\sigma'(\mathbf{w}^\top \mathbf{M}^\top  \mathbf{Mx} + \mathbf{w}^\top \mathbf{M}_\perp^\top \mathbf{M}_\perp \mathbf{x})\right] \mathbf{M}^\top\mathbf{M}\mathbf{x} \right]}_{\circled{1}} \nonumber \\
&\quad + \underbrace{\mathbb{E}_{\mathbf{Mx}}\left[g(\mathbf{Mx})
 \mathbb{E}_{\mathbf{M}_\perp \mathbf{x}}\left[\sigma'(\mathbf{w}^\top \mathbf{M}^\top  \mathbf{Mx} + \mathbf{w}^\top \mathbf{M}_\perp^\top \mathbf{M}_\perp \mathbf{x})  \mathbf{M}_\perp^\top\mathbf{M}_\perp\mathbf{x})\right] \right]}_{\circled{2}} \nonumber 
\end{align}
so  $ \|   \mathbb{E}_{\mathbf{x}}\left[f(\mathbf{x})
\sigma'(\mathbf{w}^\top \mathbf{x}) \mathbf{x}\right]\|_2 \leq \left\| \circled{1} \right\|_2 +  \left\| \circled{2} \right\|_2$ by the triangle inequality.  First we consider $\circled{1}$. We have
\begin{align}
& \mathbb{E}_{\mathbf{Mx}}\left[g(\mathbf{Mx})\mathbb{E}_{\mathbf{M}_\perp\mathbf{x}}\left[
\sigma'(\mathbf{w}^\top \mathbf{M}^\top  \mathbf{Mx} + \mathbf{w}^\top \mathbf{M}_\perp^\top \mathbf{M}_\perp \mathbf{x})\right] \mathbf{M}^\top\mathbf{M}\mathbf{x} \right] \nonumber \\
 &=  \mathbb{E}_{\mathbf{Mx}}\left[g(\mathbf{Mx})\mathbb{P}_{\mathbf{M}_\perp\mathbf{x}}\left[
\mathbf{w}^\top \mathbf{M}_\perp^\top \mathbf{M}_\perp \mathbf{x} > - \mathbf{w}^\top \mathbf{M}^\top  \mathbf{Mx}\right] \mathbf{M}^\top\mathbf{M}\mathbf{x} \right] \nonumber \\
  &=  \mathbb{E}_{\mathbf{Mx}}\left[g(\mathbf{Mx})\left(\frac{1}{2} + \frac{1}{2}\erf\left( \frac{\mathbf{w}^\top \mathbf{M}^\top  \mathbf{Mx} }{  \sqrt{2}\|\mathbf{M}_{\perp}\mathbf{w}\|_2  } \right) \right)\mathbf{M}^\top\mathbf{M}\mathbf{x} \right]  \label{hnfv} \\
 &= \frac{1}{2}\mathbb{E}_{\mathbf{Mx}}\left[g(\mathbf{Mx}) \mathbf{M}^\top\mathbf{M}\mathbf{x} \right]  + \frac{1}{2}\mathbb{E}_{\mathbf{Mx}}\left[g(\mathbf{Mx}) \erf\left( \frac{\mathbf{w}^\top \mathbf{M}^\top  \mathbf{Mx} }{  \sqrt{2}\|\mathbf{M}_{\perp}\mathbf{w}\|_2  }  \right)\mathbf{M}^\top\mathbf{M}\mathbf{x} \right] \nonumber
\end{align}
where \eqref{hnfv} is due to the Gaussian CDF. Thus by the triangle inequality,
\begin{align}
    \left\|\circled{1}\right\|_2 &\leq  \frac{1}{2}\left\|\mathbb{E}_{\mathbf{Mx}}\left[g(\mathbf{Mx}) \mathbf{M}^\top\mathbf{M}\mathbf{x} \right] \right\|_2 + \frac{1}{2}\left\|\mathbb{E}_{\mathbf{Mx}}\left[g(\mathbf{Mx}) \erf\left( \frac{\mathbf{w}^\top \mathbf{M}^\top  \mathbf{Mx} }{  \sqrt{2}\|\mathbf{M}_{\perp}\mathbf{w}\|_2  }  \right)\mathbf{M}^\top\mathbf{M}\mathbf{x}\right] \right\|_2 \label{dfg} \end{align}
For the first term,
\begin{align}
    \left\|\mathbb{E}_{\mathbf{Mx}}\left[g(\mathbf{Mx}) \mathbf{M}^\top\mathbf{M}\mathbf{x} \right] \right\|_2 &\leq \mathbb{E}_{\mathbf{Mx}}\left[\left\|g(\mathbf{Mx}) \mathbf{M}^\top\mathbf{M}\mathbf{x}\right\|_2   \right] \label{jen1} \\
    &= \mathbb{E}_{\mathbf{Mx}}\left[\left\| \mathbf{M}\mathbf{x}\right\|_2   \right] \nonumber \\
    &\leq  \mathbb{E}_{\mathbf{Mx}}\left[\left\| \mathbf{M}\mathbf{x}\right\|_2^2   \right]^{1/2} \label{jen2} \\
    &= \sqrt{r}  \label{bhg}
\end{align}
where \eqref{jen1} and \eqref{jen2} follow by Jensen's inequality.
For the second term in \eqref{dfg}, we have
\begin{align}
&\left\|\mathbb{E}_{\mathbf{Mx}}\left[g(\mathbf{Mx}) \erf\left( \frac{\mathbf{w}^\top \mathbf{M}^\top  \mathbf{Mx} }{  \sqrt{2}\|\mathbf{M}_{\perp}\mathbf{w}\|_2  }  \right)\mathbf{M}^\top\mathbf{M}\mathbf{x}\right] \right\|_2  \nonumber \\
    &\leq \mathbb{E}_{\mathbf{Mx}}\left[\left\|g(\mathbf{Mx}) \erf\left( \frac{\mathbf{w}^\top \mathbf{M}^\top  \mathbf{Mx} }{  \sqrt{2}\|\mathbf{M}_{\perp}\mathbf{w}\|_2  }  \right)\mathbf{M}^\top\mathbf{M}\mathbf{x}\right\|_2\right]  \label{kiy} \\
    &\leq  \left(\mathbb{E}_{\mathbf{Mx}}\left[\left|g(\mathbf{Mx}) \erf\left( \frac{\mathbf{w}^\top \mathbf{M}^\top  \mathbf{Mx} }{  \sqrt{2}\|\mathbf{M}_{\perp}\mathbf{w}\|_2  }  \right) \right|^2 \right] \mathbb{E}_{\mathbf{Mx}}\left[ \left\|\mathbf{M}^\top\mathbf{M}\mathbf{x}\right\|_2^2\right]\right)^{1/2}  \label{fv} \\
    &= {\sqrt{r}} \mathbb{E}_{\mathbf{Mx}}\left[\left|g(\mathbf{Mx}) \erf\left( \frac{\mathbf{w}^\top \mathbf{M}^\top  \mathbf{Mx} }{  \sqrt{2}\|\mathbf{M}_{\perp}\mathbf{w}\|_2  }  \right) \right|^2 \right]^{1/2}  \nonumber \\
    &\leq {\sqrt{r}} \mathbb{E}_{\mathbf{Mx}}\left[\left( \frac{\sqrt{\pi}\mathbf{w}^\top \mathbf{M}^\top  \mathbf{Mx} }{ 2\|\mathbf{M}_{\perp}\mathbf{w}\|_2  }  \right) ^2 \right]^{1/2}  \label{xs} \\
    &= \frac{\sqrt{\pi r}}{2} \frac{\|\mathbf{Mw}\|_2}{\|\mathbf{M}_\perp\mathbf{w}\|_2} \label{xa}
\end{align}
where \eqref{kiy} follows by Jensen's inequality, \eqref{fv} follows by the Cauchy-Schwarz inequality, and \eqref{xs} follows since $|\erf(x)|\leq |\sqrt{\pi/2} \; x|$. Combining \eqref{dfg}, \eqref{bhg} and \eqref{xa} yields
\begin{align}
    \left\|\circled{1}\right\|_2 \leq \frac{\sqrt{r}}{2}\left( 1 + \frac{\sqrt{\pi}}{2} \frac{\|\mathbf{Mw}\|_2}{\|\mathbf{M}_\perp\mathbf{w}\|_2} \right)
\end{align}

Next we consider $\circled{2}$. We have
\begin{align}
&\mathbb{E}_{\mathbf{Mx}}\left[g(\mathbf{Mx})
 \mathbb{E}_{\mathbf{M}_\perp \mathbf{x}}\left[\sigma'(\mathbf{w}^\top \mathbf{M}^\top  \mathbf{Mx} + \mathbf{w}^\top \mathbf{M}_\perp^\top \mathbf{M}_\perp \mathbf{x})  \mathbf{M}_\perp^\top\mathbf{M}_\perp\mathbf{x})\right] \right] \nonumber \\
 &= \frac{1}{\sqrt{2\pi}} \mathbb{E}_{\mathbf{Mx}}\left[g(\mathbf{Mx})
 \exp\left( -\frac{(\mathbf{w}^\top \mathbf{M}^\top \mathbf{Mx})^2 }{2} \right) 
  \right] \frac{\mathbf{M}_\perp^\top \mathbf{M}_\perp \mathbf{w}}{\| \mathbf{M}_\perp \mathbf{w}\|_2 } \label{hnv}  
\end{align}
by Lemma \ref{lem:exp-perp}, thus
\begin{align}
\left\|\circled{2} \right\|_2 = \frac{1}{\sqrt{2\pi}} \left|\mathbb{E}_{\mathbf{Mx}}\left[g(\mathbf{Mx})
 \exp\left( -\frac{(\mathbf{w}^\top \mathbf{M}^\top \mathbf{Mx})^2 }{2} \right) 
  \right]\right|.
  \end{align}
  Next we upper bound $\left|\mathbb{E}_{\mathbf{Mx}}\left[g(\mathbf{Mx})
 \exp\left( -\frac{(\mathbf{w}^\top \mathbf{M}^\top \mathbf{Mx})^2 }{2} \right) 
  \right]\right|$. We have
\begin{align}
     &\left|\mathbb{E}_{\mathbf{Mx}}\left[g(\mathbf{Mx})
 \exp\left( -\frac{(\mathbf{w}^\top \mathbf{M}^\top \mathbf{Mx})^2 }{2} \right) 
  \right]\right| \nonumber \\
  &\leq  \mathbb{E}_{\mathbf{Mx}}\left[
 \exp\left( -\frac{(\mathbf{w}^\top \mathbf{M}^\top \mathbf{Mx})^2 }{2} \right) 
  \right] \label{tgd} \\
  &=  \int_{\mathbb{R}^r} \frac{1}{(2\pi)^{r/2} } \exp\left(- \frac{( \mathbf{w}^\top \mathbf{M}^\top \mathbf{Mx} )^2}{2} - \frac{\mathbf{x}^\top \mathbf{M}^\top \mathbf{Mx}}{2} \right) d \mathbf{Mx} \label{wes} \\
  &= \int_{\mathbb{R}^r} \frac{1}{(2\pi)^{r/2} } \exp\left(- \frac{\mathbf{x}^\top \mathbf{M}^\top  \left(  \mathbf{I}_r + \mathbf{Mww}^\top \mathbf{M}^\top \right) \mathbf{Mx}}{2} \right) d \mathbf{Mx} \nonumber \\
  &= {\text{det}( \mathbf{I}_r + \mathbf{Mww}^\top \mathbf{M}^\top )^{1/2} }\nonumber \\
  &\quad \times \int_{\mathbb{R}^r} \frac{1}{(2\pi)^{r/2} \text{det}( \mathbf{I}_r + \mathbf{Mww}^\top \mathbf{M}^\top )^{1/2} } \exp\left(- \frac{\mathbf{x}^\top \mathbf{M}^\top  \left(  \mathbf{I}_r + \mathbf{Mww}^\top \mathbf{M}^\top \right) \mathbf{Mx}}{2} \right) d \mathbf{Mx} \label{gz} \\
  &= {\text{det}( \mathbf{I}_r + \mathbf{Mww}^\top \mathbf{M}^\top )^{1/2} } \nonumber \\
  &= \left( 1 + \|\mathbf{Mw}\|_2^2 \right)^{1/2}\label{wqp} 
\end{align}
where \eqref{tgd} follows since $\exp()$ is positive, \eqref{wes} follows since $\mathbf{w}^\top \mathbf{M}^\top \mathbf{Mx}$ is a Gaussian random variable with mean zero and variance $\|\mathbf{Mw}\|_2^2$, and \eqref{gz} follows due to  the multivariate normal distribution, and \eqref{wqp} follows by the matrix determinant lemma. Therefore,
\begin{align}
  \left\| \circled{2} \right\|_2 &\leq \left( \frac{1 + \|\mathbf{Mw}\|_2^2 }{2\pi } \right)^{1/2},
\end{align}
completing the proof.
\end{proof}


\subsection{Finite-task and finite-sample concentration results}

\begin{lemma}[Initialization I] \label{lem:init}




    For any $\delta\in (0,1)$, define the set 
    \begin{align}
        \mathcal{G}_{\mathbf{w}}(\delta) \coloneqq \big\{&\mathbf{w}\in \mathbb{R}^d :  \|\mathbf{M}\mathbf{w}\|_2 \leq c \nu_{\mathbf{w}}( \sqrt{r} + \sqrt{\log(m/\delta)}), \nonumber \\
        &\quad c\nu_{\mathbf{w}}(\sqrt{d-r} - \sqrt{\log(m/\delta)}) \leq \|\mathbf{M}_\perp\mathbf{w} \|_2 \leq c \nu_{\mathbf{w}}(\sqrt{d-r} + \sqrt{\log(m/\delta)}), \nonumber \\
        &\quad c\nu_{\mathbf{w}}(\sqrt{d} - \sqrt{\log(m/\delta)}) \leq \|\mathbf{w} \|_2 \leq c\nu_{\mathbf{w}}(\sqrt{d} + \sqrt{\log(m/\delta)}) \big\}
    \end{align} for an absolute constant $c$.
    Then with probability at least $1-\delta$, $\mathbf{w}_{j}\in \mathcal{G}_{\mathbf{w}}(\delta)$ for all $j\in [m]$.
\end{lemma}

\begin{proof}
Since each $\mathbf{w}_j\sim \mathcal{N}(\mathbf{0}_d,\nu_{\mathbf{w}}^2\mathbf{I}_d)$, each $\|\mathbf{M}\mathbf{w}_j\|_2$, $\|\mathbf{M}_\perp\mathbf{w}_j\|_2$, and $\|\mathbf{w}_j\|_2$ are sub-Gaussian with parameters $\nu_{\mathbf{w}}\sqrt{r}$, $\nu_{\mathbf{w}}\sqrt{d-r}$, and $\nu_{\mathbf{w}}\sqrt{d}$. That is, 
\begin{align}
    \mathbb{P}_{\mathbf{w}_{j}}[| \|\mathbf{Mw}_{j}\|_2 - \nu_{\mathbf{w}}\sqrt{r}| \leq t] \leq e^{-c' t^2/\nu_{\mathbf{w}}^2}
\end{align}
for any $t>0$, and likewise for $ \|\mathbf{M}_\perp\mathbf{w}_{j}\|_2$ and  $\|\mathbf{w}_{j}\|_2$ (with $ \nu_{\mathbf{w}}\sqrt{r}$ replaced by $ \nu_{\mathbf{w}}\sqrt{d-r}$ and $ \nu_{\mathbf{w}}\sqrt{d}$, respectively).
Choosing $t = c'' \nu_{\mathbf{w}}\sqrt{\log(m/\delta)}$ and union bounding over all $j\in [m]$ completes the proof.
\end{proof}

\begin{lemma}[Initialization II] \label{lem:init2}




Suppose $m >r$. For an absolute constant $c$ and any $\delta\in (0,1)$,
    \begin{align}
        \sigma_{\min}(\mathbf{M}\mathbf{W}^0) \geq \nu_{\mathbf{w}}\sqrt{m}\left(1 - c\frac{\sqrt{r} + \sqrt{\log(1/\delta)}}{\sqrt{m}}  \right)
    \end{align}
    with probability at least $1-\delta$.
\end{lemma}

\begin{proof}
The result follows from the fact that each row of $\mathbf{W}^0$ is drawn independently from $\mathcal{N}(\mathbf{0}_d,\nu_{\mathbf{w}}^2\mathbf{I}_d)$, so each of the $r$ rows of $\mathbf{M}\mathbf{W}^0$ are drawn i.i.d. from $\mathcal{N}(\mathbf{0}_m,\nu_{\mathbf{w}}^2\mathbf{I}_m)$ (recalling that the rows of $\mathbf{M}$ are orthogonal, so $\mathbf{u}_i^\top \mathbf{w}_j$ and $\mathbf{u}_{i'}^\top \mathbf{w}_j$ are independent for any two distinct rows $\mathbf{u}_i$ and $\mathbf{u}_{i'}$ of $\mathbf{M}$). A standard (sub-)Gaussian matrix concentration  inequality yields the result (e.g. Equation 4.21 in \cite{vershynin2018high}).
\end{proof}

\begin{lemma}[Initialization III] \label{lem:init3}
For an absolute constant $c$ and any $\delta \in (0,1)$,
    \begin{align}
        \|\mathbf{W}^0\|_2 \leq \nu_{\mathbf{w}}\sqrt{m}\left(1 + \frac{\sqrt{d}+\sqrt{\log(1/\delta)}}{\sqrt{m}}\right)
    \end{align}
    for some absolute constant with probability at least $1-\delta$.
\end{lemma}
\begin{proof}
    As in Lemma \ref{lem:init2}, the proof follows by standard (sub-)Gaussian matrix concentration (e.g. Equation 4.21 in \cite{vershynin2018high}).
\end{proof}



\begin{lemma}\label{lem:all}
    For any $\delta\in (0,1)$ define the event $E_{\mathbf{w}}(\delta)$ as follows: 
\begin{align}
    E_{\mathbf{w}}(\delta) :=\Bigg\{ &\mathbf{w}_j^0 \in \mathcal{G}_{\mathbf{x}}(\delta), \quad 
     \sigma_{\min}(\mathbf{M}\mathbf{W}^0) \geq \nu_{\mathbf{w}}\sqrt{m}\left(1 - c\frac{\sqrt{r} - \sqrt{\log(1/\delta)}}{\sqrt{m}}  \right),\nonumber \\
     &\|\mathbf{W}^0\|_2 \leq \nu_{\mathbf{w}}\sqrt{m}\left(1 + \frac{\sqrt{d}+\sqrt{\log(1/\delta)}}{\sqrt{m}}\right)
    \Bigg\}.
\end{align}
Then $\mathbb{P}(E_{\mathbf{w}}(\delta))\geq 1 - 3\delta$. 
\end{lemma}

\begin{proof}
    The proof follows immediately from Lemmas \ref{lem:init}, \ref{lem:init2}, \ref{lem:init3} and a union bound.
\end{proof}

Next we begin to analyze the first gradient-based update of the algorithm.
Throughout, let $ \hat{\mathcal{D}}_{i,\mathbf{a}} = \{(\mathbf{x}_{i,k}, f_i(\mathbf{x}_{i,k}))\}_{k=1}^{n_1}$ and $ \hat{\mathcal{D}}_{i,\mathbf{W}} = \{(\mathbf{x}_{i,l}, f_i(\mathbf{x}_{i,l}))\}_{l=1}^{n_2}$, where all samples are drawn i.i.d. from $\mathcal{D}_i$, for each $i \in [T]$.

\begin{lemma}\label{lem:a1}
    Let $\lambda_{\mathbf{a}} = \frac{1}{\eta}$ and $\mathbf{x}\sim \mathcal{N}(\mathbf{0}_d,\mathbf{I}_d)$. On the first iteration of the multitask learning algorithm described in Section \ref{sec:alg}, the locally updated head for task $i$ is:
    \begin{align}
        \mathbf{a}_{i}^1 = \frac{\eta}{n_1} 
        \sum_{k=1}^{n_1} f_{i}(\mathbf{x}_{i,k}) {\sigma}(\mathbf{W}^0\mathbf{x}_{i,k}) 
    \end{align}
    for all tasks $i\in [T]$.
\end{lemma}

\begin{proof}
Since $\mathbf{a}_{i}^0 = \mathbf{0}_m$ for all $i$,
$(\mathbf{a}_{i}^0)^\top  \sigma(\mathbf{W}^0\mathbf{x} + \mathbf{b}^0)=0$ for all $\mathbf{x}$ and $i$.
Therefore, 
$\max(1 - f_{i}(\mathbf{x})( \mathbf{a}_i^0)^\top  \sigma(\mathbf{W}^0\mathbf{x} + \mathbf{b}^0) , 0) = 1 - f_{i}(\mathbf{x}) (\mathbf{a}_i^0)^\top  \sigma(\mathbf{W}^0\mathbf{x} )$, and
\begin{align} 
\mathbf{a}_{i}^1 &= \mathbf{a}_i^0 - \eta \nabla_{\mathbf{a}} \hat{\mathcal{L}}_i(\mathbf{W}^0,\mathbf{b}^0,\mathbf{a}_i^0; \hat{\mathcal{D}}_{i,\mathbf{a}}) \nonumber \\
&= 
-\frac{\eta}{n_1} \sum_{k=1}^{n_1}\nabla_{\mathbf{a}} ( \max(1 - f_{i}(\mathbf{x}_{i,k}) (\mathbf{a}_i^0)^\top  \sigma(\mathbf{W}^0\mathbf{x}_{i,k} + \mathbf{b}^0) , 0) )\nonumber \\
&= 
-\frac{\eta}{n_1} \sum_{k=1}^{n_1}\nabla_{\mathbf{a}} (1 - f_{i}(\mathbf{x}_{i,k}) (\mathbf{a}_i^0)^\top  \sigma(\mathbf{W}^0\mathbf{x}_{i,k} + \mathbf{b}^0)  )\nonumber \\
&= 
\frac{\eta}{n_1} \sum_{k=1}^{n_1} f_{i}(\mathbf{x}_{i,k})  \sigma(\mathbf{W}^0\mathbf{x}_{i,k})  \nonumber 
\end{align}
where in the last equality we have used $\mathbf{b}^0= \mathbf{0}_m$ by choice of initialization.
\end{proof}

    
    

Next, we substitute the updated heads in the global empirical loss. Since the gradient computation for the update of $\mathbf{W}$ does not backpropagate through the update of the heads,  we use $\bar{\sigma}(\cdot)$ to denote the stop-gradient ReLU activation, meaning all model parameters inside are treated as constants for the purposes of later gradient updates. In particular, from Lemma \ref{lem:a1} we have
\begin{align}
        \mathbf{a}_{i}^1 = \frac{\eta}{n_1} 
        \sum_{k=1}^{n_1} f_{i}(\mathbf{x}_{i,k}) {\sigma}(\mathbf{W}^0\mathbf{x}_{i,k}) \label{jkj}
    \end{align}
    for all $i \in [T]$.
\begin{lemma}\label{lem:upd_loss}
    After updating the heads on the first iteration, the empirical loss averaged across the task datasets $\{\hat{\mathcal{D}}_{i,\mathbf{W}}\}_{i=1}^T$ used for updating the neuron weights is given by
        \begin{align}
        \hat{\mathcal{L}}(\mathbf{W}^0,& \mathbf{b}^0, \{\mathbf{a}_{i}^1 \}_{i=1}^T;\{\hat{\mathcal{D}}_{i,\mathbf{W}}\}_{i=1}^T)\nonumber \\
        &\coloneqq \frac{1}{T}\sum_{i=1}^T \hat{\mathcal{L}}_i(\mathbf{W}^0, \mathbf{b}^0, \mathbf{a}_{i}^1 ;\hat{\mathcal{D}}_{i,\mathbf{W}}) \nonumber \\
&=  1 -   
\frac{\eta}{T n_1 n_2} \sum_{i=1}^T \sum_{l=1}^{n_2}  \sum_{k=1}^{n_1} 
f_{i}(\mathbf{x}_{i,l})  f_{i}(\mathbf{x}_{i,k})\bar{\sigma}(\mathbf{W}^0\mathbf{x}_{i,k})^\top \sigma(\mathbf{W}^0\mathbf{x}_{i,l})   \nonumber \\
&- \frac{1}{T n_2} \sum_{i=1}^T \sum_{l=1}^{n_2}    \chi\left\{ \eta f_{i}(\mathbf{x}_{i,l}) \left( \frac{1}{n_1} \sum_{k=1}^{n_1} f_{i}(\mathbf{x}_{i,k}) \bar{\sigma}(\mathbf{W}^0\mathbf{x}_{i,k} ) \right)^\top \sigma(\mathbf{W}^0\mathbf{x}_{i,l}) >1\right\}   \nonumber \\
&\quad \quad \quad \quad \quad \quad \; \times  \left( 1 - \eta f_{i}(\mathbf{x}_{i,l}) \left(\frac{1}{n_1} \sum_{k=1}^{n_1}   f_{i}(\mathbf{x}_{i,k}) \bar{\sigma}(\mathbf{W}^0\mathbf{x}_{i,k})\right)^\top \sigma(\mathbf{W}^0\mathbf{x}_{i,l}) \right)  \nonumber 
    \end{align}
\end{lemma}

\begin{proof}
First, by the fact that $\mathbf{b}^0 = \mathbf{0}_m$ due to the choice of initialization, we have
\begin{align}
    &\hat{\mathcal{L}}(\mathbf{W}^0,\mathbf{b}^0,\{\mathbf{a}^1_{i}\}_{i=1}^T; \{\hat{\mathcal{D}}_{i,\mathbf{W}}\}_{i=1}^T) \nonumber \\
   &= \frac{1}{T} \sum_{i=1}^T \hat{\mathcal{L}}_i(\mathbf{W}^0,\mathbf{b}^0,\{\mathbf{a}_{i}^1\}_{i=1}^T; \hat{\mathcal{D}}_{i,\mathbf{W}} )
    \nonumber \\
    &=\frac{1}{Tn_2 } \sum_{i=1}^T \sum_{l=1}^{n_2}   \max\left( 1- f_i(\mathbf{x}_{i,l}) (\mathbf{a}_i^1)^\top\sigma(\mathbf{W}^0\mathbf{x}_{i,l}), 0 \right) 
    \nonumber \\
    &= \frac{1}{Tn_2 } \sum_{i=1}^T \sum_{l=1}^{n_2}   \chi\left\{ f_i(\mathbf{x}_{i,l})  (\mathbf{a}_i^1)^\top\sigma(\mathbf{W}^0\mathbf{x}_{i,l}) < 1   \right\}
    \left( 1- f_i(\mathbf{x}_{i,l}) (\mathbf{a}_i^1)^\top\sigma(\mathbf{W}^0\mathbf{x}_{i,l}) \right)  \nonumber\\
    &= 1 - \frac{1}{Tn_2 } \sum_{i=1}^T \sum_{l=1}^{n_2}  
     f_i(\mathbf{x}_{i,l}) (\mathbf{a}_i^1)^\top\sigma(\mathbf{W}^0\mathbf{x}_{i,l})  \nonumber \\
    &\quad -  \frac{1}{Tn_2 } \sum_{i=1}^T \sum_{l=1}^{n_2} \chi\left\{ f_i(\mathbf{x}_{i,l})  (\mathbf{a}_i^1)^\top\sigma(\mathbf{W}^0\mathbf{x}_{i,l}) > 1   \right\} \left( 1- f_i(\mathbf{x}_{i,l}) (\mathbf{a}_i^1)^\top\sigma(\mathbf{W}^0\mathbf{x}_{i,l}) \right) \nonumber 
\end{align}
Substituting the value of $\mathbf{a}_i^1$ from \eqref{jkj} completes the proof.
\end{proof}

Next we show that after one update of the heads, the model predictions are still close to zero, so $\max()$ in the hinge loss is mostly inactive.
\begin{lemma}
    \label{lem:err_a} Suppose $\nu_{\mathbf{w}}^2 = O( \frac{1}{\eta^2 dm\log(T) (d+m)} )$.
    With probability at least $1 -\delta$, 
    \begin{align}
    &\frac{\eta}{T n_2} \sum_{i=1}^T \sum_{l=1}^{n_2}    \chi\left\{ \eta f_{i}(\mathbf{x}_{i,l}) \left( \frac{1}{n_1} \sum_{k=1}^{n_1} f_{i}(\mathbf{x}_{i,k}) {\sigma}'(\mathbf{W}^0\mathbf{x}_{i,k} ) \right)^\top \sigma(\mathbf{W}^0\mathbf{x}_{i,l}) >1\right\}\nonumber \\
    &=\eta O\left(  \exp \left( -\frac{ c }{\eta^2  \nu_{\mathbf{w}}^2 {m\log(T/\delta)}({d} + {\log(m/\delta)}) ({d} + {m} )} \right)\right) + \eta \; O\left(  \frac{  \sqrt{\log(1/\delta)} }{ \sqrt{Tn_2} } \right) 
\end{align}
for an absolute constant $c$.
\end{lemma}

\begin{proof} 
Consider any fixed $\mathbf{W}^0$ satisfying $E_{\mathbf{w}}(\delta_1)$ for $\delta_1\in (0,1)$, 
which occurs with probability at least $1-3 \delta_1$ by Lemma \ref{lem:all}. 
For ease of notation we replace $\mathbf{w}_j^0$ with $\mathbf{w}_j.$ 
Recall that 
\begin{align}
  \eta f_{i}(\mathbf{x}_{i,l}) \left( \frac{1}{n_1} \sum_{k=1}^{n_1} f_{i}(\mathbf{x}_{i,k}) {\sigma}'(\mathbf{W}^0\mathbf{x}_{i,k} ) \right)^\top \sigma(\mathbf{W}^0\mathbf{x}_{i,l})  &=    f_{i}(\mathbf{x}_{i,l}) (\mathbf{a}_i^1)^\top \sigma(\mathbf{W}^0\mathbf{x}_{i,l}) 
\end{align}
by the computation of $\mathbf{a}_i^1$ in Lemma \ref{lem:a1}.
For any fixed $f_i$ and fixed $\mathbf{a}_i^1$, \\
$  \chi\left\{  f_{i}(\mathbf{x}_{i,l}) (\mathbf{a}_i^1)^\top \sigma(\mathbf{W}^0\mathbf{x}_{i,l}) >1\right\}$ is a Bernoulli random variable with parameter
\begin{align}
    &\mathbb{P}_{\mathbf{x_{i,l}}}\left(  f_{i}(\mathbf{x}_{i,l}) (\mathbf{a}_i^1)^\top \sigma(\mathbf{W}^0\mathbf{x}_{i,l}) >1 \right) \nonumber \\ 
    &\leq \mathbb{P}_{\mathbf{x}_{i,l}}\left(  \left| (\mathbf{a}_i^1)^\top \sigma(\mathbf{W}^0\mathbf{x}_{i,l}) \right| > {1}{ } \right) \label{aaq}\\
    &\leq \mathbb{P}_{\mathbf{x}_{i,l}}\left(  \left| (\mathbf{a}_i^1)^\top \sigma(\mathbf{W}^0\mathbf{x}_{i,l}) - \mathbb{E}_{\mathbf{x}_{i,l}}[(\mathbf{a}_i^1)^\top \sigma(\mathbf{W}^0\mathbf{x}_{i,l})] \right| > {1}{} -  \mathbb{E}_{\mathbf{x}_{i,l}}[(\mathbf{a}_i^1)^\top \sigma(\mathbf{W}^0\mathbf{x}_{i,l})]  \right) \label{aaaq} \\
    &\leq 2 \exp \left( -\frac{ \left(1 -  \mathbb{E}_{\mathbf{x}_{i,l}}[(\mathbf{a}_i^1)^\top \sigma(\mathbf{W}^0\mathbf{x}_{i,l})] \right)^2 }{  c\|\mathbf{a}_i^1\|_2^2 \|\mathbf{W}_0\|_2^2 } \right) \nonumber \\
    &\leq 2 \exp \left( -\frac{ \left( 1 -  \|\mathbf{a}_i^1\|_2 \|\mathbf{W}^0\|_2 \right)^2 }{  c\|\mathbf{a}_i^1\|_2^2 \|\mathbf{W}_0\|_2^2 } \right) \nonumber \\
    &\leq 2 \exp \left( -\frac{  1 -  2\|\mathbf{a}_i^1\|_2 \|\mathbf{W}^0\|_2 }{  c\|\mathbf{a}_i^1\|_2^2 \|\mathbf{W}_0\|_2^2 } \right) \nonumber \\
    &=: \gamma_i \label{gmm}
\end{align}
for some absolute constant $c$, where \eqref{aaq} follows since $ f_{i}(\mathbf{x}_{i,l}) \in \{-1,1\}$, and \eqref{aaaq} follows since $(\mathbf{a}_i^1)^\top \sigma(\mathbf{W}^0\mathbf{x}_{i,l}) - \mathbb{E}_{\mathbf{x}_{i,l}}[(\mathbf{a}_i^1)^\top \sigma(\mathbf{W}^0\mathbf{x}_{i,l})]$ is sub-Gaussian with mean $0$ and variance $O(\|\mathbf{a}_i^1\|_2^2 \|\mathbf{W}_0\|_2^2 )$. Next, since for fixed $\mathbf{a}_i^1$ and $f_i$, the random variables $\{\chi\{ \eta f_{i}(\mathbf{x}_{i,l}) (\mathbf{a}_i^1)^\top \sigma(\mathbf{W}^0\mathbf{x}_{i,l}) > 1 \} \}_{l=1}^{n_1}$ are i.i.d., we have by Hoeffding's inequality
\begin{align}
&\mathbb{P}_{\mathbf{x}_{i,l}}\left( \left|\frac{1}{n_2} \sum_{l=1}^{n_2} \chi\{  f_{i}(\mathbf{x}_{i,l}) (\mathbf{a}_i^1)^\top \sigma(\mathbf{W}^0\mathbf{x}_{i,l}) > 1 \} - 
\mathbb{P}_{\mathbf{x}}( f_{i}(\mathbf{x}) (\mathbf{a}_i^1)^\top \sigma(\mathbf{W}^0\mathbf{x}) > 1 )\right| > t
\right)\nonumber \\
&\leq  2\exp\left( -2 n_2t^2 \right)
\end{align}
for any $t>0$. So $\frac{1}{n_2} \sum_{l=1}^{n_2} \chi\{  f_{i}(\mathbf{x}_{i,l}) (\mathbf{a}_i^1)^\top \sigma(\mathbf{W}^0\mathbf{x}_{i,l}) > 1 \} - 
\mathbb{P}_{\mathbf{x}}(\eta f_{i}(\mathbf{x}) (\mathbf{a}_i^1)^\top \sigma(\mathbf{W}^0\mathbf{x}) > 1 )$ is sub-Gaussian with mean zero and variance proxy $\frac{1}{n_2}$. Also, each random variable in \\
$\{\frac{1}{n_2} \sum_{l=1}^{n_2} \chi\{  f_{i}(\mathbf{x}_{i,l}) (\mathbf{a}_i^1)^\top \sigma(\mathbf{W}^0\mathbf{x}_{i,l}) > 1 \} - 
\mathbb{P}_{\mathbf{x}}(\eta f_{i}(\mathbf{x}) (\mathbf{a}_i^1)^\top \sigma(\mathbf{W}^0\mathbf{x}) > 1 ) \}_{i=1}^T $ is independent, so again by Hoeffding's inequality,
\begin{align}
   & \mathbb{P}_{\mathbf{x}_{i,l}}\left( \left|\frac{1}{Tn_2}\sum_{i=1}^T \sum_{l=1}^{n_2} \chi\{  f_{i}(\mathbf{x}_{i,l}) (\mathbf{a}_i^1)^\top \sigma(\mathbf{W}^0\mathbf{x}_{i,l}) > 1 \} - \frac{1}{T}\sum_{i=1}^T
\mathbb{P}_{\mathbf{x}}( f_{i}(\mathbf{x}) (\mathbf{a}_i^1)^\top \sigma(\mathbf{W}^0\mathbf{x}) > 1 )\right| > t
\right)\nonumber\\
&\leq 2\exp\left( -{2Tn_2 t^2} \right)
\end{align}
Set $t = O\left(\frac{\sqrt{\log(1/\delta_2)}}{\sqrt{Tn_2}}\right)$, then we have 
\begin{align}
    &\frac{1}{Tn_2}\sum_{i=1}^T \sum_{l=1}^{n_2} \chi\{  f_{i}(\mathbf{x}_{i,l}) (\mathbf{a}_i^1)^\top \sigma(\mathbf{W}^0\mathbf{x}_{i,l}) > 1 \}\nonumber \\
    &= \frac{1}{T}\sum_{i=1}^T \mathbb{P}_{\mathbf{x}}( f_{i}(\mathbf{x}) (\mathbf{a}_i^1)^\top \sigma(\mathbf{W}^0\mathbf{x}) > 1 ) + O\left( \frac{\sqrt{\log(1/\delta)}}{\sqrt{Tn_2}}\right) \nonumber \\
    &\leq \frac{1}{T}\sum_{i=1}^T \gamma_i + O\left( \frac{\sqrt{\log(1/\delta_2)}}{\sqrt{Tn_2}}\right)  \label{30}
\end{align}
with probability at least $1-\delta_2$ over the sampling of $\{\mathbf{x}_{i,l}\}_{i,l}$. It remains to bound each $\gamma_i$. 
Next, for any $i\in [T]$, we have 
\begin{align}
    \|\mathbf{a}_i^1\|_2 &= \left\|\frac{\eta}{n_1} \sum_{k=1}^{n_1} f_{i}(\mathbf{x}_{i,k}) {\sigma}(\mathbf{W}^0\mathbf{x}_{i,k} ) \right\|_2
    \leq \eta \sqrt{m } \max_{j\in[m]} \left| \frac{1}{n_1} \sum_{k=1}^{n_1} f_{i}(\mathbf{x}_{i,k}) {\sigma}(\mathbf{w}_j^\top\mathbf{x}_{i,k} )\right|
\end{align}
Each $ f_{i}(\mathbf{x}_{i,k}) {\sigma}(\mathbf{w}_j^\top\mathbf{x}_{i,k} ) $ is sub-Gaussian with mean $O(\|\mathbf{w}_j\|_2)$ and variance $O(\|\mathbf{w}_j\|_2^2)$. Also, for fixed $f_i$, the random variables $ \{f_{i}(\mathbf{x}_{i,k}) {\sigma}(\mathbf{w}_j^\top\mathbf{x}_{i,k} )\}_{k=1}^{n_1}$ are independent. So, we have $$\|\mathbf{a}_i^1\|_2 = O(\eta \sqrt{m} t \max_{j\in[m]}\|\mathbf{w}_j\|_2 )$$ with probability at least $1 - e^{-t^2}$ for any $t>0$. Union bounding over all $i\in [T]$ and setting $t = \Theta(\sqrt{\log(T)} + \sqrt{\log(1/\delta_3)})$ yields $$\max_{i\in [T]}\|\mathbf{a}_i^1\|_2= O (\eta  \sqrt{m \log(T/\delta_3)} \max_{j\in[m]}\|\mathbf{w}_j\|_2 )$$ with probability at least $1-\delta_3$. Now applying Lemma \ref{lem:init} and the fact that $E_{\mathbf{w}}(\delta)$ holds results in $\max_{i\in [T]}\|\mathbf{a}_i^1\|_2= O (\eta  \nu_{\mathbf{w}}\sqrt{m\log(T/\delta_3)}(\sqrt{d} + \sqrt{\log(m/\delta_1)}) )$
and \begin{align}
    \max_{i\in [T]}\|\mathbf{a}_i^1\|_2\|\mathbf{W}^0\|_2 &= O\left(  \eta  \nu_{\mathbf{w}}\sqrt{m\log(T/\delta_3)}(\sqrt{d} + \sqrt{\log(m/\delta_1)}) (\sqrt{d} + \sqrt{m} + \sqrt{\log(1/\delta_1)
    } )\right)\nonumber \\
    &=  O\left(  \eta  \nu_{\mathbf{w}}\sqrt{m\log(T/\delta_3)}(\sqrt{d} + \sqrt{\log(m/\delta_1)}) (\sqrt{d} + \sqrt{m} )\right)
\end{align}
with probability at least $1-\delta_3-3\delta_1$. 
Set $\delta_3+3\delta_1 + \delta_2 \leq \delta$,
then from \eqref{gmm} and using $\nu_{\mathbf{w}} = O( \frac{1}{ \eta  \sqrt{m\log(T/\delta)}(\sqrt{d} + \sqrt{\log(m/\delta)}) (\sqrt{d} + \sqrt{m} )} )$, we have $\max_{i\in [T]}\|\mathbf{a}_i^1\|_2\|\mathbf{W}^0\|_2 = O(1)$ and
\begin{align}
    \frac{1}{T}\sum_{i=1}^T\gamma_i \leq \max_{i\in [T]}\gamma_i &= \frac{2}{T} \sum_{i=1}^T \exp\left( -\frac{  1 -  2\|\mathbf{a}_i^1\|_2 \|\mathbf{W}^0\|_2 }{  c\|\mathbf{a}_i^1\|_2^2 \|\mathbf{W}_0\|_2^2 } \right)\nonumber \\
    &\leq 2 \exp\left( -\frac{  1 -  2 \max_{i\in [T]}\|\mathbf{a}_i^1\|_2 \|\mathbf{W}^0\|_2 }{  c\max_{i\in [T]}\|\mathbf{a}_i^1\|_2^2 \|\mathbf{W}_0\|_2^2 } \right) \nonumber \\
    &\leq 2 c'\exp\left( -\frac{  1}{  c\max_{i\in [T]}\|\mathbf{a}_i^1\|_2^2 \|\mathbf{W}_0\|_2^2 } \right) \nonumber \\
    &\leq 2 c'  \exp \left( -\frac{ c'' }{\eta^2  \nu_{\mathbf{w}}^2 {m\log(T/\delta)}({d} + {\log(m/\delta)}) ({d} + {m} )} \right)\nonumber 
\end{align}
with probability at least $1-\delta$,
completing the proof in light of \eqref{30}.
\end{proof}

\begin{lemma}\label{lem:emp_grad}
  For any $\delta \in (0,1)$, with probability at least $1 -  \delta$,   for all neurons $j\in [m]$,  the gradient used to compute $\mathbf{w}_{j}^1$ 
    satisfies
\begin{align}
    \Bigg\|\nabla_{\mathbf{w}_j} &\hat{\mathcal{L}}( \mathbf{W}^0, \mathbf{b}^0, \{\mathbf{a}_{i}^1\}_{i=1}^T; \{\hat{\mathcal{D}}_{i,\mathbf{W}}\}_{i=1}^T )\nonumber \\
    &\quad \quad \quad + \frac{\eta}{T n_1 n_2} \sum_{i=1}^T \sum_{l=1}^{n_2}  \sum_{k=1}^{n_1}
f_{i}(\mathbf{x}_{i,l})  f_{i}(\mathbf{x}_{i,k})\sigma'(\mathbf{w}_j^\top\mathbf{x}_{i,k}) \sigma'(\mathbf{w}_j^\top \mathbf{x}_{i,l}) \mathbf{x}_{i,l}\mathbf{x}_{i,k}^\top\mathbf{w}_j^0\Bigg\|_2 \nonumber \\
&=   
 O\Bigg( \eta \nu_{\mathbf{w}} (\sqrt{d} + \sqrt{\log(m/\delta)}) \sqrt{d\log(Tn_2/\delta)} \left(1 + {\frac{ \sqrt{\log(T/\delta)}}{\sqrt{n_1}}}\right) \nonumber \\
     &\quad \quad \quad \quad \times \left(   \exp \left( -\frac{ c }{\eta^2  \nu_{\mathbf{w}}^2 {m\log(T/\delta)}({d} + {\log(m/\delta)}) ({d} + {m} )} \right) + \frac{\sqrt{\log(1/\delta)}}{\sqrt{Tn_2}}\right) \Bigg)\nonumber
\end{align}
where $c$ is an absolute constant and $\sigma'(x) = \chi\{x>0\}$ denotes the derivative of the ReLU.
\end{lemma}

\begin{proof}
Consider any fixed $\mathbf{W}^0$ satisfying $E_{\mathbf{w}}(\delta_1)$ for $\delta_1\in (0,1)$, 
which occurs with probability at least $1-3 \delta_1$ by Lemma \ref{lem:all}. 
For ease of notation we write $\mathbf{w}_j =\mathbf{w}_j^0$.  
Using Lemma \ref{lem:upd_loss} and the chain rule, 
\begin{align}
 &\nabla_{\mathbf{w}_j} \hat{\mathcal{L}}( \mathbf{W}^0, \mathbf{b}^0, \{\mathbf{a}_{i}^1\}_{i=1}^T; \{\hat{\mathcal{D}}_{i,\mathbf{W}}\}_{i=1}^T ) \nonumber \\
 &=    - \frac{\eta}{T n_1 n_2} \sum_{i=1}^T \sum_{l=1}^{n_2}  \sum_{k=1}^{n_1}
f_{i}(\mathbf{x}_{i,l})  f_{i}(\mathbf{x}_{i,k})\bar{\sigma}(\mathbf{w}_j^\top \mathbf{x}_{i,k}) \sigma'(\mathbf{w}_j^\top \mathbf{x}_{i,l})\mathbf{x}_{i,l}   \nonumber \\
&\quad +\frac{\eta}{T n_2} \sum_{i=1}^T \sum_{l=1}^{n_2}    \chi\left\{ \eta f_{i}(\mathbf{x}_{i,l}) \left( \frac{1}{n_1} \sum_{k=1}^{n_1} f_{i}(\mathbf{x}_{i,k}) \bar{\sigma}(\mathbf{W}^0\mathbf{x}_{i,k} ) \right)^\top \sigma(\mathbf{W}^0\mathbf{x}_{i,l}) >1\right\}   \nonumber \\
&\quad \quad \quad \quad \quad \quad \quad \; \times  f_{i}(\mathbf{x}_{i,l}) \left(\frac{1}{n_1} \sum_{k=1}^{n_1}   f_{i}(\mathbf{x}_{i,k}) \bar{\sigma}(\mathbf{w}_j^\top\mathbf{x}_{i,k})\right) \sigma'(\mathbf{w}_j^\top\mathbf{x}_{i,l}) \mathbf{x}_{i,l}
 \nonumber \\
 &=   \underbrace{- \frac{\eta}{T n_1 n_2} \sum_{i=1}^T \sum_{l=1}^{n_2}  \sum_{k=1}^{n_1}
f_{i}(\mathbf{x}_{i,l})  f_{i}(\mathbf{x}_{i,k}){\sigma}'(\mathbf{w}_j^\top \mathbf{x}_{i,k}) \sigma'(\mathbf{w}_j^\top \mathbf{x}_{i,l})\mathbf{x}_{i,l}\mathbf{x}_{i,k}^\top \mathbf{w}_j}_{\circled{1}}    \nonumber \\
&\quad +\underbrace{\frac{\eta}{T n_2} \sum_{i=1}^T \sum_{l=1}^{n_2}    \chi\left\{ \eta f_{i}(\mathbf{x}_{i,l}) \left( \frac{1}{n_1} \sum_{k=1}^{n_1} f_{i}(\mathbf{x}_{i,k}) \bar{\sigma}(\mathbf{W}^0\mathbf{x}_{i,k} ) \right)^\top \sigma(\mathbf{W}^0\mathbf{x}_{i,l}) >1\right\}}_{\circled{2}}   \nonumber \\
&\quad \quad \quad \quad \quad \quad \quad \; \underbrace{\times  f_{i}(\mathbf{x}_{i,l}) \left(\frac{1}{n_1} \sum_{k=1}^{n_1}   f_{i}(\mathbf{x}_{i,k}) \bar{\sigma}(\mathbf{w}_j^\top\mathbf{x}_{i,k})\right) \sigma'(\mathbf{w}_j^\top\mathbf{x}_{i,l}) \mathbf{x}_{i,l}}_{\circled{2}}
 \label{full} 
\end{align}
where \eqref{full} follows since $\bar{\sigma}(x) = \sigma'(x)x$ when $\bar{\sigma}$ is the ReLU activation. By  \eqref{full} and the triangle inequality we have $\|\nabla_{\mathbf{w}_j} \hat{\mathcal{L}}( \mathbf{W}^0, \mathbf{b}^0, \{\mathbf{a}_{i}^1\}_{i=1}^T; \{\hat{\mathcal{D}}_{i,\mathbf{W}}\}_{i=1}^T ) - \circled{1}\|_2 \leq \|\circled{2}\|_2  $, so the result follows by bounding $\|\circled{2}\|_2$. To do this, note that by Hölder's inequality,
\begin{align}
 \| \circled{2}\|_2  &\leq \Bigg\| \frac{\eta}{T n_2} \sum_{i=1}^T \sum_{l=1}^{n_2}    \chi\left\{ \eta f_{i}(\mathbf{x}_{i,l}) \left( \frac{1}{n_1} \sum_{k=1}^{n_1} f_{i}(\mathbf{x}_{i,k}) \bar{\sigma}(\mathbf{W}^0\mathbf{x}_{i,k} ) \right)^\top \sigma(\mathbf{W}^0\mathbf{x}_{i,l}) >1\right\}   \nonumber \\
&\quad \quad \quad \quad \quad \quad \quad \; \times  f_{i}(\mathbf{x}_{i,l}) \left(\frac{1}{n_1} \sum_{k=1}^{n_1}   f_{i}(\mathbf{x}_{i,k}) \bar{\sigma}(\mathbf{w}_j^\top\mathbf{x}_{i,k})\right) \sigma'(\mathbf{w}_j^\top\mathbf{x}_{i,l}) \mathbf{x}_{i,l} \Bigg\|_2 \nonumber \\
&\leq  \eta {\frac{1}{T n_2} \sum_{i=1}^T \sum_{l=1}^{n_2}    \chi\left\{ \eta f_{i}(\mathbf{x}_{i,l}) \left( \frac{1}{n_1} \sum_{k=1}^{n_1} f_{i}(\mathbf{x}_{i,k}) \bar{\sigma}(\mathbf{W}^0\mathbf{x}_{i,k} ) \right)^\top \sigma(\mathbf{W}^0\mathbf{x}_{i,l}) >1\right\}}  \nonumber\\
&\quad \times {\max_{i\in [T], l\in [n_2]}  \left\|   \left(\frac{1}{n_1} \sum_{k=1}^{n_1}   f_{i}(\mathbf{x}_{i,k}) \bar{\sigma}(\mathbf{w}_j^\top\mathbf{x}_{i,k})\right) \sigma'(\mathbf{w}_j^\top\mathbf{x}_{i,l}) \mathbf{x}_{i,l}  \right\|_2 }\nonumber\\
&\leq  \eta \underbrace{\frac{1}{T n_2} \sum_{i=1}^T \sum_{l=1}^{n_2}    \chi\left\{ \eta f_{i}(\mathbf{x}_{i,l}) \left( \frac{1}{n_1} \sum_{k=1}^{n_1} f_{i}(\mathbf{x}_{i,k}) \bar{\sigma}(\mathbf{W}^0\mathbf{x}_{i,k} ) \right)^\top \sigma(\mathbf{W}^0\mathbf{x}_{i,l}) >1\right\}}_{\circled{2a}}  \nonumber\\
&\quad \times \underbrace{\max_{i\in [T]}  \left|\frac{1}{n_1} \sum_{k=1}^{n_1}   f_{i}(\mathbf{x}_{i,k}) \bar{\sigma}(\mathbf{w}_j^\top\mathbf{x}_{i,k})\right| \max_{i\in [T], l\in [n_2]}  \left\|   \mathbf{x}_{i,l}  \right\|_2 }_{\circled{2b}} \label{brea}
\end{align}
By Lemma \ref{lem:err_a}, we have 
\begin{align}
    \circled{2a} = \eta O\left(  \exp \left( -\frac{ c }{\eta^2  \nu_{\mathbf{w}}^2 {m\log(T/\delta_3)}({d} + {\log(m/\delta_3)}) ({d} + {m} )} \right)\right) + \eta \; O\left(  \frac{  \sqrt{\log(1/\delta_3)} }{ \sqrt{Tn_2} } \right) 
\end{align}
with probability at least $1 - \delta_3$. 
It remains to control $\circled{2b}$.
Fix $\mathbf{w}_j$, $i$, and $l$, then each random variable $ f_{i}(\mathbf{x}_{i,k}) \bar{\sigma}(\mathbf{w}_j^\top\mathbf{x}_{i,k})$ is sub-Gaussian with mean $O(\|\mathbf{w}_j\|_2)$
and variance $O(\|\mathbf{w}_j\|_2^2)$,
and each random variable in $\{ f_{i}(\mathbf{x}_{i,k}) \bar{\sigma}(\mathbf{w}_j^\top\mathbf{x}_{i,k}) \}_{k=1}^{n_1}$ is i.i.d. Thus, 
      $\frac{1}{n_1} \sum_{k=1}^{n_1}   f_{i}(\mathbf{x}_{i,k}) \bar{\sigma}(\mathbf{w}_j^\top\mathbf{x}_{i,k}) $
is sub-Gaussian with mean $O(\|\mathbf{w}_j\|_2)$
and variance $O\left(\frac{\|\mathbf{w}_j\|_2^2}{n_1}\right)$, so, by a union bound over all $i\in [T]$,
\begin{align}
   \max_{i\in [T]} \left|\frac{1}{n_1} \sum_{k=1}^{n_1}   f_{i}(\mathbf{x}_{i,k}) \bar{\sigma}(\mathbf{w}_j^\top\mathbf{x}_{i,k})\right| &= O\left(\|\mathbf{w}_j\|_2\left( 1 +  \sqrt{\frac{\log(T/\delta_2)}{n_1}} \right)\right) \nonumber \\
   &= O\left( \nu_{\mathbf{w}} ( \sqrt{d} + \sqrt{\log(m/\delta_1)}) \left( 1 +  \sqrt{\frac{\log(T/\delta_2)}{n_1}} \right)\right)
\end{align}
with probability at least $1-3\delta_1 - \delta_2$. 
Next, $ \left\| \mathbf{x}_{i,l}  \right\|_2^2 $ is sub-exponential with 
mean $O\left({d}\right)$ and variance 
$O\left(d^2\right)$. So, with probability at least $1-\delta_4$, 
\begin{align}
    &\max_{i\in[T],l\in [n_2]} \|\mathbf{x}_{i,l}  \|_2^2 = O\left(d(1 +  \log(Tn_2/\delta_4) ) \right)  = O\left(d \log(Tn_2/\delta_4)  \right) \nonumber \\
    &\implies \max_{i\in[T],l\in [n_2]} \|\mathbf{x}_{i,l}  \|_2  = O\left(\sqrt{d \log(Tn_2/\delta_4)}  \right) 
\end{align}

Combining these bounds via a union bound,  applying \eqref{brea},  and setting $\delta_1, \delta_2, \delta_3, \delta_4=\Theta( \delta)$ yields
\begin{align}
    \circled{2} &=  
     \;  O\Bigg( \eta \nu_{\mathbf{w}} (\sqrt{d} + \sqrt{\log(m/\delta)}) \sqrt{d\log(Tn_2/\delta)} \left(1 + {\tfrac{ \sqrt{\log(T/\delta)}}{\sqrt{n_1}}}\right) \nonumber \\
     &\quad \quad \quad \quad \times \left(   \exp \left( -\frac{ c }{\eta^2  \nu_{\mathbf{w}}^2 {m\log(T/\delta)}({d} + {\log(m/\delta)}) ({d} + {m} )} \right) + \frac{\sqrt{\log(1/\delta)}}{\sqrt{Tn_2}}\right) \Bigg) \nonumber  
\end{align} for an absolute constant $c$ 
with probability at least $1 -\delta$, completing the proof.
\end{proof}

\begin{lemma} \label{lem:grad_concen1}
    For any $\delta\in (0,1)$, with probability at least $1 - \delta$, for all $j\in [m]$, 
    \begin{align}
        &\Bigg\| 
        \frac{1}{T n_1 n_2} \sum_{i=1}^T \sum_{l=1}^{n_2}  \sum_{k=1}^{n_1}
f_{i}(\mathbf{x}_{i,l})  f_{i}(\mathbf{x}_{i,k})\sigma'((\mathbf{w}_j^0)^\top\mathbf{x}_{i,k}) \sigma'((\mathbf{w}_j^0)^\top \mathbf{x}_{i,l}) \mathbf{x}_{i,l}\mathbf{x}_{i,k}^\top\mathbf{w}_j^0
       \nonumber \\
       &-  \mathbb{E}_{i\sim \mathcal{T}}\mathbb{E}_{(\mathbf{x}, f_i(\mathbf{x}))\sim \mathcal{D}_{i}} \mathbb{E}_{(\mathbf{x}', f_i(\mathbf{x}'))\sim \mathcal{D}_{i}}\left[f_{i}(\mathbf{x})  f_{i}(\mathbf{x}')\sigma'((\mathbf{w}_j^0)^\top\mathbf{x}') \sigma'((\mathbf{w}_j^0)^\top \mathbf{x}') \mathbf{x}(\mathbf{x}')^\top\mathbf{w}_j^0\right] 
        \Bigg\|_2 \nonumber \\
        &= O\left(\nu_{\mathbf{w}} \sqrt{\frac{d + \log(m/\delta)}{T} \log(d/\delta)}  \left(  \sqrt{r\log(m/\delta)} + \sqrt{\frac{d}{n_2}} \right) \right). \nonumber
    \end{align}
\end{lemma}


\begin{proof} 
For ease of notation we replace $\mathbb{E}_{(\mathbf{x}, f_i(\mathbf{x}))\sim \mathcal{D}_{i}} \mathbb{E}_{(\mathbf{x}', f_i(\mathbf{x}'))\sim \mathcal{D}_{i}}$ with $\mathbb{E}_{\mathbf{x},\mathbf{x}'}$ and $\mathbb{E}_{i\sim \mathcal{T}}$ with $\mathbb{E}_i$, and write $\mathbf{w}_j=\mathbf{w}_j^0$. Consider any fixed $\mathbf{W}^0$ satisfying $E_{\mathbf{w}}(\delta_1)$ for $\delta_1\in (0,1)$, 
which occurs with probability at least $1-3 \delta_1$ by Lemma \ref{lem:all}. 
We have
\begin{align}
 &\Bigg\| 
        \frac{1}{T n_1 n_2} \sum_{i=1}^T \sum_{l=1}^{n_2}  \sum_{k=1}^{n_1}
f_{i}(\mathbf{x}_{i,l})  f_{i}(\mathbf{x}_{i,k})\sigma'(\mathbf{w}_j^\top\mathbf{x}_{i,k}) \sigma'(\mathbf{w}_j^\top \mathbf{x}_{i,l}) \mathbf{x}_{i,l}\mathbf{x}_{i,k}^\top\mathbf{w}_j
       \nonumber \\
       &-  \mathbb{E}_{i}\mathbb{E}_{\mathbf{x},\mathbf{x}'} \left[f_{i}(\mathbf{x})  f_{i}(\mathbf{x}')\sigma'(\mathbf{w}_j^\top\mathbf{x}') \sigma'(\mathbf{w}_j^\top \mathbf{x}') \mathbf{x}(\mathbf{x}')^\top\mathbf{w}_j\right] 
        \Bigg\|_2 \nonumber \\
        &= \left\| \frac{1}{T} \sum_{i=1}^T \mathbf{q}_i -  \mathbb{E}_{i}\mathbb{E}_{\mathbf{x},\mathbf{x}'} \left[f_{i}(\mathbf{x})  f_{i}(\mathbf{x}')\sigma'(\mathbf{w}_j^\top\mathbf{x}') \sigma'(\mathbf{w}_j^\top \mathbf{x}') \mathbf{x}(\mathbf{x}')^\top\mathbf{w}_j\right]  \right\|_2
        \label{twonorms}
\end{align}
where $\mathbf{q}_i \coloneqq \frac{1}{n_1n_2} \sum_{l=1}^{n_2}  \sum_{k=1}^{n_1}
f_{i}(\mathbf{x}_{i,l})  f_{i}(\mathbf{x}_{i,k})\sigma'(\mathbf{w}_j^\top\mathbf{x}_{i,k}) \sigma'(\mathbf{w}_j^\top \mathbf{x}_{i,l}) \mathbf{x}_{i,l}\mathbf{x}_{i,k}^\top\mathbf{w}_j$. By the linearity of the expectation, 
\begin{align}
    \mathbb{E}[\mathbf{q}_i]
    &= \frac{1}{n_1n_2} \sum_{l=1}^{n_2}  \sum_{k=1}^{n_1} \mathbb{E}_i \mathbb{E}_{\mathbf{x}_{i,l},\mathbf{x}_{i,k}}[ f_{i}(\mathbf{x}_{i,l})  f_{i}(\mathbf{x}_{i,k})\sigma'(\mathbf{w}_j^\top\mathbf{x}_{i,k}) \sigma'(\mathbf{w}_j^\top \mathbf{x}_{i,l}) \mathbf{x}_{i,l}\mathbf{x}_{i,k}^\top\mathbf{w}_j ]  \nonumber \\
    &= \mathbb{E}_{i}\mathbb{E}_{\mathbf{x},\mathbf{x}'} \left[f_{i}(\mathbf{x})  f_{i}(\mathbf{x}')\sigma'(\mathbf{w}_j^\top\mathbf{x}') \sigma'(\mathbf{w}_j^\top \mathbf{x}') \mathbf{x}(\mathbf{x}')^\top\mathbf{w}_j\right] 
\end{align}
which means that the random vectors $\mathbf{q}_i -  \mathbb{E}_{i}\mathbb{E}_{\mathbf{x},\mathbf{x}'} \left[f_{i}(\mathbf{x})  f_{i}(\mathbf{x}')\sigma'(\mathbf{w}_j^\top\mathbf{x}') \sigma'(\mathbf{w}_j^\top \mathbf{x}') \mathbf{x}(\mathbf{x}')^\top\mathbf{w}_j\right] = \mathbf{q}_i - \mathbb{E}[\mathbf{q}_i ]$ in \eqref{twonorms} are mean zero. 
Next we bound $ \|\frac{1}{T}\sum_{i=1}^T \mathbf{q}_i - \mathbb{E}[\mathbf{q}_i ]\|_2 $
by bounding each coordinate of $\frac{1}{T}\sum_{i=1}^T \mathbf{q}_i - \mathbb{E}[\mathbf{q}_i ]$ separately. Let $q_{i,h}$ denote the $h$-th entry of $\mathbf{q}_i$, and $x_{i,l,h}$ denote the $h$-th entry of $\mathbf{x}_{i,l}$. Note that each $q_{i,h}$ is the sum of products of two sub-Gaussian random variables ($f_{i}(\mathbf{x}_{i,l}) \sigma'(\mathbf{w}_j^\top \mathbf{x}_{i,l}) {x}_{i,l,h} $ and $ f_{i}(\mathbf{x}_{i,k})\sigma'(\mathbf{w}_j^\top\mathbf{x}_{i,k}) \mathbf{x}_{i,k}^\top\mathbf{w}_j $), so $q_{i,h}$ is the sum of sub-exponential random variables and is therefore sub-exponential. Its variance is upper bounded by:
\begin{align}
    &\mathbb{E}\left[({q}_{i,h} - \mathbb{E}[{q}_{i,h} ])^2 \right]  \nonumber \\
    &=  \mathbb{E}\left[\left(\left(\frac{1}{n_2} \sum_{l=1}^{n_2} 
f_{i}(\mathbf{x}_{i,l}) \sigma'(\mathbf{w}_j^\top \mathbf{x}_{i,l}) {x}_{i,l,h} \right) \left(\frac{1}{n_1}\sum_{k=1}^{n_1} f_{i}(\mathbf{x}_{i,k})\sigma'(\mathbf{w}_j^\top\mathbf{x}_{i,k}) \mathbf{x}_{i,k}^\top\mathbf{w}_j\right) - \mathbb{E}[{q}_{i,h} ]\right)^2 \right]\nonumber \\
&\leq 4 \underbrace{  \mathbb{E}_{i,\hat{\mathcal{D}}_i,\hat{\mathcal{D}}_i' }\Bigg[\Bigg(\left(\frac{1}{n_2} \sum_{l=1}^{n_2} 
f_{i}(\mathbf{x}_{i,l}) \sigma'(\mathbf{w}_j^\top \mathbf{x}_{i,l}) {x}_{i,l,h} \right) \left(\frac{1}{n_1}\sum_{k=1}^{n_1} f_{i}(\mathbf{x}_{i,k})\sigma'(\mathbf{w}_j^\top\mathbf{x}_{i,k}) \mathbf{x}_{i,k}^\top\mathbf{w}_j\right)}_{\circled{1}}\nonumber\\
&\quad \quad \quad \quad \underbrace{ - 
\mathbb{E}_{\mathbf{x}}\left[
f_{i}(\mathbf{x}) \sigma'(\mathbf{w}_j^\top \mathbf{x}) {x}_h \right] \left(\frac{1}{n_1}\sum_{k=1}^{n_1} f_{i}(\mathbf{x}_{i,k})\sigma'(\mathbf{w}_j^\top\mathbf{x}_{i,k}) \mathbf{x}_{i,k}^\top\mathbf{w}_j\right) \Bigg)^2 \Bigg]}_{\circled{1}}\nonumber \\
&\quad + 4\underbrace{ \mathbb{E}_{i,\hat{\mathcal{D}}_i' }\Bigg[\Bigg(\mathbb{E}_{\mathbf{x}}\left[
f_{i}(\mathbf{x}) \sigma'(\mathbf{w}_j^\top \mathbf{x}) {x}_h\right] \left(\frac{1}{n_1}\sum_{k=1}^{n_1} f_{i}(\mathbf{x}_{i,k})\sigma'(\mathbf{w}_j^\top\mathbf{x}_{i,k}) \mathbf{x}_{i,k}^\top\mathbf{w}_j\right)}_{\circled{2}} \nonumber \\
&\quad \quad \quad \quad \underbrace{ - \mathbb{E}_{\mathbf{x}}\left[
f_{i}(\mathbf{x}) \sigma'(\mathbf{w}_j^\top \mathbf{x}) {x}_h\right] \mathbb{E}_{\mathbf{x}'} \left[f_{i}(\mathbf{x}')\sigma'(\mathbf{w}_j^\top\mathbf{x}') (\mathbf{x}')^\top\mathbf{w}_j\right]\Bigg)^2 \Bigg]}_{\circled{2}} \nonumber \\
&\quad + 2 \underbrace{ \mathbb{E}_i\Big[\Big(\mathbb{E}_{\mathbf{x}}\left[
f_{i}(\mathbf{x}) \sigma'(\mathbf{w}_j^\top \mathbf{x}) {x}_h\right] \mathbb{E}_{\mathbf{x}'} \left[f_{i}(\mathbf{x}')\sigma'(\mathbf{w}_j^\top\mathbf{x}') (\mathbf{x}')^\top\mathbf{w}_j\right]}_{\circled{3}}\nonumber \\
&\quad \quad \quad \quad \underbrace{ - \mathbb{E}_{i}\left[\mathbb{E}_{\mathbf{x}}\left[
f_{i}(\mathbf{x}) \sigma'(\mathbf{w}_j^\top \mathbf{x}) {x}_h \right] \mathbb{E}_{\mathbf{x}'} \left[f_{i}(\mathbf{x}')\sigma'(\mathbf{w}_j^\top\mathbf{x}') (\mathbf{x}')^\top\mathbf{w}_j\right]\right] \Big)^2 \Big]}_{\circled{3}} \label{3re}
\end{align}
where \eqref{3re} follows from the triangle inequality and the fact that $(a+b)^2\leq 2a^2 + 2b^2$.
To bound $\circled{1}$, first let $\mathbf{s}_i \coloneqq \mathbb{E}_{\mathbf{x}}\left[
f_{i}(\mathbf{x}) \sigma'(\mathbf{w}_j^\top \mathbf{x}) {x}_h\right]$ and $s_{i,h}$ be its $h$-th element. Also denote $\overline{{s}_h^2} \coloneqq \mathbb{E}_i [s_{i,h}^2]$.
Observe that\begin{align}
    \circled{1} &=\mathbb{E}_{i,\hat{\mathcal{D}}_i,\hat{\mathcal{D}}_i' }\Bigg[\Bigg(\frac{1}{n_2} \sum_{l=1}^{n_2} 
f_{i}(\mathbf{x}_{i,l}) \sigma'(\mathbf{w}_j^\top \mathbf{x}_{i,l}) {x}_{i,l,h} - s_{i,h} \Bigg)^2\left(\frac{1}{n_1}\sum_{k=1}^{n_1} f_{i}(\mathbf{x}_{i,k})\sigma'(\mathbf{w}_j^\top\mathbf{x}_{i,k}) \mathbf{x}_{i,k}^\top\mathbf{w}_j\right)^2\Bigg] \nonumber \\
&\leq \mathbb{E}_{i,{\mathcal{D}}_i}\Bigg[\Bigg(\frac{1}{n_2} \sum_{l=1}^{n_2} 
f_{i}(\mathbf{x}_{i,l}) \sigma'(\mathbf{w}_j^\top \mathbf{x}_{i,l}) {x}_{i,l,h} - s_{i,h} \Bigg)^{4} \Bigg]^{1/2} \nonumber \\
&\quad \quad \quad \quad \times \mathbb{E}_{\hat{\mathcal{D}}_i}\Bigg[\left(\frac{1}{n_1}\sum_{k=1}^{n_1} f_{i}(\mathbf{x}_{i,k})\sigma'(\mathbf{w}_j^\top\mathbf{x}_{i,k}) \mathbf{x}_{i,k}^\top\mathbf{w}_j\right)^4\Bigg]^{1/2} \label{cs4} \\
&= \mathbb{E}_{i,{\mathcal{D}}_i}\Bigg[\Bigg(\frac{1}{n_2} \sum_{l=1}^{n_2} 
f_{i}(\mathbf{x}_{i,l}) \sigma'(\mathbf{w}_j^\top \mathbf{x}_{i,l}) {x}_{i,l,h} - s_{i,h} \Bigg)^{4} \Bigg]^{1/2}\; O \left( \|\mathbf{w}_j\|_2^2 \right)\label{jnjn}\\
&= O\left( \frac{ \|\mathbf{w}_j\|_2^2}{n_2} \right) \label{mnm}\\
&= O\left( \nu_{\mathbf{w}}^2 \frac{d + \log(m/\delta_1) }{n_2} \right)
\end{align}
where \eqref{cs4} follows by the Cauchy-Schwarz inequality,
\eqref{jnjn} follows since
$$\frac{1}{n_1}\sum_{k=1}^{n_1} f_{i}(\mathbf{x}_{i,k})\sigma'(\mathbf{w}_j^\top\mathbf{x}_{i,k}) \mathbf{x}_{i,k}^\top\mathbf{w}_j$$ is sub-Gaussian with mean $O(\|\mathbf{w}_j\|_2)$ and variance $O(\frac{\|\mathbf{w}_j\|_2^2}{n_1}
)$, and \eqref{mnm} follows since  $$\frac{1}{n_2} \sum_{l=1}^{n_2} 
f_{i}(\mathbf{x}_{i,l}) \sigma'(\mathbf{w}_j^\top \mathbf{x}_{i,l}) {x}_{i,l,h} - s_{i,h}$$ is sub-Gaussian with mean zero and variance $O(\frac{1}{n_2} \mathbb{E}_{\mathbf{x}}[ f_i(\mathbf{x})^2 \sigma'(\mathbf{w}_j^\top \mathbf{x})x_h^2 ] )= O(\frac{1}{n_2})$.
To bound $\circled{2}$, consider that 
\begin{align}
    \circled{2} &= \mathbb{E}_{i,\hat{\mathcal{D}}_i' }\nonumber \\
    &\quad \Bigg[\mathbb{E}_{\mathbf{x}}\left[
f_{i}(\mathbf{x}) \sigma'(\mathbf{w}_j^\top \mathbf{x}) {x}_h \right]^2 \left(\frac{1}{n_1}\sum_{k=1}^{n_1} f_{i}(\mathbf{x}_{i,k})\sigma'(\mathbf{w}_j^\top\mathbf{x}_{i,k}) \mathbf{x}_{i,k}^\top\mathbf{w}_j- \mathbb{E}_{\mathbf{x}'}\left[
f_{i}(\mathbf{x}') \sigma'(\mathbf{w}_j^\top \mathbf{x}') (\mathbf{x}')^\top\mathbf{w}_j\right]\right)^2\Bigg]\nonumber \\
 &= \mathbb{E}_{i}\left[s_{i,h}^2 \mathbb{E}_{\hat{\mathcal{D}}_i' } \left[ \left(\frac{1}{n_1}\sum_{k=1}^{n_1} f_{i}(\mathbf{x}_{i,k})\sigma'(\mathbf{w}_j^\top\mathbf{x}_{i,k}) \mathbf{x}_{i,k}^\top\mathbf{w}_j- \mathbb{E}_{\mathbf{x}'}\left[
f_{i}(\mathbf{x}') \sigma'(\mathbf{w}_j^\top \mathbf{x}') (\mathbf{x}')^\top\mathbf{w}_j\right]\right)^2\right]\right]\nonumber \\
&\leq \mathbb{E}_{i}\left[s_{i,h}^2 \max_{i'} \mathbb{E}_{\hat{\mathcal{D}}_{i'}' } \left[ \left(\frac{1}{n_1}\sum_{k=1}^{n_1} f_{i'}(\mathbf{x}_{i',k})\sigma'(\mathbf{w}_j^\top\mathbf{x}_{i',k}) \mathbf{x}_{i',k}^\top\mathbf{w}_j- \mathbb{E}_{\mathbf{x}'}\left[
f_{i'}(\mathbf{x}') \sigma'(\mathbf{w}_j^\top \mathbf{x}') (\mathbf{x}')^\top\mathbf{w}_j\right]\right)^2\right]\right]\label{frf}\\
&= \mathbb{E}_{i}\left[s_{i,h}^2 \right] \max_{i'} \mathbb{E}_{\hat{\mathcal{D}}_{i'}' } \left[ \left(\frac{1}{n_1}\sum_{k=1}^{n_1} f_{i'}(\mathbf{x}_{i',k})\sigma'(\mathbf{w}_j^\top\mathbf{x}_{i',k}) \mathbf{x}_{i',k}^\top\mathbf{w}_j- \mathbb{E}_{\mathbf{x}'}\left[
f_{i'}(\mathbf{x}') \sigma'(\mathbf{w}_j^\top \mathbf{x}') (\mathbf{x}')^\top\mathbf{w}_j\right]\right)^2\right]\nonumber \\
 &=  \mathbb{E}_{i}\left[s_{i,h}^2 \right] O\left( \nu_{\mathbf{w}}^2 \frac{d +\log(m/\delta_1)}{n_1} \right) \label{rdr} \\
&=O\left( \overline{s_h^2}\nu_{\mathbf{w}}^2 \frac{d +\log(m/\delta_1) }{n_1} \right)  \nonumber 
\end{align}
where \eqref{frf} follows since $s_{i,h}^2\geq 0$ and \\
$ \mathbb{E}_{\hat{\mathcal{D}}_{i'}' } \left[ \left(\frac{1}{n_1}\sum_{k=1}^{n_1} f_{i'}(\mathbf{x}_{i',k})\sigma'(\mathbf{w}_j^\top\mathbf{x}_{i',k}) \mathbf{x}_{i',k}^\top\mathbf{w}_j- \mathbb{E}_{\mathbf{x}'}\left[
f_{i'}(\mathbf{x}') \sigma'(\mathbf{w}_j^\top \mathbf{x}') (\mathbf{x}')^\top\mathbf{w}_j\right]\right)^2\right]\geq 0$, and \eqref{rdr} follows since for all $i'$, $\frac{1}{n_1}\sum_{k=1}^{n_1} f_{i'}(\mathbf{x}_{i',k})\sigma'(\mathbf{w}_j^\top\mathbf{x}_{i',k}) \mathbf{x}_{i',k}^\top\mathbf{w}_j- \mathbb{E}_{\mathbf{x}'}\left[
f_{i'}(\mathbf{x}') \sigma'(\mathbf{w}_j^\top \mathbf{x}') (\mathbf{x}')^\top\mathbf{w}_j\right]$ is sub-Gaussian with mean zero and variance $O(\frac{\|\mathbf{w}_j\|_2^2}{n_1}) = O(\frac{d\nu_{\mathbf{w}}^2}{n_1})$. 
To control $\circled{3}$, note that
\begin{align}
    \circled{3} &= \mathbb{E}_i\Big[\Big(\mathbb{E}_{\mathbf{x}}\left[
f_{i}(\mathbf{x}) \sigma'(\mathbf{w}_j^\top \mathbf{x}) {x}_h\right] \mathbb{E}_{\mathbf{x}'} \left[f_{i}(\mathbf{x}')\sigma'(\mathbf{w}_j^\top\mathbf{x}') (\mathbf{x}')^\top\mathbf{w}_j\right]\Big)^2\Big] \nonumber \\
&\quad \quad - \mathbb{E}_i\left[\mathbb{E}_{\mathbf{x}}\left[
f_{i}(\mathbf{x}) \sigma'(\mathbf{w}_j^\top \mathbf{x}) {x}_h\right] \mathbb{E}_{\mathbf{x}'} \left[f_{i}(\mathbf{x}')\sigma'(\mathbf{w}_j^\top\mathbf{x}') (\mathbf{x}')^\top\mathbf{w}_j\right]\right]^2 \nonumber \\
&= \mathbb{E}_i\Big[s_{i,h}^2 \mathbb{E}_{\mathbf{x}'} \left[f_{i}(\mathbf{x}')\sigma'(\mathbf{w}_j^\top\mathbf{x}') (\mathbf{x}')^\top\mathbf{w}_j\right]^2\Big]  \nonumber \\
&\leq  \mathbb{E}_i\Big[s_{i,h}^2 \max_{i'} \mathbb{E}_{\mathbf{x}'} \left[f_{i'}(\mathbf{x}')\sigma'(\mathbf{w}_j^\top\mathbf{x}') (\mathbf{x}')^\top\mathbf{w}_j\right]^2\Big]  \nonumber \\
&=  \mathbb{E}_i[s_{i,h}^2]  \max_{i'} \mathbb{E}_{\mathbf{x}'} \left[f_{i'}(\mathbf{x}')\sigma'(\mathbf{w}_j^\top\mathbf{x}') (\mathbf{x}')^\top\mathbf{w}_j\right]^2\Big]  \nonumber \\
&= \mathbb{E}_i\Big[s_{i,h}^2\Big] O( \nu_{\mathbf{w}}^2( d+ \log(m/\delta_1))  )  \nonumber \\
&= O( \overline{s_h^2}\nu_{\mathbf{w}}^2 (d + \log(m/\delta_1)  ) )
\end{align}
Combining the bounds on $\circled{1},\circled{2}$ and $\circled{3}$ yields $\mathbb{E}\left[({q}_{i,h} - \mathbb{E}[{q}_{i,h} ]  )^2 \right] = O(\nu_{\mathbf{w}}^2 (d+\log(m/\delta_1)) (\overline{s^2_h} + \frac{1}{n_2}) )$, therefore, since each random variable in $\{{q}_{i,h}\}_{i=1}^T$ is i.i.d.,
Bernstein's inequality gives
\begin{align}
   & \mathbb{P}_{\{{q}_{i,h}\}_i}\left(\left| \frac{1}{T}\sum_{i=1}^T {q}_{i,h} - \mathbb{E}[{q}_{i,h} ] \right| > t_h \right) \nonumber \\
   &\leq 2 \exp\left(- cT \min\left( \frac{t_h^2}{\nu_{\mathbf{w}}^2(d + \log(m/\delta_1)) ( \overline{s_h^2} + \frac{1}{n_2} )}, \frac{t_h}{\nu_{\mathbf{w}} \sqrt{d + \log(m/\delta_1)} \sqrt{ \overline{s_h^2} + \frac{1}{n_2} } }\right) \right) \label{berno}
\end{align}
with probability at least $1- 3 \delta_1$ over the selection of $\mathbf{W}^0$ for an absolute constant $c$ and  any $t_h\geq 0$.
Set $t_h=O\left(\nu_{\mathbf{w}} \sqrt{\frac{d + \log(m/\delta_1)}{T} (\overline{s_h^2} + \frac{1}{n_2}) \log(d/\delta) } \right)$ for all $h\in [d]$, then as long as $T \geq \log(d/\delta_2)$, via a union bound over all $h\in [d]$
we have
\begin{align}
    \left\| \frac{1}{T}\sum_{i=1}^T \mathbf{q}_i - \mathbb{E}[\mathbf{q}_i]  \right\|_2 &= \left( \sum_{h=1}^d \left( \frac{1}{T}\sum_{i=1}^T {q}_{i,h} - \mathbb{E}[{q}_{i,h}] \right)^2 \right)^{1/2} \nonumber \\
    &\leq  c \left( \sum_{h=1}^d \nu_{\mathbf{w}}^2 {\frac{d + \log(m/\delta_1)}{T} \left(\overline{s_h^2} + \frac{1}{n_2}\right) \log(d/\delta_2) }\right)^{1/2}  \label{bern} \\
    &= c \nu_{\mathbf{w}} \sqrt{\frac{d + \log(m/\delta_1)}{T} \log(d/\delta_2)} \left( \sum_{h=1}^d \overline{s_h^2} + \frac{d}{n_2} \right)^{1/2} \nonumber \\
    &= c \nu_{\mathbf{w}} \sqrt{\frac{d + \log(m/\delta_1)}{T} \log(d/\delta_2)} \left( \mathbb{E}_i\left[\sum_{h=1}^d s_{i,h}^2\right] + \frac{d}{n_2} \right)^{1/2} \nonumber \\
    &= c \nu_{\mathbf{w}} \sqrt{\frac{d + \log(m/\delta_1)}{T} \log(d/\delta_2)} \left( \mathbb{E}_i\left[\|\mathbf{s}_i\|_2^2\right] + \frac{d}{n_2} \right)^{1/2} \nonumber \\
    &\leq c' \nu_{\mathbf{w}} \sqrt{\frac{d + \log(m/\delta_1)}{T} \log(d/\delta_2)} \left(  r  +  \frac{r\|\mathbf{Mw}_j\|_2^2}{\|\mathbf{M}_\perp\mathbf{w}_j\|_2^2}  +  {\|\mathbf{Mw}_j\|_2^2 }+ \frac{d}{n_2} \right)^{1/2} \label{ou} \\
     &\leq c' \nu_{\mathbf{w}} \sqrt{\frac{d + \log(m/\delta_1)}{T} \log(d/\delta_2)} \left(  r (1+ \log(m/\delta_1) )  + \frac{d}{n_2} \right)^{1/2} \label{ouu} \\
    &= O\left(\nu_{\mathbf{w}} \sqrt{\frac{d + \log(m/\delta_1)}{T} \log(d/\delta_2)}  \left(  \sqrt{r\log(m/\delta_1)} + \sqrt{\frac{d}{n_2}} \right) \right) \nonumber 
\end{align}
with probability at least $1-3\delta_1-\delta_2$  for absolute constants $c,c'$, where \eqref{bern} follows with probability at least $1-3\delta_1-\delta_2$ due to \eqref{berno} and our choice of $t_h$, \eqref{ou} follows by Lemma \ref{lem:exp2} and the fact that $(a+b)^2\leq 2a^2+2b^2$, and \eqref{ouu} follows since $\mathbf{w}_j \in \mathcal{G}_{\mathbf{w}}$. Setting $\delta_1, \delta_2 = \Theta( \delta)$ completes the proof.
\end{proof}

\begin{lemma}\label{lem:compute_grad}
For any $\delta\in (0,1)$, with probability at least $1- \delta$, for all  $j\in [m]$,
    \begin{align}
\nabla_{\mathbf{w}_j} \hat{\mathcal{L}}( \mathbf{W}^0, \mathbf{b}^0, \{\mathbf{a}_{i}^1\}_{i=1}^T; \{\hat{\mathcal{D}}_{i,\mathbf{W}}\}_{i=1}^T )  &=  - \eta  2^{-r}  \mathbf{A}(\mathbf{w}_{j}^0) \mathbf{w}_{j}^0 + \eta \mathbf{e}  
\end{align} 
where 
\begin{align}\|\mathbf{e} \|_2 &=  \nu_{\mathbf{w}}\;
O\Bigg( (\sqrt{d} + \sqrt{\log(m/\delta)}) \sqrt{d\log(Tn_2/\delta)} \left(1 + {\frac{ \sqrt{\log(T/\delta)}}{\sqrt{n_1}}}\right) \nonumber \\
     &\quad \quad \quad \quad \times \left(   \exp \left( -\frac{ c }{\eta^2  \nu_{\mathbf{w}}^2 {m\log(T/\delta)}({d} + {\log(m/\delta)}) ({d} + {m} )} \right) + \frac{\sqrt{\log(1/\delta)}}{\sqrt{Tn_2}}\right) \Bigg) \nonumber \\
     &\quad + \nu_{\mathbf{w}}\; O\left( \sqrt{\frac{d + \log(m/\delta)}{T} \log(d/\delta)}  \left(  \sqrt{r\log(m/\delta)} + \sqrt{\frac{d}{n_2}} \right) \right)
\end{align} 
and  for a vector $\mathbf{w}\in \mathbb{R}^d$,
\begin{align}
    \mathbf{A}(\mathbf{w}) \coloneqq  \mathbb{E}_{\mathbf{u}} \Big[ \mathbb{E}_{\mathbf{x},\mathbf{x}'} \Big[{\sigma}'(\mathbf{w}^\top\mathbf{x}')\sigma'(\mathbf{w}^\top\mathbf{x})\mathbf{x}(\mathbf{x}')^\top | \sign(\mathbf{M}\mathbf{x})  =  \sign(\mathbf{M}\mathbf{x}') = \mathbf{u}\Big] \Big]
\end{align}
where $\mathbf{u}\sim \text{Unif}(\mathcal{H}^r)$ is a  random vector  drawn uniformly from the Rademacher hypercube in $r$ dimensions.
\end{lemma}

\begin{proof}
Let $\mathbf{w}_j = \mathbf{w}_{j}^0$. We have
\begin{align}
    &\nabla_{\mathbf{w}_j} \hat{\mathcal{L}}( \mathbf{W}^0, \mathbf{b}^0, \{\mathbf{a}_{i}^1\}_{i=1}^T; \{\hat{\mathcal{D}}_{i,\mathbf{W}}\}_{i=1}^T ) \nonumber \\
    &= -\frac{\eta}{T n_1 n_2} \sum_{i=1}^T \sum_{l=1}^{n_2}  \sum_{k=1}^{n_1}
f_{i}(\mathbf{x}_{i,l})  f_{i}(\mathbf{x}_{i,k})\sigma'(\mathbf{w}_j^\top\mathbf{x}_{i,k}) \sigma'(\mathbf{w}_j^\top \mathbf{x}_{i,l}) \mathbf{x}_{i,l}\mathbf{x}_{i,k}^\top\mathbf{w}_j  + \eta \mathbf{e}_1\nonumber \\
&= - \eta \mathbb{E}_{i\sim \mathcal{T}}\mathbb{E}_{(\mathbf{x}, f_i(\mathbf{x}))\sim \mathcal{D}_{i}} \mathbb{E}_{(\mathbf{x}', f_i(\mathbf{x}'))\sim \mathcal{D}_{i}}\left[f_{i}(\mathbf{x})  f_{i}(\mathbf{x}')\sigma'(\mathbf{w}_j^\top\mathbf{x}) \sigma'(\mathbf{w}_j^\top \mathbf{x}') \mathbf{x}(\mathbf{x}')^\top\mathbf{w}_j\right] 
        + \eta \mathbf{e}_1 + \eta \mathbf{e}_2
\end{align} 
where \begin{align}
\|\mathbf{e}_1\|_2 &=  O\Bigg( \nu_{\mathbf{w}} (\sqrt{d} + \sqrt{\log(m/\delta_1)}) \sqrt{d\log(Tn_2/\delta_1)} \left(1 + {\frac{ \sqrt{\log(T/\delta_1)}}{\sqrt{n_1}}}\right) \nonumber \\
     &\quad \quad \quad \quad \times \left(   \exp \left( -\frac{ c }{\eta^2  \nu_{\mathbf{w}}^2 {m(\log(T/\delta_1))}({d} + {\log(m/\delta_1)}) ({d} + {m} )} \right) + \frac{\sqrt{\log(1/\delta_1)}}{\sqrt{Tn_2}}\right) \Bigg)\nonumber \end{align}
with probability at least $1-\delta_1$ by Lemma \ref{lem:emp_grad} and \begin{align}
\|\mathbf{e}_2\|_2 = O\left(\nu_{\mathbf{w}} \sqrt{\frac{d + \log(m/\delta_2)}{T} \log(d/\delta_2)}  \left(  \sqrt{r\log(m/\delta_2)} + \sqrt{\frac{d}{n_2}} \right) \right)\end{align} with probability at least $1-\delta_2$ by Lemma \ref{lem:grad_concen1}. Set $\delta_1,\delta_2=\Theta(\delta)$ and apply the triangle inequality to complete the bound on $\| \mathbf{e}\|_2=\|\mathbf{e}_1+\mathbf{e}_2\|_2$ in the theorem statement. 

Next, for ease of notation we replace $\mathbb{E}_{(\mathbf{x}, f_i(\mathbf{x}))\sim \mathcal{D}_{i}} \mathbb{E}_{(\mathbf{x}', f_i(\mathbf{x}'))\sim \mathcal{D}_{i}}$ with $\mathbb{E}_{\mathbf{x},\mathbf{x}'}$ and $\mathbb{E}_{i\sim \mathcal{T}}$ with $\mathbb{E}_i$. Also, let $\mathbf{u} \sim \text{Unif}(\mathcal{H}^r)$ be uniformly drawn from the $r$-dimensional Rademacher hypercube $\mathcal{H}^r=\{-1,1\}^r$.
We have
\begin{align}
  &  \mathbb{E}_{i}\mathbb{E}_{\mathbf{x},\mathbf{x}'} \left[f_{i}(\mathbf{x})  f_{i}(\mathbf{x}')\sigma'(\mathbf{w}_j^\top\mathbf{x}) \sigma'(\mathbf{w}_j^\top \mathbf{x}') \mathbf{x}(\mathbf{x}')^\top\mathbf{w}_j\right]  \nonumber \\
  &= \mathbb{E}_{\mathbf{x},\mathbf{x}'} \left[  \mathbb{E}_{i}[f_{i}(\mathbf{x})  f_{i}(\mathbf{x}')]\sigma'(\mathbf{w}_j^\top\mathbf{x}) \sigma'(\mathbf{w}_j^\top \mathbf{x}') \mathbf{x}(\mathbf{x}')^\top\mathbf{w}_j\right]  \nonumber \\
  &= \mathbb{E}_{\mathbf{x},\mathbf{x}'} \Big[ \chi\{\sign(\mathbf{M}\mathbf{x})  =  \sign(\mathbf{M}\mathbf{x}') \} {\sigma}'(\mathbf{w}_{j}^\top\mathbf{x}') \sigma'(\mathbf{w}_{j}^\top\mathbf{x})\mathbf{x}(\mathbf{x}')^\top \mathbf{w}_j \Big] \nonumber \\
  &= 2^{-r}  \mathbb{E}_{\mathbf{u}} \Big[ \mathbb{E}_{\mathbf{x},\mathbf{x}'} \Big[{\sigma}'(\mathbf{w}_{j}^\top\mathbf{x}')\sigma'(\mathbf{w}_{j}^\top\mathbf{x})\mathbf{x}(\mathbf{x}')^\top | \sign(\mathbf{M}\mathbf{x})  =  \sign(\mathbf{M}\mathbf{x}') = \mathbf{u}\Big] \Big] \mathbf{w}_{j} \nonumber \\
  &= - 2^{-r}\mathbf{A}(\mathbf{w}_{j}) \mathbf{w}_{j} \label{shm}
\end{align}
where \eqref{shm} follows by the definition of $\mathbf{A}(\mathbf{w})$.
\end{proof}

\subsection{Analysis of the population gradient}

Next we define matrices capturing the energy of $\mathbf{A}(\mathbf{w}_j)$ in the ground-truth subspace and its perpendicular complement. 
   Again let  $\mathbf{u}\sim \text{Unif}(\mathcal{H}^r)$ be a random variable drawn uniformly from the Rademacher hypercube in $r$ dimensions, and $\mathbf{x}$ and $\mathbf{x}'$ be drawn independently from $\mathcal{N}(\mathbf{0}_d,\mathbf{I}_d)$, then for any $\mathbf{w}\in \mathbb{R}^d$ the matrices $\mathbf{A}_{||,||}(\w)$, $\mathbf{A}_{||,\perp}(\w)$, $\mathbf{A}_{\perp,||}(\w)$, and $\mathbf{A}_{\perp,\perp}(\w)$ are defined as
\begin{align}
    \mathbf{A}_{||,||}(\w) &= \mathbf{M}\mathbf{A}(\mathbf{w})\mathbf{M}^\top\nonumber \\
    &=
    \mathbb{E}_{\mathbf{u}} \left[   \mathbb{E}_{\mathbf{x},\mathbf{x}'} \left[
    {\sigma}'(\mathbf{w}^\top  \mathbf{x}) \sigma'(\mathbf{w}^\top \mathbf{x}') \mathbf{M}\mathbf{x}(\mathbf{x}')^\top\mathbf{M}^\top |\sign(\mathbf{M}\mathbf{x})  =  \sign(\mathbf{M}\mathbf{x}') = \mathbf{u}
    \right] \right]\\
    \mathbf{A}_{||,\perp}(\w) &= \mathbf{M}\mathbf{A}(\mathbf{w})\mathbf{M}_\perp^\top \nonumber \\
    &= \mathbb{E}_{\mathbf{u}} \left[   \mathbb{E}_{\mathbf{x},\mathbf{x}'} \left[\sigma'( \mathbf{w}^\top \mathbf{x})  \sigma'(\mathbf{w}^\top \mathbf{x})\mathbf{M}\mathbf{x} (\mathbf{x}')^\top \mathbf{M}_\perp^\top | \sign(\mathbf{M}\mathbf{x}) =  \sign(\mathbf{M}\mathbf{x}') = \mathbf{u} \right]\right] \\
    \mathbf{A}_{\perp,||}(\w) &= \mathbf{M}_\perp\mathbf{A}(\mathbf{w})\mathbf{M}^\top \nonumber \\
    &=  \mathbb{E}_{\mathbf{u}}\left[\mathbb{E}_{\mathbf{x},\mathbf{x}'}\left[\sigma'( \mathbf{w}^\top \mathbf{x} )\sigma'( \mathbf{w}^\top \mathbf{x}' ) \mathbf{M}_\perp \mathbf{x}(\mathbf{x}')^\top \mathbf{M}^\top  |\sign(\mathbf{M}\mathbf{x})  =  \sign(\mathbf{M}\mathbf{x}') = \mathbf{u}\right] \right]   \\
    \mathbf{A}_{\perp,\perp}(\w) &= \mathbf{M}_\perp\mathbf{A}(\mathbf{w})\mathbf{M}_\perp^\top \nonumber \\
    &=  \mathbb{E}_{\mathbf{u}}\left[ \mathbb{E}_{,\mathbf{x},\mathbf{x}'}\left[ \sigma'( \mathbf{w}^\top \mathbf{x} )\sigma'( \mathbf{w}^\top \mathbf{x}' ) \mathbf{M}_\perp \mathbf{x}(\mathbf{x}')^\top \mathbf{M}_\perp^\top |  \sign(\mathbf{M}\mathbf{x})  =  \sign(\mathbf{M}\mathbf{x}') = \mathbf{u}\right]\right]
\end{align}

Next we control the matrices  $\mathbf{A}_{||,||}(\w)$, $\mathbf{A}_{||,\perp}(\w)$, $\mathbf{A}_{\perp,||}(\w)$, and $\mathbf{A}_{\perp,\perp}(\w)$. 

\begin{lemma}\label{lem:apara}
    For any $\delta \in (0,1)$ such that $\delta = \Omega(me^{-d})$, then if $ \mathbf{w}\in \mathcal{G}_{\mathbf{w}}(\delta)$, we have 
    \begin{align}
       \left\|\mathbf{A}_{\parallel,\parallel}(\mathbf{w}) - \frac{1}{2\pi} \mathbf{I}_r\right\|_2 &=  O\left( \frac{ r^{3} + \log^3(m/\delta) }{d }\right).
    \end{align}
\end{lemma}

\begin{proof}
Recall that $\mathbf{u}$ is drawn uniformly from $\mathcal{H}^r$.
From Lemma \ref{lem:compute_grad}, we have
\begin{align}
    &\mathbf{A}_{\parallel,\parallel}(\mathbf{w}) \nonumber \\
    &= \mathbb{E}_{\mathbf{u}} \left[ \mathbb{E}_{\mathbf{x},\mathbf{x}'} \left[   {\sigma}'(\mathbf{w}^\top  \mathbf{x}') \sigma'(\mathbf{w}^\top \mathbf{x})  \mathbf{M}\mathbf{x}(\mathbf{Mx}')^\top |\sign(\mathbf{Mx}')=\sign({\mathbf{Mx}})=\mathbf{u}\right] \right] \nonumber \\
    &= \mathbb{E}_{\mathbf{u}} \left[
 \mathbb{E}_{\mathbf{M}\mathbf{x},{\mathbf{M}\mathbf{x}'}} \left[\mathbb{E}_{\mathbf{M}_\perp \mathbf{x},{\mathbf{M}_\perp \mathbf{x}'}} \left[{\sigma}'(\mathbf{w}^\top  \mathbf{x}') \sigma'(\mathbf{w}^\top \mathbf{x})\right] \mathbf{M}\mathbf{x}(\mathbf{M}\mathbf{x}')^\top   | \sign(\mathbf{Mx}')=\sign({\mathbf{Mx}}) =\mathbf{u}\right] \right]  \nonumber \\
 &=  \mathbb{E}_{\mathbf{u}} \Big[ \mathbb{E}_{\mathbf{M}\mathbf{x},{\mathbf{M}\mathbf{x}'}} \Big[  \mathbf{M}\mathbf{x}(\mathbf{x}')^\top\mathbf{M}^\top \mathbb{P}_{\mathbf{M}_\perp \mathbf{x}} \left[\mathbf{w}^\top \mathbf{M}_\perp^\top  \mathbf{M}_\perp \mathbf{x} > -\mathbf{w}^\top\mathbf{M}^\top \mathbf{Mx}  \right]   \nonumber \\
 &\quad \quad \quad \quad \quad  \times 
 \mathbb{P}_{{\mathbf{M}_\perp \mathbf{x}}'} \left[ \mathbf{w}^\top \mathbf{M}_\perp^\top  \mathbf{M}_\perp \mathbf{x}' > -\mathbf{w}^\top\mathbf{M}^\top \mathbf{Mx}'  \right] |\sign(\mathbf{Mx}')=\sign({\mathbf{Mx}}) =\mathbf{u}  \Big]\Big]  \nonumber \\
 &= \mathbb{E}_{\mathbf{u},\mathbf{M}\mathbf{x},{\mathbf{M}\mathbf{x}'}} \Bigg[  \mathbf{M}\mathbf{x}(\mathbf{x}')^\top \mathbf{M}^\top \left(\frac{1}{2} + \frac{1}{2}\erf\left(\frac{\mathbf{w}^\top \mathbf{M}^\top  {\mathbf{Mx}} }{ \sqrt{2}\| {\mathbf{M}_\perp \mathbf{w}} \|_2 } \right) \right) \nonumber \\
&\quad \quad \quad \quad \quad    \left(\frac{1}{2} + \frac{1}{2}\erf\left(\frac{\mathbf{w}^\top  \mathbf{Mx}' }{\sqrt{2} \| {\mathbf{M}_{\perp}\mathbf{w}} \|_2 } \right) \right) |\sign(\mathbf{Mx}')=\sign({\mathbf{Mx}}) =\mathbf{u}  \Bigg] \label{gu} \\
 &= \frac{1}{4}\mathbb{E}_{\mathbf{u},\mathbf{M}\mathbf{x},{\mathbf{M}\mathbf{x}'}} \left[   \mathbf{M}\mathbf{x}(\mathbf{x}')^\top \mathbf{M}^\top  |\sign(\mathbf{Mx}')=\sign({\mathbf{Mx}}) =\mathbf{u} \right] \nonumber \\
 &\quad + \frac{1}{2}\mathbb{E}_{\mathbf{u},\mathbf{M}\mathbf{x},\mathbf{M}\mathbf{x}'} \left[  \mathbf{M}\mathbf{x}(\mathbf{x}')^\top \mathbf{M}^\top 
\erf\left(\frac{\mathbf{w}^\top \mathbf{M}^\top{\mathbf{Mx}} }{ \sqrt{2}\| {\mathbf{M}_{\perp}\mathbf{w}} \|_2 } \right) |\sign(\mathbf{Mx}')=\sign({\mathbf{Mx}}) =\mathbf{u} \right]  \nonumber \\
 &\quad + \frac{1}{4}\mathbb{E}_{\mathbf{u},\mathbf{M}\mathbf{x},{\mathbf{M}\mathbf{x}'}} \nonumber \\
 &\quad \quad \quad \quad \left[   \mathbf{M}\mathbf{x}(\mathbf{x}')^\top \mathbf{M}^\top 
 \erf\left(\frac{\mathbf{w}^\top \mathbf{M}^\top{\mathbf{Mx}} }{ \sqrt{2}\| {\mathbf{M}_{\perp}\mathbf{w}} \|_2 } \right) \erf\left(\frac{\mathbf{w}^\top \mathbf{M}^\top\mathbf{Mx}' }{ \sqrt{2}\| {\mathbf{M}_{\perp}\mathbf{w}} \|_2 } \right) |\sign(\mathbf{Mx}')=\sign({\mathbf{Mx}}) =\mathbf{u}  \right]  \label{apara3}
\end{align}
where \eqref{gu} follows using the Gaussian CDF.
For the first term in \eqref{apara3}, we can re-write $\mathbf{M}\mathbf{x}$ conditioned on $\sign({\mathbf{Mx}}) =\mathbf{u}$ as $\diag(\mathbf{u}) |\mathbf{M}\mathbf{x}|$, where $|\cdot|$ denotes element-wise absolute value, to obtain
\begin{align}
     &\frac{1}{4}\mathbb{E}_{\mathbf{u},\mathbf{M}\mathbf{x},{\mathbf{M}\mathbf{x}'}} \left[  \mathbf{M}\mathbf{x}(\mathbf{x}')^\top \mathbf{M}^\top | \sign(\mathbf{Mx}')=\sign({\mathbf{Mx}}) =\mathbf{u} \right]   \nonumber \\
     &=  \frac{1}{4}\mathbb{E}_{\mathbf{u},\mathbf{M}\mathbf{x},{\mathbf{M}\mathbf{x}'}} \left[  \diag(\mathbf{u})|\mathbf{M}\mathbf{x}||\mathbf{Mx}'|^\top \diag(\mathbf{u}) \right]  \nonumber \\
     &=   \frac{1}{4}\mathbb{E}_{\mathbf{u}} \left[  \diag(\mathbf{u})\mathbb{E}_{\mathbf{M}\mathbf{x}} \left[  |\mathbf{M}\mathbf{x}|\right] \mathbb{E}_{{\mathbf{M}\mathbf{x}'}} \left[|\mathbf{Mx}'|\right]^\top \diag(\mathbf{u}) \right]    \nonumber \\
      &=  \frac{1}{2\pi}\mathbb{E}_{\mathbf{u}} \left[  \diag(\mathbf{u})\mathbf{1}_r\mathbf{1}_{r}^\top \diag(\mathbf{u}) \right]       \label{halkn} \\
      &=  \frac{1}{2\pi}\mathbb{E}_{\mathbf{u}} \left[  \mathbf{u}\mathbf{u}^\top \right]       \nonumber \\
     &= \frac{1}{2\pi} \mathbf{I}_r      \label{dfc} 
\end{align}
where \eqref{halkn} follows since each element of $|\mathbf{M}\mathbf{x}|$ and $|\mathbf{M}\mathbf{x}'|$ is a standard half-normal random variable.
For the second term in \eqref{apara3}, we have
\begin{align}
     &\frac{1}{2}\mathbb{E}_{\mathbf{u},\mathbf{M}\mathbf{x},\mathbf{M}\mathbf{x}'} \left[  \mathbf{M}\mathbf{x}(\mathbf{x}')^\top \mathbf{M}^\top 
\erf\left(\frac{\mathbf{w}^\top \mathbf{M}^\top{\mathbf{Mx}} }{ \sqrt{2}\| {\mathbf{M}_{\perp}\mathbf{w}} \|_2 } \right) | \sign(\mathbf{Mx}')=\sign({\mathbf{Mx}})=\mathbf{u}   \right] \nonumber \\
&= \frac{1}{2}\mathbb{E}_{\mathbf{u},\mathbf{M}\mathbf{x},\mathbf{M}\mathbf{x}'} \left[  \diag(\mathbf{u})|\mathbf{M}\mathbf{x}||\mathbf{Mx}'|^\top \diag(\mathbf{u}) 
\erf\left(\frac{\mathbf{w}^\top \mathbf{M}^\top \diag(\mathbf{u})|{\mathbf{Mx}}| }{ \sqrt{2}\| {\mathbf{M}_{\perp}\mathbf{w}} \|_2 } \right)   \right]\label{jn} \\
&= \frac{1}{2}\mathbb{E}_{\mathbf{u},\mathbf{M}\mathbf{x},\mathbf{M}\mathbf{x}'} \left[  \diag(-\mathbf{u})|\mathbf{M}\mathbf{x}||\mathbf{Mx}'|^\top \diag(-\mathbf{u}) 
\erf\left(\frac{\mathbf{w}^\top \mathbf{M}^\top \diag(-\mathbf{u})|{\mathbf{Mx}}| }{ \sqrt{2}\| {\mathbf{M}_{\perp}\mathbf{w}} \|_2 } \right)   \right] \label{jj}\\
&= -\frac{1}{2}\mathbb{E}_{\mathbf{u},\mathbf{M}\mathbf{x},\mathbf{M}\mathbf{x}'} \left[  \diag(\mathbf{u})|\mathbf{M}\mathbf{x}||\mathbf{Mx}'|^\top \diag(\mathbf{u}) 
\erf\left(\frac{\mathbf{w}^\top \mathbf{M}^\top \diag(\mathbf{u})|{\mathbf{Mx}}| }{ \sqrt{2}\| {\mathbf{M}_{\perp}\mathbf{w}} \|_2 } \right)   \right] \label{nnj} \\
&= \mathbf{0} \label{lst}
\end{align}
where \eqref{jj} follows from the fact that $\mathbf{u}$ and $-\mathbf{u}$ have the same distribution, \eqref{nnj} follows since $\erf()$ is an odd function, and \eqref{lst} follows since $x=-x\iff x=0$.

The final term in \eqref{apara3} is 
\begin{align}
&\mathbb{E}_{\mathbf{u},\mathbf{M}\mathbf{x},\mathbf{M}\mathbf{x}'} \left[  \mathbf{M}\mathbf{x}(\mathbf{x}')^\top \mathbf{M}^\top 
 \erf\left(\frac{\mathbf{w}^\top \mathbf{M}^\top{\mathbf{Mx}} }{ \sqrt{2}\| {\mathbf{M}_{\perp}\mathbf{w}} \|_2 } \right) \erf\left(\frac{\mathbf{w}^\top \mathbf{M}^\top\mathbf{Mx}' }{ \sqrt{2}\| {\mathbf{M}_{\perp}\mathbf{w}} \|_2 } \right)| \sign(\mathbf{Mx}')=\sign({\mathbf{Mx}})  =\mathbf{u} \right]  \nonumber \\
 &= \mathbb{E}_{\mathbf{u}}
 \Bigg[\mathbb{E}_{\mathbf{M}\mathbf{x}} \left[  \mathbf{M}\mathbf{x} 
 \erf\left(\frac{\mathbf{w}^\top \mathbf{M}^\top{\mathbf{Mx}} }{ \sqrt{2}\| {\mathbf{M}_{\perp}\mathbf{w}} \|_2 } \right) \bigg| \sign({\mathbf{Mx}})=\mathbf{u}\right] \nonumber \\
 &\quad \quad \times \mathbb{E}_{\mathbf{Mx}'}\left[(\mathbf{x}')^\top \mathbf{M}^\top\erf\left(\frac{\mathbf{w}^\top \mathbf{M}^\top\mathbf{Mx}' }{ \sqrt{2}\| {\mathbf{M}_{\perp}\mathbf{w}} \|_2 } \right)\bigg| \sign(\mathbf{Mx}')=\mathbf{u}\right] \Bigg] \label{46}
\end{align}
Again to remove the conditioning, we can equivalently write 
 \begin{align}
     \mathbb{E}_{\mathbf{M}\mathbf{x}} \left[  \mathbf{M}\mathbf{x} 
 \erf\left(\frac{\mathbf{w}^\top \mathbf{M}^\top{\mathbf{Mx}} }{ \sqrt{2}\| {\mathbf{M}_{\perp}\mathbf{w}} \|_2 } \right) \bigg| \sign({\mathbf{Mx}})=\mathbf{u}\right] 
 &= \mathbb{E}_{\mathbf{M}\mathbf{x}} \left[ \text{diag}(\mathbf{u}) |\mathbf{M}\mathbf{x} |
 \erf\left(\frac{\mathbf{w}^\top \mathbf{M}^\top  \text{diag}(\mathbf{u}) | {\mathbf{Mx}}| }{ \sqrt{2}\| {\mathbf{M}_{\perp}\mathbf{w}} \|_2 } \right) \right] 
 \end{align}
For all $\mathbf{u}\in \mathcal{H}^r$, we have 
\begin{align}
   & \left\|\mathbb{E}_{\mathbf{M}\mathbf{x}} \left[ \text{diag}(\mathbf{u}) |\mathbf{M}\mathbf{x} |
 \erf\left(\frac{\mathbf{w}^\top \mathbf{M}^\top  \text{diag}(\mathbf{u}) | {\mathbf{Mx}}| }{ \sqrt{2}\| {\mathbf{M}_{\perp}\mathbf{w}} \|_2 } \right) \right] \right\|_2 \nonumber \\
    &\leq  \mathbb{E}_{\mathbf{M}\mathbf{x}} \left[\left\| \text{diag}(\mathbf{u}) |\mathbf{M}\mathbf{x} |
 \erf\left(\frac{\mathbf{w}^\top \mathbf{M}^\top  \text{diag}(\mathbf{u}) | {\mathbf{Mx}}| }{ \sqrt{2}\| {\mathbf{M}_{\perp}\mathbf{w}} \|_2 } \right) \right\|_2 \right]  \nonumber \\
 &\leq \mathbb{E}_{\mathbf{M}\mathbf{x}} \left[\left\| \text{diag}(\mathbf{u}) |\mathbf{M}\mathbf{x} |\right\|_2
\left|\erf\left(\frac{\mathbf{w}^\top \mathbf{M}^\top  \text{diag}(\mathbf{u}) | {\mathbf{Mx}}| }{ \sqrt{2}\| {\mathbf{M}_{\perp}\mathbf{w}} \|_2 } \right) \right|\right]  \nonumber \\
 &\leq \mathbb{E}_{\mathbf{M}\mathbf{x}} \left[\left\| \text{diag}(\mathbf{u}) |\mathbf{M}\mathbf{x} |\right\|_2
\left|\frac{\mathbf{w}^\top \mathbf{M}^\top  \text{diag}(\mathbf{u}) | {\mathbf{Mx}}| }{ \| {\mathbf{M}_{\perp}\mathbf{w}} \|_2 } \right|\right]  \label{jhj} \\
 &\leq \left\|\text{diag}(\mathbf{u})\right\|_2  \mathbb{E}_{\mathbf{M}\mathbf{x}} \left[ \|\mathbf{M}\mathbf{x}\|_2^2\right]
\frac{\|\mathbf{w}^\top \mathbf{M}^\top  \text{diag}(\mathbf{u}) \|_2 }{ \| {\mathbf{M}_{\perp}\mathbf{w}} \|_2 }  \nonumber \\
 &\leq \frac{c r (\sqrt{r} +\sqrt{\log(m/\delta)}) }{\sqrt{d} } \label{wb}
\end{align}
for an absolute constant $c$, where \eqref{jhj} follows since $|\erf(x)|\leq \sqrt{2}|x| $, and \eqref{wb} follows
since $\mathbf{w}\in \mathcal{G}_{\mathbf{w}}(\delta)$. 
and $m/\delta = O(e^d)$, thus $\| {\mathbf{M}_{\perp}\mathbf{w}} \|_2 = \Omega(\sqrt{d})$.
Therefore, using \eqref{46},
\begin{align}
    &\left\|\mathbb{E}_{\mathbf{u},\mathbf{M}\mathbf{x},\mathbf{M}\mathbf{x}'} \left[  \mathbf{M}\mathbf{x}(\mathbf{x}')^\top \mathbf{M}^\top 
 \erf\left(\frac{\mathbf{w}^\top \mathbf{M}^\top{\mathbf{Mx}} }{ \sqrt{2}\| {\mathbf{M}_{\perp}\mathbf{w}} \|_2 } \right) \erf\left(\frac{\mathbf{w}^\top \mathbf{M}^\top\mathbf{Mx}' }{ \sqrt{2}\| {\mathbf{M}_{\perp}\mathbf{w}} \|_2 } \right)| \sign(\mathbf{Mx}')=\sign({\mathbf{Mx}})  =\mathbf{u} \right]  \right\|_2 \nonumber \\
 &\leq \frac{c' r^{3} + c'\log^3(m/\delta) }{d }
\end{align}
completing the proof.
\end{proof}

\begin{lemma} \label{lem:aperp}
    For any $\delta \in (0,1)$ such that $\delta = \Omega(me^{-d})$, then if $ \mathbf{w}\in \mathcal{G}_{\mathbf{w}}(\delta)$, we have 
    \begin{align}
        \left\|\mathbf{A}_{\perp,\perp}(\mathbf{w}) - \left(1 - \frac{\|\mathbf{Mw}\|_2^2}{\|\mathbf{M}_\perp\mathbf{w}\|_2^2} \right) \frac{\mathbf{M}_\perp \mathbf{w} \mathbf{w}^\top \mathbf{M}_\perp^\top }{2\pi \|\mathbf{M}_\perp\mathbf{w}\|_2^2} \right\|_2 &\leq O\left(\frac{r^4 + \log^4(m/\delta)}{d^2}\right).
    \end{align}
\end{lemma}

\begin{proof}
We have
\begin{align}
  &  \mathbf{A}_{\perp,\perp}(\mathbf{w}) \nonumber \\
  &=  \mathbb{E}_{\mathbf{u}} \left[\mathbb{E}_{\mathbf{x},\mathbf{x}'} \left[  \mathbf{M}_\perp \mathbf{x}(\mathbf{M}_{\perp}\mathbf{x}')^\top {\sigma}'(\mathbf{w}^\top  \mathbf{x}') \sigma'(\mathbf{w}^\top \mathbf{x}) | \sign(\mathbf{Mx}')=\sign({\mathbf{Mx}}) =\mathbf{u}  \right]\right] \nonumber \\
    &= \mathbb{E}_{\mathbf{u},\mathbf{M}\mathbf{x},\mathbf{Mx}'} \left[ \mathbb{E}_{\mathbf{M}_\perp \mathbf{x},\mathbf{M}_{\perp}\mathbf{x}'} \left[ \mathbf{M}_\perp \mathbf{x}(\mathbf{M}_{\perp}\mathbf{x}')^\top {\sigma}'(\mathbf{w}^\top  \mathbf{x}') \sigma'(\mathbf{w}^\top \mathbf{x}) \right] | \sign(\mathbf{Mx}')=\sign({\mathbf{Mx}})=\mathbf{u}  \right] \nonumber \\
    &= \mathbb{E}_{\mathbf{u},\mathbf{M}\mathbf{x},\mathbf{Mx}'} \left[ \mathbb{E}_{\mathbf{M}_\perp \mathbf{x}} \left[ \mathbf{M}_\perp \mathbf{x} \sigma'(\mathbf{w}^\top \mathbf{x})  \right] \mathbb{E}_{\mathbf{M}_{\perp}\mathbf{x}'} \left[(\mathbf{M}_{\perp}\mathbf{x}')^\top {\sigma}'(\mathbf{w}^\top  \mathbf{x}') \right] | \sign(\mathbf{Mx}')=\sign({\mathbf{Mx}}) =\mathbf{u} \right] \label{cb}  
\end{align}
Using Lemma \ref{lem:exp-perp} to compute $\mathbb{E}_{\mathbf{M}_\perp \mathbf{x}} \left[ \mathbf{M}_\perp \mathbf{x} \sigma'(\mathbf{w}^\top \mathbf{x})  \right] $ and $\mathbb{E}_{\mathbf{M}_{\perp}\mathbf{x}'} \left[(\mathbf{M}_{\perp}\mathbf{x}')^\top {\sigma}'(\mathbf{w}^\top  \mathbf{x}') \right]$  yields 
\begin{align}
   &\mathbf{A}_{\perp,\perp}(\mathbf{w})\nonumber \\
   &=  \mathbb{E}_{\mathbf{u},\mathbf{M}\mathbf{x},\mathbf{Mx}'} \left[ \exp\left( -\frac{(\mathbf{x}^\top\mathbf{M}^\top \mathbf{Mw})^2+ ((\mathbf{x}')^\top \mathbf{M}^\top\mathbf{Mw})^2}{2\|\mathbf{M}_\perp\mathbf{w}\|_2^2} \right)    \bigg| \sign(\mathbf{Mx}')=\sign({\mathbf{Mx}}) =\mathbf{u} \right] \nonumber \\
    &\quad \times \frac{1}{2\pi \|\mathbf{M}_\perp\mathbf{w}\|_2^2}\mathbf{M}_\perp \mathbf{w} \mathbf{w}^\top \mathbf{M}_\perp^\top   \label{where}
\end{align}
We analyze the scalar term in the top line. We have
\begin{align}
&\mathbb{E}_{\mathbf{u}} \left[\mathbb{E}_{\mathbf{M}\mathbf{x},\mathbf{Mx}'} \left[ \exp\left( -\frac{(\mathbf{x}^\top\mathbf{M}^\top \mathbf{Mw})^2+ ((\mathbf{x}')^\top \mathbf{M}^\top\mathbf{Mw})^2}{2\|\mathbf{M}_\perp\mathbf{w}\|_2^2} \right)    \bigg| \sign(\mathbf{Mx}')=\sign({\mathbf{Mx}}) =\mathbf{u} \right]\right] \nonumber \\
    &= 
    \mathbb{E}_{\mathbf{u}}\Bigg[ \mathbb{E}_{\mathbf{M}\mathbf{x}} \left[ \exp\left( -\frac{(\mathbf{x}^\top \mathbf{M}^\top \mathbf{Mw})^2}{2\|\mathbf{M}_\perp\mathbf{w}\|_2^2} \right)    \bigg| \sign({\mathbf{Mx}}) = \mathbf{u}  \right] \nonumber \\
    &\quad \quad \quad \quad \times \mathbb{E}_{\mathbf{Mx}'} \left[ \exp\left( -\frac{((\mathbf{x}')^\top \mathbf{M}^\top\mathbf{Mw})^2}{2\|\mathbf{M}_\perp\mathbf{w}\|_2^2} \right)    \bigg| \sign(\mathbf{Mx}')= \mathbf{u} \right]\Bigg] \nonumber \\
    &= 
    \mathbb{E}_{\mathbf{u}} \left[ \mathbb{E}_{\mathbf{M}\mathbf{x}} \left[ \exp\left( -\frac{(\mathbf{x}^\top\mathbf{M}^\top \mathbf{Mw})^2}{2\|\mathbf{M}_\perp\mathbf{w}\|_2^2} \right)    \bigg| \sign({\mathbf{Mx}}) = \mathbf{u}  \right]^2 \right] \label{f}
\end{align}
where, using the Taylor expansion of $\exp({-x^2})$ and  re-writing $\mathbf{M}\mathbf{x}$ conditioned on $\{\sign({\mathbf{Mx}}) =\mathbf{u}\}$ as $\diag(\mathbf{u}) |\mathbf{M}\mathbf{x}|$,
\begin{align}
    &\mathbb{E}_{\mathbf{M}\mathbf{x}} \left[ \exp\left( -\frac{(\mathbf{x}^\top\mathbf{M}^\top\mathbf{Mw})^2}{2\|\mathbf{M}_\perp\mathbf{w}\|_2^2} \right)    \bigg| \sign({\mathbf{Mx}}) = \mathbf{u}  \right] \nonumber \\
    &= \mathbb{E}_{\mathbf{M}\mathbf{x}} \left[ \exp\left( -\frac{(|\mathbf{M}\mathbf{x}|^\top \text{diag}(\mathbf{u}) \mathbf{Mw})^2}{2\|\mathbf{M}_\perp\mathbf{w}\|_2^2} \right)    \right] \nonumber \\
    &= 1 - \frac{1}{2\|\mathbf{M}_\perp\mathbf{w}\|_2^2} \mathbb{E}_{\mathbf{M}\mathbf{x}} \left[(|\mathbf{M}\mathbf{x}|^\top \diag(\mathbf{u}) \mathbf{Mw})^2\right]  + O\left( \frac{r^4 + \log^4(m/\delta)}{ d^2} \right) \label{fb}
\end{align}
where \eqref{fb} follows since  $\|\mathbf{M}_\perp\mathbf{w}\|_2^2 = \Omega( (\sqrt{d} - \sqrt{\log(m/\delta)})^4) = \Omega(d^2)$ and $$\mathbb{E}_{\mathbf{M}\mathbf{x}}[ (|\mathbf{M}\mathbf{x}|^\top \text{diag}(\mathbf{u}) \mathbf{Mw})^4 ]\leq \mathbb{E}_{\mathbf{M}\mathbf{x}}[ \|\mathbf{M}\mathbf{x}\|_2^4 ]\|\mathbf{Mw}\|_2^4 = O(r^4 + \log^4(m/\delta))$$ since $\mathbf{w}\in \mathcal{G}_{\mathbf{w}}$. To compute the second term in \eqref{fb}, note that
\begin{align}
    \mathbb{E}_{\mathbf{M}\mathbf{x}} \left[(|\mathbf{M}\mathbf{x}|^\top \text{diag}(\mathbf{u}) \mathbf{Mw})^2\right]  &=   \mathbf{w}^\top \mathbf{M}^\top \text{diag}(\mathbf{u})\mathbb{E}_{\mathbf{M}\mathbf{x}}  \left[|\mathbf{M}\mathbf{x}||\mathbf{M}\mathbf{x}|^\top \right]\text{diag}(\mathbf{u}) \mathbf{Mw} \nonumber \\
    &= \mathbf{w}^\top \mathbf{M}^\top \text{diag}(\mathbf{u})\left(\frac{2}{\pi}\mathbf{1}_{r}\mathbf{1}_{r}^\top + \left(1 - \frac{2}{\pi}\right)\mathbf{I}_{r}\right)\text{diag}(\mathbf{u}) \mathbf{Mw} \nonumber \\               
    &=  \mathbf{w}^\top \mathbf{M}^\top \left(\frac{2}{\pi}\mathbf{uu}^\top + \left(1 - \frac{2}{\pi}\right)\mathbf{I}_{r}\right)\mathbf{Mw} \nonumber 
\end{align}
therefore
\begin{align}
   &\mathbb{E}_{\mathbf{u}} \left[ \mathbb{E}_{\mathbf{M}\mathbf{x}} \left[ \exp\left( -\frac{(\mathbf{x}^\top\mathbf{M}^\top\mathbf{Mw})^2}{2\|\mathbf{M}_\perp\mathbf{w}\|_2^2} \right)    \bigg| \sign({\mathbf{Mx}}) = \mathbf{u}  \right]^2 \right] \nonumber \\
   &= \mathbb{E}_{\mathbf{u}} \left[\left( 1 -  \frac{1}{2\|\mathbf{M}_\perp\mathbf{w}\|_2^2}\mathbf{w}^\top \mathbf{M}^\top \left(\frac{2}{\pi}\mathbf{uu}^\top + \left(1 - \frac{2}{\pi}\right)\mathbf{I}_{r}\right)\mathbf{Mw} + O\left( \frac{r^4 + \log^4(m/\delta)}{ d^{2}} \right) \right)^2 \right] \nonumber \\
   &= 1 -  \frac{1}{\|\mathbf{M}_\perp\mathbf{w}\|_2^2}\mathbf{w}^\top \mathbf{M}^\top \left(\frac{2}{\pi}\mathbb{E}_{\mathbf{u}}[\mathbf{uu}^\top] + \left(1 - \frac{2}{\pi}\right)\mathbf{I}_{r}\right)\mathbf{Mw} \nonumber \\
   &\quad + \frac{1}{4\|\mathbf{M}_\perp \mathbf{w}\|_2^4}\mathbf{w}^\top \mathbf{M}^\top \mathbb{E}_{\mathbf{u}}\left[ \left(\frac{2}{\pi}\mathbf{uu}^\top + \left(1 - \frac{2}{\pi}\right)\mathbf{I}_{r}\right)\mathbf{Mw}\mathbf{w}^\top \mathbf{M}^\top \left(\frac{2}{\pi}\mathbf{uu}^\top + \left(1 - \frac{2}{\pi}\right)\mathbf{I}_{r}\right)\right] \mathbf{Mw} \nonumber \\
   &\quad +  O\left(\frac{r^4 + \log^4(m/\delta)}{ d^{2}} \right) \nonumber \\
   &=  1 -  \frac{\|\mathbf{Mw}\|_2^2}{\|\mathbf{M}_\perp\mathbf{w}\|_2^2} + \frac{\left((1- \frac{2}{\pi})^2 + \frac{4}{\pi}\left(1 - \frac{2}{\pi}\right) \right)\|\mathbf{M} \mathbf{w}\|_2^4     }{4\|\mathbf{M}_\perp \mathbf{w}\|_2^4} \nonumber \\
   &\quad  + \frac{4}{\pi^2}
   \mathbf{w}^\top \mathbf{M}^\top \mathbb{E}_{\mathbf{u}}\left[ \mathbf{uu}^\top \mathbf{Mw} \mathbf{w}^\top \mathbf{M}^\top  \mathbf{uu}^\top \right] \mathbf{Mw}  +  O\left(\frac{r^4 + \log^4(m/\delta)}{ d^{2}}  \right) \nonumber \\
   &=  1 -  \frac{\|\mathbf{Mw}\|_2^2}{\|\mathbf{M}_\perp\mathbf{w}\|_2^2} + \frac{\left((1- \frac{2}{\pi})^2 + \frac{4}{\pi}\left(1 - \frac{2}{\pi}\right) \right)\|\mathbf{M} \mathbf{w}\|_2^4     }{4\|\mathbf{M}_\perp \mathbf{w}\|_2^4}  \nonumber \\
   &\quad + \frac{1}{\pi^2 \|\mathbf{M}_\perp \mathbf{w}\|_2^4 } 
   \mathbf{w}^\top \mathbf{M}^\top \mathbb{E}_{\mathbf{u}}\left[ \mathbf{uu}^\top (\mathbf{u}^\top\mathbf{Mw})^2 \right] \mathbf{Mw}  +  O\left(\frac{r^4 + \log^4(m/\delta)}{ d^{2} } \right) \nonumber \\
   &=  1 -  \frac{\|\mathbf{Mw}\|_2^2}{\|\mathbf{M}_\perp\mathbf{w}\|_2^2} + \frac{\left((1- \frac{2}{\pi})^2 + \frac{4}{\pi}\left(1 - \frac{2}{\pi}\right) \right)\|\mathbf{M} \mathbf{w}\|_2^4     }{4\|\mathbf{M}_\perp \mathbf{w}\|_2^4}  \nonumber \\
   &\quad + \frac{1}{\pi^2 \|\mathbf{M}_\perp \mathbf{w}\|_2^4 } 
   \mathbf{w}^\top \mathbf{M}^\top \left( \|\mathbf{Mw}\|_2^2\mathbf{I}_r + 2 \mathbf{Mw}\mathbf{w}^\top \mathbf{M}^\top - \diag( \mathbf{Mw} )^2 \right) \mathbf{Mw}  +  O\left(\frac{r^4 + \log^4(m/\delta)}{ d^{2}} \right) \nonumber \\
   &= 1 -  \frac{\|\mathbf{Mw}\|_2^2}{\|\mathbf{M}_\perp\mathbf{w}\|_2^2} + \frac{\left((1- \frac{2}{\pi})^2 + \frac{4}{\pi}\left(1 - \frac{2}{\pi}\right) \right)\|\mathbf{M} \mathbf{w}\|_2^4     }{4\|\mathbf{M}_\perp \mathbf{w}\|_2^4} + \frac{3\|\mathbf{Mw}\|_2^4 - \|\mathbf{Mw}\|_4^4 }{\pi^2 \|\mathbf{M}_\perp \mathbf{w}\|_2^4 }  +  O\left(\frac{r^4 + \log^4(m/\delta)}{ d^{2}} \right) \nonumber \\
   &= 1 - \frac{\|\mathbf{Mw}\|_2^2}{\|\mathbf{M}_\perp\mathbf{w}\|_2^2} + O\left(\frac{r^4 + \log^4(m/\delta)}{ d^{2}} \right), \nonumber 
\end{align}
where we have used $\w\in \mathcal{G}_{\mathbf{w}}(\delta)$ and $m/\delta = O(e^d)$. Combining this with \eqref{where} and \eqref{f} completes the proof.
\end{proof}

\begin{lemma}\label{lem:aparaperp}
   For any $\delta \in (0,1)$ such that $\delta = \Omega(me^{-d})$, then if $ \mathbf{w}\in \mathcal{G}_{\mathbf{w}}(\delta)$, we have 
    \begin{align}
        \left\|\mathbf{A}_{\parallel,\perp}(\mathbf{w}) - \frac{ \mathbf{Mw}\mathbf{w}^\top \mathbf{M}_\perp^\top }{2{\pi}\|\mathbf{M}_\perp\mathbf{w}\|_2^2} \right\|_2 &\leq O\left(\frac{r^{3.5} + \log^{3.5}(m/\delta)}{d^{1.5}}\right)
    \end{align}
\end{lemma}

\begin{proof}
Arguing similarly to the previous two lemmas, we obtain
\begin{align}
    \mathbf{A}_{\parallel,\perp}(\mathbf{w}) &= \mathbb{E}_{\mathbf{u},\mathbf{x},\mathbf{x}'} \left[  \mathbf{M}\mathbf{x}(\mathbf{M}_{\perp}\mathbf{x}')^\top {\sigma}'(\mathbf{w}^\top  \mathbf{x}') \sigma'(\mathbf{w}^\top \mathbf{x}) | \sign(\mathbf{Mx}')=\sign({\mathbf{Mx}}) =\mathbf{u}  \right] \nonumber \\
    &= \mathbb{E}_{\mathbf{u},\mathbf{x},\mathbf{Mx}'} \left[  \mathbf{M}\mathbf{x} \sigma'(\mathbf{w}^\top \mathbf{x}) \E_{\mathbf{M}_{\perp}\mathbf{x}'} \left[(\mathbf{M}_{\perp}\mathbf{x}')^\top {\sigma}'(\mathbf{w}^\top  \mathbf{x}')\right] | \sign(\mathbf{Mx}')=\sign({\mathbf{Mx}}) = \mathbf{u}  \right] \nonumber \\
    &=  \mathbb{E}_{\mathbf{u},\mathbf{x},\mathbf{Mx}'} \left[  \mathbf{M}\mathbf{x} \sigma'(\mathbf{w}^\top \mathbf{x}) \exp\left( - \frac{((\mathbf{x}')^\top \mathbf{M}^\top\mathbf{Mw})^2}{2\|\mathbf{M}_\perp\mathbf{w}\|_2^2} \right)| \sign(\mathbf{Mx}')=\sign({\mathbf{Mx}})=\mathbf{u}   \right] \nonumber \\
    &\quad \times  \frac{ \mathbf{w}^\top \mathbf{M}_{\perp}^\top }{\sqrt{2 \pi}\|\mathbf{M}_\perp \mathbf{w}\|_2}  \label{nm} 
\end{align}
where \eqref{nm} follows by Lemma \ref{lem:exp-perp}.
Next, since the only term that depends on $\mathbf{M}_\perp \mathbf{x}$ is $\sigma'(\mathbf{w}^\top \mathbf{x})$, we have
\begin{align}
&\mathbf{A}_{\perp,\perp}(\mathbf{w})\nonumber \\
 &=  \mathbb{E}_{\mathbf{u},\mathbf{M}\mathbf{x},\mathbf{Mx}'} \left[  \mathbf{M}\mathbf{x} \mathbb{E}_{\mathbf{M}_\perp \mathbf{x}}[\sigma'(\mathbf{w}^\top \mathbf{x})] \exp\left( - \frac{((\mathbf{x}')^\top \mathbf{M}^\top\mathbf{Mw})^2}{2\|\mathbf{M}_\perp\mathbf{w}\|_2^2} \right)| \sign(\mathbf{Mx}')=\sign({\mathbf{Mx}})=\mathbf{u}   \right] \nonumber \\
    &\quad \times  \frac{ \mathbf{w}^\top \mathbf{M}_{\perp}^\top }{\sqrt{2 \pi}\|\mathbf{M}_\perp \mathbf{w}\|_2}  \nonumber \\
     &=  \mathbb{E}_{\mathbf{u},\mathbf{M}\mathbf{x},\mathbf{Mx}'} \bigg[  \mathbf{M}\mathbf{x} \mathbb{P}_{\mathbf{M}_\perp \mathbf{x}}[\mathbf{w}^\top \mathbf{M}_\perp^\top \mathbf{M}_\perp \mathbf{x} > - \mathbf{w}^\top \mathbf{M}^\top \mathbf{M}\mathbf{x} ] \exp\left( - \frac{((\mathbf{x}')^\top \mathbf{M}^\top\mathbf{Mw})^2}{2\|\mathbf{M}_\perp\mathbf{w}\|_2^2} \right)\nonumber \\
    & \quad\quad \quad \quad \quad \quad \quad| \sign(\mathbf{Mx}')=\sign({\mathbf{Mx}})=\mathbf{u}   \bigg] \nonumber \\
    &\quad \times  \frac{ \mathbf{w}^\top \mathbf{M}_\perp^\top }{\sqrt{2 \pi}\|\mathbf{M}_\perp \mathbf{w}\|_2}  \nonumber \\
    &=  \mathbb{E}_{\mathbf{u},\mathbf{Mx},\mathbf{Mx}'} \bigg[  \mathbf{M}\mathbf{x} \left(\frac{1}{2} + \frac{1}{2} \erf\left( \frac{\mathbf{x}^\top\mathbf{M}^\top\mathbf{Mw}}{\sqrt{2}\|\mathbf{M}_\perp \mathbf{w}\|_2} \right) \right) \exp\left( \! - \frac{((\mathbf{x}')^\top \mathbf{M}^\top\mathbf{Mw})^2}{2\|\mathbf{M}_\perp\mathbf{w}\|_2^2} \right)\nonumber \\
    & \quad\quad \quad \quad \quad \quad \quad| \sign(\mathbf{Mx}')=\sign({\mathbf{Mx}})=\mathbf{u}   \bigg] \nonumber \\
    &\quad \times \frac{ \mathbf{w}^\top \mathbf{M}_\perp^\top }{\sqrt{2 \pi}\|\mathbf{M}_\perp \mathbf{w}\|_2}   \label{hhj} 
    \end{align}
where \eqref{hhj} is due to the Gaussian CDF.
Note that
\begin{align}
    &\mathbb{E}_{\mathbf{u},\mathbf{Mx},\mathbf{Mx}'} \left[  \mathbf{M}\mathbf{x}  \exp\left( - \frac{((\mathbf{x}')^\top \mathbf{M}^\top\mathbf{Mw})^2}{2\|\mathbf{M}_\perp\mathbf{w}\|_2^2} \right)| \sign(\mathbf{Mx}')=\sign({\mathbf{Mx}})  = \mathbf{u}  \right] \nonumber \\
    &= \mathbb{E}_{\mathbf{u},\mathbf{M}\mathbf{x},\mathbf{Mx}'} \left[ \diag(\mathbf{u}) |\mathbf{M}\mathbf{x} | \exp\left( - \frac{(|\mathbf{Mx}'|^\top\diag(\mathbf{u})\mathbf{Mw})^2}{2\|\mathbf{M}_\perp\mathbf{w}\|_2^2} \right) \right] \nonumber \\
    &= \mathbb{E}_{\mathbf{M}\mathbf{x},\mathbf{Mx}'} \left[ \mathbb{E}_{\mathbf{u}}\left[ \diag(\mathbf{u}) |\mathbf{M}\mathbf{x} | \left(1  - \frac{(|\mathbf{Mx}'|^\top\diag(\mathbf{u})\mathbf{Mw})^2}{2\|\mathbf{M}_\perp\mathbf{w}\|_2^2} + \frac{(|\mathbf{Mx}'|^\top\diag(\mathbf{u})\mathbf{Mw})^4}{2\|\mathbf{M}_\perp \mathbf{w}\|_2^4} -\dots \right)\right] \right] \label{aza} \\
    &= \mathbf{0} \label{ewe}
\end{align}
where \eqref{ewe} follows since each term in \eqref{aza} is an odd power of $\mathbf{u}$.
  Thus, from \eqref{hhj} we have
  \begin{align}
\mathbf{A}_{\parallel,\perp}(\mathbf{w})     &=  \mathbb{E}_{\mathbf{u},\mathbf{Mx},\mathbf{Mx}'} \bigg[  \mathbf{M}\mathbf{x} \erf\left( \frac{\mathbf{x}^\top\mathbf{M}^\top\mathbf{Mw}}{\sqrt{2}\|\mathbf{M}_\perp \mathbf{w}\|_2} \right)  \exp\left( - \frac{((\mathbf{x}')^\top \mathbf{M}^\top\mathbf{Mw})^2}{2\|\mathbf{M}_\perp\mathbf{w}\|_2^2} \right)\nonumber \\
    & \quad\quad \quad \quad \quad \quad \quad| \sign(\mathbf{Mx}')=\sign({\mathbf{Mx}})=\mathbf{u}   \bigg] \nonumber \\
    &\quad \times \frac{ \mathbf{w}^\top \mathbf{M}_\perp^\top }{2\sqrt{2\pi}\|\mathbf{M}_\perp \mathbf{w}\|_2} \label{sds}
    \end{align}
Next we take the Taylor expansions of $\erf(x)$ and $\exp(-x^2)$ to obtain 
    \begin{align}
\mathbf{A}_{\parallel,\perp}(\mathbf{w}) 
&= \mathbb{E}_{\mathbf{u},\mathbf{Mx},\mathbf{Mx}'} \Bigg[  \mathbf{M}\mathbf{x} \left( \frac{\mathbf{x}^\top\mathbf{M}^\top\mathbf{Mw}}{\|\mathbf{M}_\perp \mathbf{w}\|_2} 
 - \frac{ (\mathbf{x}^\top\mathbf{M}^\top\mathbf{Mw})^3 }{6\|\mathbf{M}_\perp \mathbf{w}\|_2^3} + \dots \right)\nonumber \\
 &\quad \quad \quad \quad \times \left(1  - \frac{((\mathbf{x}')^\top \mathbf{M}^\top\mathbf{Mw})^2}{2\|\mathbf{M}_\perp\mathbf{w}\|_2^2} + \dots \right)| \sign(\mathbf{Mx}')=\sign({\mathbf{Mx}})   = \mathbf{u} \Bigg]\frac{ \mathbf{w}^\top \mathbf{M}_\perp^\top }{2{\pi}\|\mathbf{M}_\perp \mathbf{w}\|_2} \nonumber\\
    &= \mathbb{E}_{\mathbf{u},\mathbf{Mx},\mathbf{Mx}'} \left[  \mathbf{M}\mathbf{x}\mathbf{x}^\top\mathbf{M}^\top | 
 \sign(\mathbf{Mx}')=\sign({\mathbf{Mx}})  = \mathbf{u} \right]\frac{ \mathbf{Mw}\mathbf{w}^\top \mathbf{M}_\perp^\top }{2{\pi}\|\mathbf{M}_\perp\mathbf{w}\|_2^2}  + \mathbf{E} \nonumber \\
 &= \mathbb{E}_{\mathbf{Mx}} \left[  \mathbf{M}\mathbf{x}\mathbf{x}^\top\mathbf{M}^\top \right]\frac{ \mathbf{Mw}\mathbf{w}^\top \mathbf{M}_\perp^\top }{2{\pi}\|\mathbf{M}_\perp\mathbf{w}\|_2^2}  + \mathbf{E} \nonumber \\
 &=   \frac{ \mathbf{Mw}\mathbf{w}^\top \mathbf{M}_\perp^\top }{2{\pi}\|\mathbf{M}_\perp\mathbf{w}\|_2^2}  + \mathbf{E} \nonumber 
\end{align}
where $\|\mathbf{E}\|_2 = O\left( \frac{r^{3.5}+\log^{3.5}(m/\delta)}{d^{1.5}} \right)$ since $\|\mathbf{M}_\perp \mathbf{w}\|_2 = \Omega(\sqrt{d})$ and $\|\mathbf{Mw}\|_2 = O(\sqrt{r} + \sqrt{\log(m/\delta)})$ as $\mathbf{w} \in \mathcal{G}_{\mathbf{w}}(\delta)$ and $m/\delta = O(e^d)$.
\end{proof}

\subsection{Full results}

\begin{lemma} \label{lem:final}
    Set $\eta = \Theta(1)$ and $\lambda_{\mathbf{w}}= \frac{1}{\eta} + \frac{ \eta}{ 2^{r+1}\pi }$. Consider any $\delta\in (0,1)$ such that $\delta = \Omega(me^{-d})$.
    Then  there is an absolute constant $c$ such that for all $j\in [m]$,
    \begin{align}
      \Bigg \|&\mathbf{Mw}_{j}^1 - \frac{\eta^2  }{2^{r+2} }\mathbf{Mw}_{j}^0  \Bigg\|_2 \nonumber \\
       &\leq  \eta^2  \nu_{\mathbf{w}} O \left( \frac{r^4 + \log^4(m)}{2^r d} 
       \right) + \eta^2\nu_{\mathbf{w}}\;O\left( \sqrt{\frac{d }{T} \log(d/\delta)}  \left(  \sqrt{r\log(m/\delta)} + \sqrt{\frac{d}{n_2}} \right) \right)
       \nonumber \\
       &\; +  \eta^2  \nu_{\mathbf{w}} \nonumber \\
       &\quad \times O\Bigg(  d\sqrt{\log(Tn_2/\delta)} \left(1 + {\frac{ \sqrt{\log(T/\delta)}}{\sqrt{n_1}}}\right)  \left(   \exp \left( -\frac{ c }{\eta^2  \nu_{\mathbf{w}}^2 {\log(T/\delta)}dm ({d} + {m} )} \right) + \frac{\sqrt{\log(1/\delta)}}{\sqrt{Tn_2}}\right) \Bigg) \nonumber \\
     \end{align}
     and
    \begin{align}
        \bigg\| &\mathbf{M}_{\perp}  \mathbf{w}_{j}^1 \bigg\|_2 \nonumber \\
        &\leq \eta^2  \nu_{\mathbf{w}} \; O\left( \frac{r^{3.5} + \log^{3.5}(m)}{2^{r}d^{1.5}} 
        \right) + \eta^2\nu_{\mathbf{w}}\; O\left( \sqrt{\frac{d }{T} \log(d/\delta)}  \left(  \sqrt{r\log(m/\delta)} + \sqrt{\frac{d}{n_2}} \right) \right)
       \nonumber \\
       &\; +  \eta^2  \nu_{\mathbf{w}} \nonumber \\
       &\quad \times O\Bigg(  d\sqrt{\log(Tn_2/\delta)} \left(1 + {\frac{ \sqrt{\log(T/\delta)}}{\sqrt{n_1}}}\right)  \left(   \exp \left( -\frac{ c }{\eta^2  \nu_{\mathbf{w}}^2 {\log(T/\delta)}dm ({d} + {m} )} \right) + \frac{\sqrt{\log(1/\delta)}}{\sqrt{Tn_2}}\right) \Bigg) \nonumber
    \end{align}
    with probability at least $1-\delta$.
\end{lemma}

\begin{proof} First note that $\mathbf{W}^0 \in E_{\mathbf{w}}(\delta_1)$ with probability at least $1-\delta_1$ by Lemma \ref{lem:all}. We consider any fixed $\mathbf{W}^0$ satisfying $E_{\mathbf{w}}(\delta_1)$ for the rest of the proof.
 Due to the computation of the gradient in Lemma \ref{lem:compute_grad}, we have
\begin{align}
    \mathbf{w}_{j}^1 &= (1 - \eta \lambda_{\mathbf{w}})\mathbf{w}_{j}^0 - \nabla_{\mathbf{w}_j} \hat{\mathcal{L}}( \mathbf{W}^0, \mathbf{b}^0, \{\mathbf{a}_{i}^1\}_{i=1}^T; \{\hat{\mathcal{D}}_{i,\mathbf{W}}\}_{i=1}^T ) \nonumber \\
    &= (1 - \eta \lambda_{\mathbf{w}})\mathbf{w}_{j}^0  + \eta^2  2^{-r}  \mathbf{A}(\mathbf{w}_{j}^0) \mathbf{w}_{j}^0 + \eta \mathbf{e}   \nonumber \\
    \mathbf{Mw}_{j}^1
    &= (1 - \eta \lambda_{\mathbf{w}}) \mathbf{Mw}_{j}^0  + {\eta^2 2^{-r}}\mathbf{A}_{\parallel,\parallel}(\mathbf{w}_{j}^0)\mathbf{Mw}_{j}^0 + {\eta^2 2^{-r}}\mathbf{A}_{\parallel,\perp}(\mathbf{w}_{j}^0)\mathbf{M}_{\perp} \mathbf{w}_{j}^0 + \eta \mathbf{Me} \nonumber \\
    \mathbf{M}_{\perp}\mathbf{w}_{{j}}^1 &= (1 - \eta \lambda_{\mathbf{w}}) \mathbf{M}_{\perp} \mathbf{w}_{j}^0  + {\eta^2 2^{-r}} \mathbf{A}_{\perp,\parallel}(\mathbf{w}_{j}^0)\mathbf{Mw}_{j}^0 + {\eta^2 2^{-r}}  \mathbf{A}_{\perp,\perp}(\mathbf{w}_{j}^0)\mathbf{M}_{\perp} \mathbf{w}_{j}^0 + \eta \mathbf{M}_{\perp}\mathbf{e} \nonumber 
\end{align}
where, with probability at least $1-\delta_2$ for $\delta_2 = \Omega(me^{-d})$, 
\begin{align}\|\mathbf{e} \|_2 &=  \eta \nu_{\mathbf{w}}\;
O\Bigg(  d\sqrt{\log(Tn_2/\delta)} \left(1 + {\frac{ \sqrt{\log(T/\delta)}}{\sqrt{n_1}}}\right) \nonumber \\
     &\quad \quad \quad \quad \times \left(   \exp \left( -\frac{ c }{\eta^2  \nu_{\mathbf{w}}^2 {\log(T/\delta)}dm ({d} + {m} )} \right) + \frac{\sqrt{\log(1/\delta)}}{\sqrt{Tn_2}}\right) \Bigg) \nonumber \\
     &\quad + \eta \nu_{\mathbf{w}}\; O\left( \sqrt{\frac{d }{T} \log(d/\delta)}  \left(  \sqrt{r\log(m/\delta)} + \sqrt{\frac{d}{n_2}} \right) \right).
\end{align} 
Next, 
\begin{align}
    \mathbf{Mw}_{j}^1 &= (1 - \eta \lambda_{\mathbf{w}})\mathbf{Mw}_{j}^0  + \frac{\eta^2 2^{-r}}{2 \pi } \mathbf{Mw}_{j}^0 + \frac{ {\eta^2 2^{-r}} }{2{\pi} \|\mathbf{M}_{\perp} \mathbf{w}_{j}^0\|_2^2} \mathbf{Mw}_{j}^0(\mathbf{M}_{\perp} \mathbf{w}_{j}^0)^\top  \mathbf{M}_{\perp} \mathbf{w}_{j}^0 \nonumber \\
    &\quad + {\eta^2 2^{-r}}\left(   \mathbf{A}_{\parallel,\parallel}(\mathbf{w}_{j}^0) - \frac{1}{2\pi}\mathbf{I}_r \right)\mathbf{Mw}_{j}^0 \nonumber \\
    &\quad + {\eta^2 2^{-r}} \left( \mathbf{A}_{\parallel,\perp}(\mathbf{w}_{j}^0) -  \frac{1 }{2{\pi} \|\mathbf{M}_{\perp} \mathbf{w}_{j}^0\|_2^2} \mathbf{Mw}_{j}^0(\mathbf{M}_{\perp} \mathbf{w}_{j}^0)^\top   \right)\mathbf{M}_{\perp} \mathbf{w}_{j}^0 + \eta \mathbf{Me} \nonumber \\
    &= \left(1 - \eta \lambda_{\mathbf{w}} + \frac{\eta^2  }{ 2^r \pi } \right) \mathbf{Mw}_{j}^0  + {\eta^2 2^{-r}}\mathbf{e}_{\parallel} + \eta \mathbf{Me}
\end{align}
where 
\begin{align}
    \|\mathbf{e}_{\parallel}\|_2&=  \left\| \left(   \mathbf{A}_{\parallel,\parallel}(\mathbf{w}_{j}^0) - \tfrac{1}{2\pi}\mathbf{I}_r \right)\mathbf{Mw}_{j}^0  +  \left( \mathbf{A}_{\parallel,\perp}(\mathbf{w}_{j}^0) -  \frac{1 }{2{\pi} \|\mathbf{M}_{\perp} \mathbf{w}_{j}^0\|_2^2} \mathbf{Mw}_{j}^0(\mathbf{M}_{\perp} \mathbf{w}_{j}^0)^\top   \right)\mathbf{M}_{\perp} \mathbf{w}_{j}^0   \right\|_2 \nonumber \\
    &\leq \left\| \left(   \mathbf{A}_{\parallel,\parallel}(\mathbf{w}_{j}^0) - \tfrac{1}{2\pi}\mathbf{I}_r \right)\mathbf{Mw}_{j}^0 \right\|_2 \nonumber \\
    &\quad +   \left\| \left( \mathbf{A}_{\parallel,\perp}(\mathbf{w}_{j}^0) -  \frac{1 }{2{\pi} \|\mathbf{M}_{\perp} \mathbf{w}_{j}^0\|_2^2} \mathbf{Mw}_{j}^0(\mathbf{M}_{\perp} \mathbf{w}_{j}^0)^\top   \right)\mathbf{M}_{\perp} \mathbf{w}_{j}^0   \right\|_2  \nonumber \\
    &\leq  \left\|    \mathbf{A}_{\parallel,\parallel}(\mathbf{w}_{j}^0) - \tfrac{1}{2\pi}\mathbf{I}_r \right\|_2\left\|\mathbf{Mw}_{j}^0 \right\|_2 \nonumber \\
    &\quad +  \left\|  \mathbf{A}_{\parallel,\perp}(\mathbf{w}_{j}^0) -  \frac{1 }{2{\pi} \|\mathbf{M}_{\perp} \mathbf{w}_{j}^0\|_2^2} \mathbf{Mw}_{j}^0(\mathbf{M}_{\perp} \mathbf{w}_{j}^0)^\top   \right\|_2 \left\|\mathbf{M}_{\perp} \mathbf{w}_{j}^0   \right\|_2  \nonumber \\
    &= O\left(\nu_{\mathbf{w}}\frac{r^4 + \log^4(m/\delta_1)}{d} \right),  \label{xvx}
\end{align} where \eqref{xvx} follows by  Lemmas \ref{lem:apara} and \ref{lem:aparaperp} and the fact that $\mathbf{w}_j^0 \in \mathcal{G}_{\mathbf{w}}(\delta_1)$. 
Similarly, 
\begin{align}
    \mathbf{M}_{\perp} \mathbf{w}_{j}^{1} &= (1 - \eta \lambda_{\mathbf{w}}) \mathbf{M}_{\perp} \mathbf{w}_{j}^0  + {\eta^2 2^{-r}}  \frac{ \mathbf{M}_{\perp} \mathbf{w}_{j}^0 (\mathbf{Mw}_{j}^0)^\top }{2{\pi}\|\mathbf{M}_{\perp} \mathbf{w}_{j}^0\|_2^2}\mathbf{Mw}_{j}^0 \nonumber \\
    &\quad + \left( 1 - \frac{\|\mathbf{Mw}_{j}^0\|_2^2}{\|\mathbf{M}_{\perp} \mathbf{w}_{j}^0\|_2^2}  \right) {\eta^2 2^{-r}}  \frac{ \mathbf{M}_{\perp} \mathbf{w}_{j}^0(\mathbf{M}_{\perp} \mathbf{w}_{j}^0)^\top }{2{\pi}\|\mathbf{M}_{\perp} \mathbf{w}_{j}^0\|_2^2}\mathbf{M}_{\perp} \mathbf{w}_{j}^0 \nonumber \\
    &\quad + {\eta^2 2^{-r}}\left(\mathbf{A}_{\perp,\parallel}(\mathbf{w}_{j}^0)- \frac{ \mathbf{M}_{\perp} \mathbf{w}_{j}^0(\mathbf{Mw}^0_{j})^\top }{2{\pi}\|\mathbf{M}_{\perp} \mathbf{w}_{j}^0\|_2^2} \right)\mathbf{Mw}_{j}^0 \nonumber \\
    &\quad + {\eta^2 2^{-r}}\left( \mathbf{A}_{\perp,\perp}(\mathbf{w}_{j}^0)  -  \left( 1 - \frac{\|\mathbf{Mw}_{j}^0\|_2^2}{\|\mathbf{M}_{\perp} \mathbf{w}_{j}^0\|_2^2}  \right) \frac{ \mathbf{M}_{\perp} \mathbf{w}_{j}^0(\mathbf{M}_{\perp} \mathbf{w}_{j}^0)^\top }{2{\pi}\|\mathbf{M}_{\perp} \mathbf{w}_{j}^0\|_2^2} \right)\mathbf{M}_{\perp} \mathbf{w}_{j}^0 + \eta \mathbf{M}_{\perp}\mathbf{e}\nonumber \\
    &=  (1 - \eta \lambda_{\mathbf{w}}) \mathbf{M}_{\perp} \mathbf{w}_{j}^0  +  {\eta^2 2^{-r}}  \frac{ \mathbf{M}_{\perp} \mathbf{w}_{j}^0(\mathbf{M}_{\perp} \mathbf{w}_{j}^0)^\top }{2{\pi}\|\mathbf{M}_{\perp} \mathbf{w}_{j}^0\|_2^2}\mathbf{M}_{\perp} \mathbf{w}_{j}^0 \nonumber \\
    &\quad + {\eta^2 2^{-r}}\left(\mathbf{A}_{\perp,\parallel}(\mathbf{w}_{j}^0)- \frac{ \mathbf{M}_{\perp} \mathbf{w}_{j}^0(\mathbf{Mw}_{j}^0)^\top }{2{\pi}\|\mathbf{M}_{\perp} \mathbf{w}_{j}^0\|_2^2} \right)\mathbf{Mw}_{j}^0 \nonumber \\
    &\quad + {\eta^2 2^{-r}}\left( \mathbf{A}_{\perp,\perp}(\mathbf{w}_{j}^0)  -  \left( 1 - \frac{\|\mathbf{Mw}_{j}^0\|_2^2}{\|\mathbf{M}_{\perp} \mathbf{w}_{j}^0\|_2^2}  \right) \frac{ \mathbf{M}_{\perp} \mathbf{w}_{j}^0(\mathbf{M}_{\perp} \mathbf{w}_{j}^0)^\top }{2{\pi}\|\mathbf{M}_{\perp} \mathbf{w}_{j}^0\|_2^2} \right)\mathbf{M}_{\perp} \mathbf{w}_{j}^0 + \eta \mathbf{M}_{\perp}\mathbf{e}\nonumber \\
    &= \left(1 - \eta \lambda_{\mathbf{w}} + \frac{\eta^2 }{ 2^{r+1} \pi} \right) \mathbf{M}_{\perp} \mathbf{w}_{j}^0  + {\eta^2 2^{-r}}\mathbf{e}_\perp + \eta \mathbf{M}_{\perp}\mathbf{e} 
\end{align}
where 
\begin{align}
  \|  \mathbf{e}_{\perp}\| &= \Bigg\|\left(\mathbf{A}_{\perp,\parallel}(\mathbf{w}_{j}^0)- \frac{ \mathbf{M}_{\perp} \mathbf{w}_{j}^0(\mathbf{Mw}_{j}^0)^\top }{2{\pi}\|\mathbf{M}_{\perp} \mathbf{w}_{j}^0\|_2^2} \right)\mathbf{Mw}_{j}^0 \nonumber \\
  &\quad \quad + \left( \mathbf{A}_{\perp,\perp}(\mathbf{w}_{j}^0)  -  \left( 1 - \frac{\|\mathbf{Mw}_{j}^0\|_2^2}{\|\mathbf{M}_{\perp} \mathbf{w}_{j}^0\|_2^2}  \right) \frac{ \mathbf{M}_{\perp} \mathbf{w}_{j}^0(\mathbf{M}_{\perp} \mathbf{w}_{j}^0)^\top }{2{\pi}\|\mathbf{M}_{\perp} \mathbf{w}_{j}^0\|_2^2} \right)\mathbf{M}_{\perp} \mathbf{w}_{j}^0 \Bigg\|_2 \nonumber \\
  &\leq \left\|\mathbf{A}_{\perp,\parallel}(\mathbf{w}_{j}^0)- \frac{ \mathbf{M}_{\perp} \mathbf{w}_{j}^0(\mathbf{Mw}_{j}^0)^\top }{2{\pi}\|\mathbf{M}_{\perp} \mathbf{w}_{j}^0\|_2^2} \right\|_2\|\mathbf{Mw}_{j}^0 \|_2\nonumber \\
  &\quad \quad + \left\| \mathbf{A}_{\perp,\perp}(\mathbf{w}_{j}^0)  -  \left( 1 - \frac{\|\mathbf{Mw}_{j}^0\|_2^2}{\|\mathbf{M}_{\perp} \mathbf{w}_{j}^0\|_2^2}  \right) \frac{ \mathbf{M}_{\perp} \mathbf{w}_{j}^0(\mathbf{M}_{\perp} \mathbf{w}_{j}^0)^\top }{2{\pi}\|\mathbf{M}_{\perp} \mathbf{w}_{j}^0\|_2^2} \right\|_2 \|\mathbf{M}_{\perp} \mathbf{w}_{j}^0\|_2 \nonumber \\
  &= O\left( \nu_{\mathbf{w}}\frac{r^{3.5} + \log^{3.5}(m/\delta_1)}{d^{1.5}} \right) \label{rfr}
\end{align}
where \eqref{rfr} follows by  Lemmas \ref{lem:aperp} and \ref{lem:aparaperp} and the fact that $\mathbf{w}_j^0\in \mathcal{G}_{\mathbf{w}}(\delta_1)$.
Applying the choice of $\lambda_{\mathbf{w}}$ and setting $\delta_1,\delta_2=\Theta(\delta)$ completes the proof.
\end{proof}
Finally we are ready to prove Proposition \ref{prop:1} and Theorem \ref{thm:1}. For convenience, we restate the statements in full detail here.

\begin{proposition} \label{prop:1:app}
    Consider the gradient-based multi-task algorithm described in Section~\ref{sec:alg} and suppose Assumption \ref{assump:diversity} holds. Further  let $\eta = \Theta(1)$,  $\lambda_{\mathbf{w}} = 1/\eta + \eta/(2^{r+1}\pi)$, and $\nu_{\mathbf{w}} = O(d^{-5/4}(m\log(T/\delta))^{-1/2})$. 
Then for any $m= O(d)$ and $\delta = \Omega(e^{-d})$,  with probability at least $1-\delta$ we have
\begin{enumerate}
\item $\frac{1}{\nu_{\mathbf{w}}\sqrt{{m}}}\|\Pi_{\parallel}(\mathbf{W}^1) -  \Omega(\frac{1}{2^r}) 
\Pi_{\parallel}(\mathbf{W}^{0}) \|_2 \\ 
\quad \quad \quad \quad = O\left(  
\frac{r^4 + \log^4(m/\delta)}{2^r d} 
+  \frac{d{\log(dTn_2/\delta)} }{\sqrt{Tn_2}} \left( 1+ \frac{\sqrt{\log(T/\delta)}}{\sqrt{n_1}} \right) + {\frac{\sqrt{dr} \log(dm/\delta) }{\sqrt{T}}}  
\right) $,
    \item $\frac{1}{\nu_{\mathbf{w}}\sqrt{{m}}}\|\Pi_{\perp}(\mathbf{W}^{+}) \|_2 
    = O\left(
\frac{r^{3.5} + \log^{3.5}(m/\delta)}{2^r d^{1.5}} 
   +  \frac{d{\log(dTn_2/\delta)} }{\sqrt{Tn_2}} \left( 1+ \frac{\sqrt{\log(T/\delta)}}{\sqrt{n_1}} \right) + {\frac{\sqrt{dr} \log(dm/\delta) }{\sqrt{T}}}   \right).$
\end{enumerate}
\end{proposition}

\begin{proof}
 The result is a direct consequence of 
Lemma \ref{lem:final}, the additional conditions on $\eta, m$ and $\nu_{\mathbf{w}}$ (which make the $  \exp \left( -\frac{ c }{\eta^2  \nu_{\mathbf{w}}^2 {\log(T/\delta)}dm ({d} + {m} )} \right)$ term negligible), and the fact that $\|\mathbf{B}\|_2\leq \sqrt{{m}}\max_{j\in [m]} \|\mathbf{b}_j\|_2$ for any matrix $\mathbf{B}\in \mathbb{R}^{m \times d}$, where $\mathbf{b}_j$ is the $j$-th row of $\mathbf{B}$. We use $ab \leq a^2 + b^2\leq (a+b)^2$ for nonnegative $a,b$ to combine $\log$ terms.
\end{proof}


\begin{theorem}[  ]\label{thm:1:app} 
    Consider that $\eta$, $\nu_{\mathbf{w}}$ and $\lambda_{\mathbf{w}}$ are set as in Proposition \ref{prop:1} and let $d =\Omega( r^4 + \log^4(m/\delta) )$, $m = \Omega(r + \log(1/\delta))$ and $m= O(d)$, $\delta=\Omega(e^{-d})$, 
    $ T = \Omega( 2^{2r} d r\log^2(dm/\delta)) $, $Tn_2 = \Omega(2^{2r} d^2\log^2(dTn_2))$, and $Tn_1 n_2 = \Omega(2^{2r} d^2(\log^2(dTn_2/\delta)\log(T/\delta))$. Let $\sigma_r(\mathbf{B})$ denote the $r$-th singular value of the matrix $\mathbf{B}$.
Then with probability at least $1-\delta $, 
we have 
\begin{align}  
    \frac{\sigma_1 (\Pi_{\perp}(\mathbf{W}^1) ) }{\sigma_{r}( \Pi_{\parallel}(\mathbf{W}^1) )} = 
     O\left(
\frac{r^{3.5} + \log^{3.5}(m/\delta)}{d^{1.5}} 
+    \frac{2^r d{\log(dTn_2/\delta)} }{\sqrt{Tn_2}} \left( 1+ \frac{\sqrt{\log(T/\delta)}}{\sqrt{n_1}} \right) + {\frac{2^r \sqrt{dr} \log(dm/\delta) }{\sqrt{T}}} 
    \right)
\end{align}
\end{theorem}

\begin{proof}
By Lemma \ref{lem:all}, we have that $E_{\mathbf{w}}(\delta_1)$ holds with probability at least $1-\delta_1$, which entails that
 $ \sigma_r(\Pi_{\parallel}(\mathbf{W}^0) )  \geq  \nu_{\mathbf{w}}\sqrt{{m}}\left(1 - c\frac{\sqrt{r} +\sqrt{\log(1/\delta_1)}}{\sqrt{{m}}}\right)$ for an absolute constant $c$. 
Thus we can invoke Lemma \ref{lem:final} to obtain
\begin{align}
   & \sigma_{r}( \Pi_{\parallel}(\mathbf{W}^1) )\nonumber \\
   &\geq \frac{\eta}{2\pi} 2^{-r-2}\sigma_{r}( \Pi_{\parallel}(\mathbf{W}^0) ) \nonumber \\
    &\quad \quad \quad \quad -  \nu_{\mathbf{w}} \sqrt{{m}} O\left(  
\tfrac{r^4 + \log^4(m/\delta_2)}{2^r d} 
+  \tfrac{d{\log(dTn_2/\delta_2)} }{\sqrt{Tn_2}} \left( 1+ \tfrac{\sqrt{\log(T/\delta_2)}}{\sqrt{n_1}} \right) + {\tfrac{ \sqrt{dr} \log(dm/\delta_2) }{\sqrt{T}}} 
\right) \nonumber \\
&\geq \frac{\eta}{2\pi} 2^{-r} \nu_{\mathbf{w}} \sqrt{{m}}  \nonumber \\
&\; \;  - 2^{-r} \nu_{\mathbf{w}} \sqrt{{m}}\nonumber \\
&\quad \times  O\left( 
\tfrac{\sqrt{r} +\sqrt{\log(1/\delta_1)}}{ \sqrt{{m}}} +   
\tfrac{r^4 + \log^4(m/\delta_2)}{ d} 
+   \tfrac{2^r d{\log(dTn_2/\delta_2)} }{\sqrt{Tn_2}} \left( 1+ \tfrac{\sqrt{\log(T/\delta_2)}}{\sqrt{n_1}} \right) + {\tfrac{2^r \sqrt{dr} \log(dm/\delta_2) }{\sqrt{T}}} 
\right)  \nonumber \\
&= \Omega(2^{-r} \nu_{\mathbf{w}} \sqrt{{m}}) \label{zm}
\end{align}
 with probability at least $1-3\delta_1 -\delta_2$.
 Likewise, we have by Proposition \ref{prop:1:app}
\begin{align}
    \sigma_1(\Pi_{\perp}(\mathbf{W}^1)) &= 2^{-r}\nu_{\mathbf{w}} \sqrt{{m}} \;  O\left(
\tfrac{r^{3.5} + \log^{3.5}(m/\delta_2)}{d^{1.5}} 
+    \tfrac{2^r d{\log(dTn_2/\delta_2)} }{\sqrt{Tn_2}} \left( 1+ \tfrac{\sqrt{\log(T/\delta_2)}}{\sqrt{n_1}} \right) + {\tfrac{2^r \sqrt{dr} \log(dm/\delta_2) }{\sqrt{T}}} 
    \right). \label{nz}
\end{align}
 with probability at least $1 -\delta_2$.
Combining \eqref{zm} and \eqref{nz} and setting $\delta_1,\delta_2 = \Theta(\delta)$ completes the proof.
\end{proof}


    







\newpage
\section{Proof of Downstream Guarantees}\label{app:down}
In this section we prove Theorem \ref{thm:downstream} and a corollary thereof. Given a function $g:\mathcal{H}^r \rightarrow \{-1,1\}$ and representation $\mathbf{M}\in \mathbb{O}^{r\times d}$, we refer to the sets $\mathcal{V}^+ \coloneqq \{\mathbf{v}= \mathbf{Mz}^\top + \mathbf{M}_{\perp}^\top \boldsymbol{\xi}: \mathbf{z} \in \mathcal{H}^r, \boldsymbol{\xi} \in \mathcal{H}^{d-r}, g(\mathbf{z}) = 1 \}$ and $\mathcal{V}^- \coloneqq \{ \mathbf{v}=\mathbf{Mz}^\top + \mathbf{M}_{\perp}^\top \boldsymbol{\xi}: \mathbf{z} \in \mathcal{H}^r, \boldsymbol{\xi} \in \mathcal{H}^{d-r}, g(\mathbf{z}) = -1 \}$ as the inverse sets of $g$. Note that solving the task described by $g$ entails finding a classifier that separates the inverse sets of $g$. As in Appendix \ref{app:thrm1}, we will abuse notation by reusing $c$ and $c'$ as absolute constants independent of all other parameters. 

\textbf{Proof summary.} Informally, the proof of Theorem \ref{thm:downstream}
 follows six steps:
 \begin{enumerate}
     \item Construct a two-layer ReLU network embedding using the weights output by the during multi-task pretraining algorithm for the first-layer weights, then randomly sampling first-layer biases and second layer weights and biases (calling this the ``learned'' embedding),
     \item Construct a nearby two-layer ReLU network embedding that is ``purified'' in the sense that it is only a function of the the $r$ label-relevant features of the input,
     \item Show that the ``purified'' network linearly separates the pair of inverse sets corresponding to any binary function on the $r$-dimensional hypercube with high probability as long as the number of neurons in each layer  is larger than some function of $r$,
     \item Prove that the outputs of the learned embedding are very close to the outputs of the ``purified'' embedding, meaning the learned embedding has the same linear separation capability with only slightly smaller margin,
     \item Prove that the learned embedding linearly separates the inverse sets with lower bounded margin, and finally 
     \item Apply a standard generalization result for linear classification showing that the empirically-optimal head achieves loss close to the minimal loss (zero).
 \end{enumerate} 

\textbf{Step 1: Construct downstream embedding.} We start by fully describing the construction of the downstream classifier as described first in Section \ref{sec:alg}. Let $\mathbf{W}^1$ be the model weights resulting from one step of the multitask representation learning algorithm. 


The downstream classifier is a linear head composed with a two-layer ReLU network embedding 
with ${m}$ neurons in the first layer and $\hat{m}$ neurons in the second layer. For now, we focus on the embedding itself, excluding the linear classification head. The weights of the first layer of the embedding are equal to the weights in ${\mathbf{W}}^1$ up to rescaling. The biases of the first layer and weights and biases of the second layer are contained in $\mathbf{b}$,
 $\hat{\mathbf{W}} := [\hat{\mathbf{w}}_1,\dots,\hat{\mathbf{w}}_{\hat{m}}]^\top $ and $ \hat{\mathbf{b}}$, respectively, where
    \begin{align}
        \mathbf{b} &\sim \text{Unif}\left(\left[-\frac{\sqrt{2}\gamma}{\sqrt{{m}}},\frac{\sqrt{2}\gamma}{\sqrt{{m}}}\right]^m\right) \nonumber \\
         \hat{\mathbf{w}}_j &\sim \mathcal{N}\left(\mathbf{0}_m, \frac{2}{\hat{m}}\mathbf{I}_{{m}}\right) \quad \forall \; j = 1,\dots, \hat{m} \nonumber\\
        \hat{\mathbf{b}} &\sim \text{Unif}\left(\left[-\frac{\sqrt{2}\hat{\gamma}}{\sqrt{\hat{m}}},\frac{\sqrt{2}\hat{\gamma}}{\sqrt{\hat{m}}}\right]^{\hat{m}}\right) \nonumber 
    \end{align}
    for some $\gamma,\hat{\gamma}>0$ to be defined later.
The full embedding is given by:
    \begin{align}
    \phi(\mathbf{v}; \alpha{\mathbf{W}}^1, \mathbf{b}, \hat{\mathbf{W}}, \hat{\mathbf{b}}) :=  \sigma( \hat{\mathbf{W}}\sigma(\alpha {\mathbf{W}}^1\mathbf{v}+ {\mathbf{b}}) + \hat{\mathbf{b}} ) \quad \forall \;\mathbf{v} \in \mathcal{H}^d \label{qw}
\end{align}
where $\alpha := \frac{2^{r+2.5}}{\sqrt{{m}} \eta^2 }$ is a rescaling factor.  For ease of notation we denote $\phi(\mathbf{v}):= \phi(\mathbf{v}; \alpha{\mathbf{W}}^1, \mathbf{b}, \hat{\mathbf{W}}, \hat{\mathbf{b}}) $.

\textbf{Step 2: Construct purified downstream embedding.} Next, the ``purified'' embedding also has ${m}$ neurons in the first layer and $\hat{m}$ neurons in the second layer, and is also parameterized by $\mathbf{b}$,
 ${\mathbf{W}} $ and $ \mathbf{b}$, for the first layer biases and second layer weights and biases, respectively, but has a different construction of the first layer weights. 
 In particular, the first-layer weights are equal to  the component of the corresponding weights in $\mathbf{W}^0$ in the rowspace of $\mathbf{M}$, up to rescaling.
 Formally, this embedding is given by 
 \begin{align}
 \tilde{\phi}(\mathbf{v}; \mathbf{W}^0\mathbf{M}^\top \mathbf{M}, \mathbf{b}, \hat{\mathbf{W}}, \hat{\mathbf{b}} ) := \sigma( \hat{\mathbf{W}}\sigma(\hat{\alpha}{\mathbf{W}}^0\mathbf{M}^\top \mathbf{M}\mathbf{v}+ {\mathbf{b}}) + \hat{\mathbf{b}} ) 
 \end{align}
 where $\hat{\alpha} = \frac{2}{\nu_{\mathbf{w}}\sqrt{m}} $.

For ease of notation we denote $\tilde{\phi}(\mathbf{v}):= \tilde{\phi}(\mathbf{v}; \hat{\alpha}{\mathbf{W}}^0\mathbf{M}^\top \mathbf{M}, \mathbf{b}, \hat{\mathbf{W}}, \hat{\mathbf{b}} )  $.
 
 \textbf{Step 3: Purified embedding linearly separates the two classes.} We start by showing that for any function $\bar{g}:\{-\frac{1}{\sqrt{r}},\frac{1}{\sqrt{r}}\}^r\rightarrow \{-1,1\}$, with high probability $\tilde{\phi}$ linearly separates the pair of inverse sets $\bar{\mathcal{V}}^+ \coloneqq \{ \bar{\mathbf{v}}= \mathbf{M}^\top \bar{\mathbf{z}} + \mathbf{M}_{\perp}^\top \boldsymbol{\xi}: \bar{\mathbf{z}} \in \{-\frac{1}{\sqrt{r}},\frac{1}{\sqrt{r}}\}^r, \boldsymbol{\xi}\in \mathcal{H}^{d-r}, \bar{g}(\bar{\mathbf{z}}) = 1 \}$ and $\bar{\mathcal{V}}^- \coloneqq \{ \bar{\mathbf{v}}= \mathbf{M}^\top \bar{\mathbf{z}} + \mathbf{M}_{\perp}^\top \boldsymbol{\xi}: \bar{\mathbf{z}} \in \{-\frac{1}{\sqrt{r}},\frac{1}{\sqrt{r}}\}^r, \boldsymbol{\xi}\in \mathcal{H}^{d-r}, \bar{g}(\bar{\mathbf{z}}) = -1 \}$ with lower bounded margin
 by adapting a result of \cite{dirksen2022separation}. Note that we have not optimized the dependence of $m$ on $\log()$ and the dependence of $\hat{m}$ on $\log()$ factors in the exponent.

\begin{lemma}[Adapted from Theorem 1 in \cite{dirksen2022separation}] \label{lem:pure}
    Let  $\delta \in (0,0.05]$, 
$\gamma = \Theta(\log({r}))$, $\hat{\gamma} = \Theta( r^{2.5} \log^4(r))$,
    $m = \Omega( r^5\log^8(r) {\log(1/\delta)})$, and
    $\hat{m} = 
    \exp\left(\Omega\left( m
    \right) \right)
    $.
    Consider any function $g:\{-\frac{1}{\sqrt{r}},\frac{1}{\sqrt{r}}\}^r \rightarrow \{-1,1\}$. With probability at least $1 - \delta$, $\tilde{\phi}$ makes the classes $\bar{\mathcal{V}}^+ \coloneqq \{ \bar{\mathbf{v}}= \mathbf{M}^\top \bar{\mathbf{z}} + \mathbf{M}_{\perp}^\top \boldsymbol{\xi}: \bar{\mathbf{z}} \in \{-\frac{1}{\sqrt{r}},\frac{1}{\sqrt{r}}\}^r, \boldsymbol{\xi}\in \mathcal{H}^{d-r}, \bar{g}(\bar{\mathbf{z}}) = 1 \}$ and $\bar{\mathcal{V}}^- \coloneqq \{ \bar{\mathbf{v}}= \mathbf{M}^\top \bar{\mathbf{z}} + \mathbf{M}_{\perp}^\top \boldsymbol{\xi}: \bar{\mathbf{z}} \in \{-\frac{1}{\sqrt{r}},\frac{1}{\sqrt{r}}\}^r, \boldsymbol{\xi}\in \mathcal{H}^{d-r}, \bar{g}(\bar{\mathbf{z}}) = -1 \}$ linearly separable with margin 
    $\mu = \exp \left(- O(r^5 \log^6(r)  \log(\log(r)/\delta) )\right)$,
    i.e. there exists a vector $\mathbf{a}\in \mathbb{R}^{\hat{m} }$ with $\|\mathbf{a}\|_2=1$ and bias $\tau\in \mathbb{R}$ such that for all $\bar{\mathbf{v}}= \mathbf{M}^\top \bar{\mathbf{z}} + \mathbf{M}_{\perp}^\top \boldsymbol{\xi}: \bar{\mathbf{z}} \in \{-\frac{1}{\sqrt{r}},\frac{1}{\sqrt{r}}\}^r, \boldsymbol{\xi}\in \mathcal{H}^{d-r}$
    \begin{align}
       \bar{g}(\bar{\mathbf{z}}) = 1 &\implies  \mathbf{a}^\top \tilde{\phi}(\bar{\mathbf{v}})+ \tau > \mu    \nonumber\\
        \bar{g}(\bar{\mathbf{z}}) = -1 &\implies  \mathbf{a}^\top \tilde{\phi}( \bar{\mathbf{v}}) + \tau < -\mu \nonumber   
    \end{align}
\end{lemma}

\begin{proof}
First note that given $\bar{\mathbf{v}}= \mathbf{M}^\top \bar{\mathbf{z}} + \mathbf{M}_{\perp}^\top \boldsymbol{\xi}$, $\tilde{\phi}(\bar{v}) = \tilde{\phi}'(\bar{\mathbf{z}})$ for a random network $\tilde{\phi}':\mathbb{R}^r\rightarrow \mathbb{R}$.
So, the problem reduces to showing whether $\tilde{\phi}'$ linear separates the classes $\bar{\mathcal{Z}}^+ \coloneqq \{ \bar{\mathbf{z}}\in \{-\frac{1}{\sqrt{r}},\frac{1}{\sqrt{r}}\}^r: \bar{g}(\bar{\mathbf{z}}) = 1 \}$ and $\bar{\mathcal{Z}}^- \coloneqq \{ \bar{\mathbf{z}} \in \{-\frac{1}{\sqrt{r}},\frac{1}{\sqrt{r}}\}^r:  \bar{g}(\bar{\mathbf{z}}) = -1 \}$. The construction of $\tilde{\phi}'$ 
matches that in Theorem 1 in \cite{dirksen2022separation}, so we can directly apply this theorem. Note that to compute the margin, we use that the distance between $\bar{\mathcal{Z}}^+$ and $\bar{\mathcal{Z}}^-$ is $\frac{2}{\sqrt{r}}$ and $|\bar{\mathcal{Z}}^+||\bar{\mathcal{Z}}^-|\leq 2^{2r}$.
\end{proof}

Next we extend the above result to the case in which $\mathbf{z}$ is  on the Rademacher hypercube.
\begin{lemma} \label{lem:pure2}
    Let 
$\gamma = \Theta(\sqrt{r}\log({r}))$, $\hat{\gamma} = \Theta( r^{3} \log^4(r))$,
and $m$, $\hat{m}$ satisfy the same conditions as in Lemma \ref{lem:pure}, for any $\delta \in (0,0.05]$.
    Consider any function $g:\mathcal{H}^r \rightarrow \{-1,1\}$. With probability at least $1 - \delta$, $\tilde{\phi}$ makes the classes 
    ${\mathcal{V}}^+ \coloneqq \{ {\mathbf{v}}= \mathbf{M}^\top {\mathbf{z}} + \mathbf{M}_{\perp}^\top \boldsymbol{\xi}: {\mathbf{z}} \in \mathcal{H}^r, \boldsymbol{\xi}\in \mathcal{H}^{d-r}, {g}({\mathbf{z}}) = 1 \}$ and ${\mathcal{V}}^- \coloneqq \{ {\mathbf{v}}= \mathbf{M}^\top {\mathbf{z}} + \mathbf{M}_{\perp}^\top \boldsymbol{\xi}: {\mathbf{z}} \in \mathcal{H}^r, \boldsymbol{\xi}\in \mathcal{H}^{d-r}, {g}({\mathbf{z}}) = -1 \}$
    linearly separable with margin 
    $\mu = 
    \exp( -O ( r^5 \log^6(r)  \log(\log(r)/\delta) ))$.
\end{lemma}
\begin{proof}
Let $\tilde{\phi}_{\sqrt{r}}:\mathbb{R}^d\rightarrow \mathbb{R}^{\hat{m}}$ denote the version of $\tilde{\phi}$ that was constructed in  Lemma \ref{lem:pure}, with $\gamma$ and $\gamma'$ set accordingly. Construct the coupled network $\tilde{\phi}:\mathbb{R}^d\rightarrow \mathbb{R}^{\hat{m}}$ by scaling the biases up by $\sqrt{r}$, so the $\gamma,\gamma$ corresponding to $\tilde{\phi}$ have the scaling defined in the current  lemma statement.
Note that $$\tilde{\phi}(\sqrt{r}\mathbf{M}^\top\bar{\mathbf{z}} + \mathbf{M}_{\perp}^\top \boldsymbol{\xi}; \mathbf{W}^0\mathbf{M}^\top \mathbf{M}, \sqrt{r}\mathbf{b}, \hat{\mathbf{W}}, \sqrt{r}\hat{\mathbf{b}} ) = \sqrt{r}\tilde{\phi}_{\sqrt{r}}(\mathbf{M}^\top\bar{\mathbf{z}}+ \mathbf{M}_{\perp}^\top \boldsymbol{\xi}), $$ thus if $\bar{\mathcal{V}}^+ \coloneqq \{ \bar{\mathbf{v}}= \mathbf{M}^\top \bar{\mathbf{z}} + \mathbf{M}_{\perp}^\top \boldsymbol{\xi}: \bar{\mathbf{z}} \in \{-\frac{1}{\sqrt{r}},\frac{1}{\sqrt{r}}\}^r, \boldsymbol{\xi}\in \mathcal{H}^{d-r}, \bar{g}(\bar{\mathbf{z}}) = 1 \}$ and $\bar{\mathcal{V}}^- \coloneqq \{ \bar{\mathbf{v}}= \mathbf{M}^\top \bar{\mathbf{z}} + \mathbf{M}_{\perp}^\top \boldsymbol{\xi}: \bar{\mathbf{z}} \in \{-\frac{1}{\sqrt{r}},\frac{1}{\sqrt{r}}\}^r, \boldsymbol{\xi}\in \mathcal{H}^{d-r}, \bar{g}(\bar{\mathbf{z}}) = -1 \}$ are linearly separated with margin $\mu'$ by $\tilde{\phi}_{\sqrt{r}}$, then ${\mathcal{V}}^+ \coloneqq \{ {\mathbf{v}}= \sqrt{r}\mathbf{M}^\top \bar{\mathbf{z}} + \mathbf{M}_{\perp}^\top \boldsymbol{\xi}: \sqrt{r}\bar{\mathbf{z}} \in \mathcal{H}^r, \boldsymbol{\xi}\in \mathcal{H}^{d-r}, {g}(\sqrt{r}\bar{\mathbf{z}}) = 1 \}$ and ${\mathcal{V}}^- \coloneqq \{ {\mathbf{v}}= \sqrt{r}\mathbf{M}^\top \bar{\mathbf{z}} + \mathbf{M}_{\perp}^\top \boldsymbol{\xi}: \sqrt{r}\bar{\mathbf{z}} \in \mathcal{H}^r, \boldsymbol{\xi}\in \mathcal{H}^{d-r}, {g}(\sqrt{r}\bar{\mathbf{z}}) = -1 \}$ are linearly separated by $\tilde{\phi}$ with margin $\sqrt{r}\mu'$. So the result follows from Lemma \ref{lem:pure}.
\end{proof}


    




    

\textbf{Step 4: Downstream embedding is close to purified embedding.} Next we compare the outputs of the ``purified'' embedding $\tilde{\phi}$  to those of the learned embedding $\phi$. 


\begin{lemma} \label{lem:true}
    Suppose the same conditions as Lemma \ref{lem:pure2} hold. Additionally, suppose $\eta = \Theta(1)$,  $\lambda_{\mathbf{w}} = 1/\eta + \eta/(2^{r+1}\pi)$, and $\nu_{\mathbf{w}} = O(d^{-5/4}(m\log(T/\delta))^{-1/2})$, $m=O(d)$, $\delta=\Omega(e^{-d})$ and Assumption \ref{assump:diversity} holds,
then  with probability at least $1-\delta$,  for all ${\mathbf{v}}= \mathbf{M}^\top {\mathbf{z}} + \mathbf{M}_{\perp}^\top \boldsymbol{\xi}: {\mathbf{z}} \in \mathcal{H}^r, \boldsymbol{\xi}\in \mathcal{H}^{d-r}$, 
\begin{align}
   & \| \tilde{\phi}( \mathbf{v}) - \phi( \mathbf{v})\|_2 \nonumber \\
   &= O\left(\frac{r^{3.5} + \log^{3.5}(m/\delta)}{2^r \sqrt{d}} + 
+  \frac{d^{1.5}{\log(dTn_2/\delta)} }{\sqrt{Tn_2}} \left( 1+ \frac{\sqrt{\log(T/\delta)}}{\sqrt{n_1}} \right) + {\frac{ d\sqrt{r} \log(dm/\delta) }{\sqrt{T}}} 
    \right). \nonumber 
\end{align}
\end{lemma}

\begin{proof}
For any ${\mathbf{v}}= \mathbf{M}^\top {\mathbf{z}} + \mathbf{M}_{\perp}^\top \boldsymbol{\xi}: {\mathbf{z}} \in \mathcal{H}^r, \boldsymbol{\xi}\in \mathcal{H}^{d-r}$, we have
\begin{align}
    &\|\tilde{\phi}( \mathbf{v}) - \phi( \mathbf{v})\|_2 \nonumber \\
    &= \left\|\sigma\left( \hat{\mathbf{W}}\sigma\left(\frac{2}{\sqrt{m}\nu_{\mathbf{w}} }{\mathbf{W}}^0\mathbf{M}^\top\mathbf{Mv} + \mathbf{b}\right) + \hat{\mathbf{b}} \right) -  \sigma\left( \hat{\mathbf{W}}\sigma\left(\frac{2}{\sqrt{m}\nu_{\mathbf{w}} \eta^2 2^{-r-2} }{\mathbf{W}}^1\mathbf{v}+ \mathbf{b}\right) + \hat{\mathbf{b}} \right)  \right\|_2 \nonumber \\
    &= \left\|\sigma\left( \hat{\mathbf{W}}\sigma\left(\frac{2}{\sqrt{m}\nu_{\mathbf{w}} }{\mathbf{W}}^0\mathbf{M}^\top\mathbf{z} + \mathbf{b}\right) + \hat{\mathbf{b}} \right) -  \sigma\left( \hat{\mathbf{W}}\sigma\left(\frac{2}{\sqrt{m}\nu_{\mathbf{w}} \eta^2 2^{-r-2} }{\mathbf{W}}^1\mathbf{v}+ \mathbf{b}\right) + \hat{\mathbf{b}} \right)  \right\|_2 \nonumber \\
    &\leq \left\| \hat{\mathbf{W}}\sigma\left(\frac{2}{\sqrt{m}\nu_{\mathbf{w}}  }{\mathbf{W}}^{0}\mathbf{M}^\top \mathbf{z}+ \mathbf{b}\right)  -   \hat{\mathbf{W}}\sigma\left(\frac{2}{\sqrt{m}\nu_{\mathbf{w}} \eta^2 2^{-r-2} }{\mathbf{W}}^{1}\mathbf{v}+ \mathbf{b} \right)   \right\|_2 \label{lip1} \\
    &\leq \|\hat{\mathbf{W}}\|_2 \left\| \frac{2}{\sqrt{m}\nu_{\mathbf{w}}  }{\mathbf{W}}^{0}\mathbf{M}^\top\mathbf{z} -   \frac{2}{\sqrt{m}\nu_{\mathbf{w}} \eta^2 2^{-r-2}}{\mathbf{W}}^{1}\mathbf{v}   \right\|_2 \label{lip2}\\
    &=  \frac{2}{\sqrt{m}\nu_{\mathbf{w}}  }\|\hat{\mathbf{W}}\|_2 \left\|{\mathbf{W}}^{0}\mathbf{M}^\top\mathbf{z} -   \frac{1}{\eta^2 2^{-r-2}}{\mathbf{W}}^{1}\mathbf{M}^\top \mathbf{Mv} -  \frac{1}{\eta^2 2^{-r-2}}{\mathbf{W}}^{1}\mathbf{M}_\perp^\top \mathbf{M}_\perp \mathbf{v} \right\|_2 \label{lip22} \\
    &\leq \frac{2}{\sqrt{m}\nu_{\mathbf{w}}  }\|\hat{\mathbf{W}}\|_2 \left\|{\mathbf{W}}^{0}\mathbf{M}^\top \mathbf{z} -   \frac{1}{\eta^2 2^{-r-2} }{\mathbf{W}}^{1}\mathbf{M}^\top\mathbf{z} \right\|_2  + \frac{2}{\sqrt{m}\nu_{\mathbf{w}}  }\|\hat{\mathbf{W}}\|_2 \left\| \frac{1}{\eta^2 2^{-r-2} }{\mathbf{W}}^{1}\mathbf{M}_{\perp}^\top\boldsymbol{\xi} \right\|_2 \label{lip3} \\
    &\leq \frac{2}{\sqrt{m}\nu_{\mathbf{w}}  }\|\hat{\mathbf{W}}\|_2 \left\|{\mathbf{W}}^{0}\mathbf{M}^\top -   \frac{1}{\eta^2 2^{-r-2} }{\mathbf{W}}^{1}\mathbf{M}^\top \right\|_2 \|\mathbf{z}\|_2  + \frac{2}{\sqrt{m}\nu_{\mathbf{w}} {\eta^2 2^{-r-2} } }\|\hat{\mathbf{W}}\|_2 \left\| {\mathbf{W}}^{1}\mathbf{M}_{\perp}^\top  \right\|_2 \|\boldsymbol{\xi} \|_2 \label{lip4} \\
    &= \frac{2\sqrt{r}}{\sqrt{m}\nu_{\mathbf{w}}  }\|\hat{\mathbf{W}}\|_2 \left\|{\mathbf{W}}^{0}\mathbf{M}^\top -   \frac{1}{\eta^2 2^{-r-2} }{\mathbf{W}}^{1}\mathbf{M}^\top \right\|_2  + \frac{2\sqrt{d-r}}{\sqrt{m}\nu_{\mathbf{w}} {\eta^22^{-r-2} } }\|\hat{\mathbf{W}}\|_2 \left\| {\mathbf{W}}^{1}\mathbf{M}_\perp^\top \right\|_2\nonumber \\
    &=  O\left(  \frac{r^{4.5} + \sqrt{r}\log^4(m/\delta_1)}{2^r d } +  \sqrt{r}\epsilon
    \right) \|\hat{\mathbf{W}}\|_2  + O\left(\frac{r^{3.5}+\log^{3.5}(m/\delta_1)}{2^r\sqrt{d}} +  \sqrt{d}\epsilon 
    \right)\|\hat{\mathbf{W}}\|_2\label{lip5} \\ 
    &=  O\left(\frac{r^{3.5}+\log^{3.5}(m/\delta_1)}{2^r\sqrt{d}} +  \sqrt{d} \epsilon \right) \left( 1+ \frac{\sqrt{\log(1/\delta_2)}}{\sqrt{\hat{m}}}\right) \label{lip6} 
\end{align}
where $\epsilon = O\left(    \frac{d{\log(dTn_2/\delta_1)} }{\sqrt{Tn_2}} \left( 1+ \frac{\sqrt{\log(T/\delta_1)}}{\sqrt{n_1}} \right) + {\frac{\sqrt{dr} \log(dm/\delta_1) }{\sqrt{T}}}   \right)$,
\eqref{lip1} and \eqref{lip2} follow since $\sigma()$ is $1$-Lipschitz and by the Cauchy-Schwarz Inequality, \eqref{lip4} follows by the Cauchy-Schwarz Inequality, \eqref{lip5} follows  with probability at least $1-\delta_1$ by Proposition \ref{prop:1:app}, and \eqref{lip6} follows with probability at least $1-\delta_2$ over the random selection of $\hat{\mathbf{W}}$. Setting $\delta_1,\delta_2=\Theta(\delta)$  completes the proof.
\end{proof}




\textbf{Step 5: Downstream embedding linearly separates the two classes.} Now we reason that $\phi$ linearly separates the two classes with high probability.
\begin{lemma} \label{lem:sep}
Suppose the same conditions as Lemma \ref{lem:true} hold. Additionally, suppose \\
$d = \Omega( \log^{7}(m)\exp( c r^5 \log^6(r)\log(\log(r)/\delta) )  )$, $T=\Omega( d^2 r \log^2(dm/\delta ) \exp( c r^5 \log^6(r)\log(\log(r)/\delta) ) )$, and ${Tn_2}{} = \log^2(dTn_2/\delta)(1+\frac{\log(T/\delta)}{n_1}) \Omega(d^3  \exp( c r^5 \log^6(r)\log(\log(r)/\delta) )  )$ for an absolute constant $c$ and $\delta \in(0,0.05]$ such that $\delta = \Omega(e^{-\min(d,\hat{m})})$.
Consider any function $g:\mathcal{H}^r \rightarrow \{-1,1\}$. With probability at least $1 - \delta$, ${\phi}$ makes the classes $\mathcal{V}^+ \coloneqq \{\mathbf{v}= \mathbf{Mz}^\top + \mathbf{M}_{\perp}^\top \boldsymbol{\xi}: \mathbf{z} \in \mathcal{H}^r, \boldsymbol{\xi} \in \mathcal{H}^{d-r}, g(\mathbf{z}) = 1 \}$ and $\mathcal{V}^- \coloneqq \{ \mathbf{v}=\mathbf{Mz}^\top + \mathbf{M}_{\perp}^\top \boldsymbol{\xi}: \mathbf{z} \in \mathcal{H}^r, \boldsymbol{\xi} \in \mathcal{H}^{d-r}, g(\mathbf{z}) = -1 \}$ linearly separable with margin 
$$\mu = \exp(-O( r^5 \log^6(r) \log(\log(r)/\delta) ))$$ 
i.e. there exists a vector $\mathbf{a}\in \mathbb{R}^{\hat{m} }$ with $\|\mathbf{a}\|_2=1$ and bias $\tau\in \mathbb{R}$ such that for all $\mathbf{v}\in \mathcal{H}^d,$
    \begin{align}
       g(\mathbf{v}) = 1 &\implies  \mathbf{a}^\top {\phi}( \mathbf{v})+ \tau > \mu   \nonumber \\
        g(\mathbf{v}) = -1 &\implies  \mathbf{a}^\top \phi(\mathbf{v}) + \tau < -\mu \nonumber   
    \end{align}
\end{lemma}

\begin{proof}
The proof follows from Lemmas \ref{lem:pure2} and \ref{lem:true}. In particular, for any point $\mathbf{v}= \mathbf{Mz}^\top + \mathbf{M}_{\perp}^\top \boldsymbol{\xi}: \mathbf{z} \in \mathcal{H}^r, \boldsymbol{\xi} \in \mathcal{H}^{d-r}, $ from Lemma \ref{lem:pure2} we have with probability at least $1-\delta$, for absolute constants $c,c'$, there exists $\mathbf{a}\in \mathbb{R}^{\hat{m} }$ with $\|\mathbf{a}\|_2=1$ and  $\tau\in \mathbb{R}$ such that
    \begin{align}
       g(\mathbf{z}) = 1 &\implies  \mathbf{a}^\top \tilde{\phi}( \mathbf{v})+ \tau >  \exp(-c r^5 \log^6(r) \log(\log(r)/\delta) )   \nonumber \\
       &\implies  \mathbf{a}^\top {\phi}( \mathbf{v})+ \tau >  \exp(-cr^5 \log^6(r) \log(\log(r)/\delta) )    \nonumber \\
       &\quad \quad \quad \quad \quad \quad \quad \quad \quad \quad \quad \quad -c'\left(\frac{r^{3.5}+\log^{3.5}(m/\delta)}{2^r\sqrt{d}} +   \sqrt{d}\epsilon \right) \left( 1+ \frac{\sqrt{\log(1/\delta)}}{\sqrt{\hat{m}}} \right)  \label{vcvc}\\
       &\implies  \mathbf{a}^\top {\phi}( \mathbf{v})+ \tau > \exp(-O( r^5 \log^6(r) \log(\log(r)/\delta) ))     \label{cdcd}    
    \end{align}
where $\epsilon = O\left(    \frac{d{\log(dTn_2/\delta)} }{\sqrt{Tn_2}} \left( 1+ \frac{\sqrt{\log(T/\delta)}}{\sqrt{n_1}} \right) + {\frac{\sqrt{dr} \log(dm/\delta) }{\sqrt{T}}}   \right)$,
\eqref{vcvc} follows with probability at least $1-\delta$   by Lemma \ref{lem:true} and a union bound, and \eqref{cdcd} follows since $d$, $T$, $Tn_2$, $n_1$ and $\hat{m}$ are sufficiently large.
Note that the input to $\tilde{\phi}(\cdot)$ is effectively $r$-dimensional since the first-layer weights immediately project the input onto an $r$-dimensional subspace.
Repeating the same argument for the case  $g(\mathbf{z}) = -1$ with the same $\mathbf{a},\tau$ completes the proof.
\end{proof}




    

 

\textbf{Step 6: Complexity of learning the linear head.}
Now that we have shown that for any binary function on the $r$-dimensional hypercube, $g$ makes its inverse sets linearly separable with high probability, we complete the proof by bounding the sample complexity of finding a linear separator. For convenience, we restate the theorem here in full detail. 
\begin{theorem}[End-to-end Guarantee] \label{thm:downstream:app}
Consider a downstream task with labeling function $g_{T+1}$.
Construct the two-layer ReLU embedding $\phi:\mathbb{R}^d \rightarrow \mathbb{R}^{\hat{m}}$ using the rescaled $\mathbf{W}^1$ for first layer weights as in  \eqref{qw}, and train the task-adapted head $(\mathbf{a}_{T+1},\tau_{T+1})$ using $N$ i.i.d. samples from the downstream task, as described in Section \ref{sec:formulation},  with regularization parameter $\hat{\lambda}_{\mathbf{a}} = {\exp(-c r^5 \log^6(r) \log(\log(r)/\delta) )}$ for an absolute constant $c$.
 Further, suppose Assumption \ref{assump:diversity} holds, 
${m} = \Omega( r^5\log^8(r) {\log(1/\delta)})$,
    $\hat{m} =\exp\left(\Omega\left( {m}\right) \right)$,
$d = \Omega( \log^{7}({m}/\delta)\exp( c r^5 \log^6(r)\log(\log(r)/\delta) )  )$, \\
$T=\Omega( d^2 r \log^2(d/\delta) \exp( c r^5 \log^6(r)\log(\log(r)/\delta) ) )$, and \\
${Tn_2}{} = \Omega(\log^2(dTn_2/\delta)(1+\frac{\log(T/\delta)}{n_1}) d^3  \exp( c r^5 \log^6(r)\log(\log(r)/\delta) )  )$, and set  $\gamma = \Theta(\sqrt{r}\log({r}))$, $\hat{\gamma} = \Theta( r^{3} \log^4(r))$, $\eta = \Theta(1)$, $\nu_{\mathbf{w}} = O(d^{-5/4}(m\log(T/\delta))^{-1/2})$, and $\lambda_{\mathbf{w}} = 1/\eta + \eta/(2^{r+1}\pi)$.

Then for any $\delta \in (0,0.05]$, with probability at least $1-\delta$ over the random  initializations, draw of $T$ pretraining tasks, draw of $n_1+n_2$ samples per task, and draw of $N$  downstream samples, we have
\begin{align}
        \mathcal{L}^{\text{eval}}_{T+1} 
        &= 
       \frac{ \exp(-O( r^5 \log^6(r) \log(\log(r)/\delta) )) }{ \sqrt{N}}. 
    \end{align}
\end{theorem}

\begin{proof}
By standard Gaussian matrix concentration, we have that with probability at least $1-\delta_1$, $\| \hat{\alpha} {\mathbf{W}}^0  \mathbf{M}^\top\|_2 = O\left(\frac{\sqrt{r} + \sqrt{\log(1/\delta_1)}}{\sqrt{m}} \right)$. Similarly, with probability at least $1-\delta_2$, $ \| \hat{\mathbf{W}}\|_2= O\left(1 +\frac{\sqrt{\log(1/\delta_2)}}{\sqrt{\hat{m}}}\right).$ Let $\delta_3 \coloneqq \delta_1 + \delta_2$.
Thus, 
by Lemma \ref{lem:true} and the triangle inequality, for any $\mathbf{v}\in \mathcal{V}(\mathbf{M})\coloneqq  \{\mathbf{M}^\top \mathbf{z} + \mathbf{M}_{\perp}^\top \boldsymbol{\xi}: \mathbf{z} \in \mathcal{H}^r, \boldsymbol{\xi} \in \mathcal{H}^{d-r}\} $, we have
\begin{align}
    \|\phi(\mathbf{v})\|_2 &\leq 2 \|\tilde{\phi}(\mathbf{v})\|_2 \nonumber\\
    &\leq 2c\left\|\hat{\mathbf{W}}\sigma\left(\frac{2}{\sqrt{m}\nu_{\mathbf{w}} }{\mathbf{W}}^0\mathbf{M}^\top\mathbf{Mv} + \mathbf{b}\right) + \hat{\mathbf{b}} \right\|_2 \nonumber \\
    &\leq 4c\left\|\hat{\mathbf{W}}\right\|_2\left\|\frac{1}{\sqrt{m}\nu_{\mathbf{w}} }{\mathbf{W}}^0\mathbf{M}^\top\mathbf{Mv} + \mathbf{b}  \right\|_2 + 2c \hat{\gamma} \nonumber \\
    &\leq 4c\left( 1 + \frac{\sqrt{\log(1/\delta_3)}}{\sqrt{\hat{m}}} \right)\left(\frac{r + \sqrt{r \log(1/\delta_3)}}{\sqrt{m}}  +   \gamma \right) + 2c \hat{\gamma} \nonumber \\
    &\leq c' \sqrt{r}\log(r) \frac{\sqrt{\log(1/\delta_3)}}{\sqrt{\hat{m}}} + c' \frac{\sqrt{r}\log(1/\delta_3)}{\sqrt{\hat{m}m}}  +  c'r^{2.5}\log^4(r)  \label{jdr}
    \\
    &=: \iota \nonumber
\end{align}
for absolute constants $c$ and $c'$, where \eqref{jdr} follows by choice of $\gamma$ and $\hat{\gamma}$.

Let $ (\mathbf{a}_{T+1}^*,\tau^*_{T+1})$ be the linear head that separates the inverse sets of $g_{T+1}$ with margin $\mu := \exp(-O( r^5 \log^6(r) \log(\log(r)/\delta) ))$, whose existence is guaranteed with high probability by Lemma \ref{lem:sep}. 
We have $|\tau^*_{T+1}|\leq \max_{\mathbf{v}\in \mathcal{H}^d} |(\mathbf{a}^*)^\top {\phi}(\mathbf{v})|\leq \iota $. Thus, $\sqrt{\|\mathbf{a}^*_{T+1}\|_2^2 + (\tau^*_{T+1})^2}\leq \sqrt{2}\iota =: B$. Since $(\mathbf{a}^*_{T+1},\tau^*_{T+1})$ linearly separates the two inverse sets with margin $\mu$, $(\mathbf{a}^*_{T+1}/\mu,\tau^*_{T+1}/\mu)$ linearly separates them with margin 1. Define $\mathcal{J}:= \{(\mathbf{a}, \tau): \mathbf{a}\in \mathbb{R}^{\hat{m}}, \tau\in \mathbb{R},\sqrt{\|\mathbf{a}^*_{T+1}\|_2^2 + (\tau^*_{T+1})^2} \leq B/\mu\}$. 
\begin{align}
 \mathcal{L}^*_{T+1}  \coloneqq &\min_{(\mathbf{a},\tau)\in \mathcal{J}} \frac{1}{2^d}\sum_{\mathbf{v}\in \mathcal{V}(\mathbf{M})} \ell( f_{T+1}(\mathbf{Mv}), \mathbf{a}^\top \phi(\mathbf{v}) + \tau )\nonumber\\
    &\leq \frac{1}{2^d}\sum_{\mathbf{v}\in \mathcal{V}(\mathbf{M})} \ell((\mathbf{a}_{T+1}^*)^\top \phi(\mathbf{v})/\mu + \tau^*_{T+1}/\mu , g_{T+1}(\mathbf{Mv}) )  \nonumber \\
    &= 0 \label{zer}
\end{align}
where $\ell$ is the hinge loss, as usual.
Recall that 
    \begin{align}
    ( \mathbf{a}_{T+1}, \tau_{T+1}) \in\  \argmin_{ \mathbf{a}\in \mathbb{R}^{\hat{m}}, \tau \in \mathbb{R}}  \frac{1}{n}\sum_{l=1}^n \ell(\mathbf{a}^\top \phi(\mathbf{v}_l )+\tau , g_{T+1}(\mathbf{Mv}_{l})  ) + \frac{\hat{\lambda}_{\mathbf{a}}}{2}(\|\mathbf{a}\|_2^2+\tau^2) \label{adaa}
\end{align}
for a set of $N$ random samples $\{\mathbf{v}_l,f_{T+1}(\mathbf{v}_l) \}_{l=1}^N$, where each $\mathbf{v}_l$ is drawn  by independently  sampling $\mathbf{z}_l\sim \text{Unif}(\mathcal{H}^r)$ and $\boldsymbol{\xi}_l\sim \text{Unif}(\mathcal{H}^{d-r})$. Equivalently, for $\hat{\lambda}_{\mathbf{a}} = \mu/B$  we have
    \begin{align}
    ( \mathbf{a}_{T+1}, \tau_{T+1}) \in\  \argmin_{ (\mathbf{a}, \tau)\in \mathcal{J}}  \frac{1}{n}\sum_{l=1}^n \ell(\mathbf{a}^\top g(\mathbf{v}_l )+\tau , g_{T+1}(\mathbf{Mv}_{l})  )  \label{adaaa}
\end{align}

Thus, by applying a standard generalization bound for $1$-Lipschitz loss functions \cite{Livni_2017}, we obtain, for an absolute constant $C$,
\begin{align}
   \mathcal{L}^{\text{eval}}_{T+1} &:=  \frac{1}{2^d}\sum_{\mathbf{v}\in \mathcal{V}(\mathbf{M})} \ell( \mathbf{a}_{T+1}^\top \phi(\mathbf{v}) + \tau_{T+1}, g_{T+1}(\mathbf{Mv}) ) \nonumber \\
   &\leq \mathcal{L}^*_{T+1}  + C\frac{B\sqrt{\log(1/\delta_4)}}{\mu\sqrt{N}} \nonumber \\
    &= C\frac{B\sqrt{\log(1/\delta_4)}}{\mu\sqrt{N}} \nonumber \\
    &= O\left(\frac{\exp(O( r^5 \log^6(r) \log(\log(r)/\delta_5) ))}{\sqrt{N}}\sqrt{\log(1/\delta_4)} \right). \label{fffinal}
\end{align}
with probability at least $1 - \delta_4 - \delta_5$ where $\delta_5 \geq \delta_3$, and where we have used the lower bound on $\mu$ from Lemma \ref{lem:sep} and the choice of $\hat{\lambda}_{\mathbf{a}}$ (which determines the choice of $B$). Setting $\delta = \delta_4 + \delta_5$ for some $\delta \in (0,0.05]$ completes the proof. 
\end{proof}

To conclude, we state and prove a corollary of Theorem \ref{thm:downstream:app}.
\begin{corollary}[Generalization to set of tasks] \label{cor:set:app}
Consider a set  of possible downstream tasks $\mathcal{S}^{\text{eval}}$ with cardinality $|\mathcal{S}^{\text{eval}}| = D$. 
Construct the two-layer ReLU embedding $\phi:\mathbb{R}^d \rightarrow \mathbb{R}^{\hat{m}}$ using the re-scaled $\mathbf{W}^1$ for first layer weights as in  \eqref{qw}, and train the task-adapted head $(\mathbf{a}_{T+1},\tau_{T+1})$ using $N$ i.i.d. samples from the downstream task, as described in Section \ref{sec:formulation}, 
with regularization parameter $\hat{\lambda}_{\mathbf{a}} = {\exp(-c r^5 \log^6(r) \log(D\log(r)/\delta) )}$ for an absolute constant $c$.
 Further, suppose Assumption \ref{assump:diversity} holds, 
${m} = \Omega( r^5\log^8(r) {\log(D/\delta)})$,
    $\hat{m} =\exp\left(\Omega\left( {m}\right) \right)$,
$d = \Omega( \log^{7}({mD/\delta})\exp( c r^5 \log^6(r)\log(D\log(r)/\delta) )  )$,\\
$T=\Omega( d^2 r \log^2(Dd/\delta) \exp( c r^5 \log^6(Dr)\log(D\log(r)/\delta) ) )$, and\\
${Tn_2}{} = \Omega(\log^2(DdTn_2/\delta)(1+\frac{\log(DT/\delta)}{n_1}) d^3  \exp( c r^5 \log^6(Dr)\log(D\log(r)/\delta) )  )$, and set \\
$\gamma = \Theta(\sqrt{r}\log({Dr}))$, $\hat{\gamma} = \Theta( r^{3} \log^4(Dr))$, $\eta = \Theta(1)$, $\nu_{\mathbf{w}} = O(d^{-5/4}(m\log(DT/\delta))^{-1/2})$, and $\lambda_{\mathbf{w}} = 1/\eta + \eta/(2^{r+1}\pi)$.


Then with probability at least $1-\delta$,  any task $T+1$ in $\mathcal{S}^{\text{eval}}$ 
satisfies
\begin{align}
        \mathcal{L}^{\text{eval}}_{T+1} 
        &= 
       \frac{\exp(O( r^5 \log^6(r) \log(D\log(r)/\delta) ))}{\sqrt{N}}, 
    \end{align}
\end{corollary}

\begin{proof}
Note that in the proof of Theorem \ref{thm:downstream:app}, $\delta$ upper bounds the probability of a bad event occurring that depends on the choice of downstream task. 
Thus, setting $\delta^{\text{new}} = \delta/D$ applying a union bound over all tasks implies $$\mathcal{L}^{\text{eval}}_{T+1} 
        = 
       \frac{\exp(O( r^5 \log^6(r) \log(\log(r)/\delta^{\text{new}}) ))}{\sqrt{N}} =\frac{\exp(O( r^5 \log^6(r) \log(D\log(r)/\delta) ))}{\sqrt{N}} $$ for all tasks $T+1\in \mathcal{T}^{\text{eval}}$ with probability at least $1 - \delta^{\text{new}} D = 1 - \delta$, as desired.
\end{proof}

\newpage

\section{Negative Results}

The proof of Theorem \ref{thm:lower} follows directly from Theorem 5 in \cite{barak2022hidden}. We formally re-state and prove 
Theorem \ref{thm:single} below. 


\begin{theorem}\label{thm:single-app}
    Consider any algorithm $\mathcal{A}$ that takes as input infinite samples from any {\em single} task in  $\mathcal{T}_{\text{s.p.}}(\mathbf{M})$ and returns an $\hat{m}$-dimensional representation $\Psi:\mathcal{H}^d \rightarrow \mathbb{R}^{\hat{m}}$. 
Then
 there exists an $\mathbf{M}\in \mathbb{O}^{r\times d}_{\{0,1\}}$ such that for any $k\in[r]$, with probability at least $1 - 2^{-r}\sum_{j=k}^r\binom{r}{j} $ over the draw of a single training task $f_1 \sim \mathcal{T}_{\text{s.p.}}(\mathbf{M})$, the representation $\Psi_{f_1}:= \mathcal{A}(f_1)$ satisfies that
for any $\epsilon >0$, 
$ \hat{m}B^2 > \epsilon^2 \binom{d-k+1}{r-k+1}$ is necessary to obtain
    \begin{align}
   \min_{\mathbf{a}_2: \|\mathbf{a}_2\|_2 \leq B}      \mathbb{E}_{\mathbf{v} \sim \text{Unif}(\mathcal{H}^d)}[\ell( \mathbf{a}_2^\top \Psi_{f_1}(\mathbf{v}), \; f_{2}(\mathbf{v}) )] \geq 1- \epsilon. \nonumber
    \end{align}
\end{theorem}




\begin{proof}
    The proof is an extension of the argument in Section 4 of \cite{malach2022hardness} to the case wherein the representation is not fixed, but depends on a training task that provides partial information about the target (test) task. First, we  establish notations.
    
Recall that any task $f \in \mathcal{T}_{\text{s.p.}}(\mathbf{M})$, satisfies that for all $\mathbf{v} \in \mathcal{H}^d$,  $f(\mathbf{v}) = g(\mathbf{Mv})$ where $g(\mathbf{Mv}) = \prod_{i \in \mathcal{S}} (\mathbf{Mv})_i$ where $\mathcal{S} \subseteq [r]$. Thus,  sampling $f \sim \mathcal{T}_{\text{s.p.}}(\mathbf{M})$ is equivalent to sampling $\mathcal{S} \sim \text{Unif}(\mathcal{P}([r]))$ where $\mathcal{P}([r])$ is the power set on $[r]$ and $\text{Unif}(\mathcal{P}([r]))$ is the uniform distribution over $\mathcal{P}([r])$. 

We condition on $\mathcal{S}$ being a strict subset of $[r]$. In particular, for any
$k \in \{1,\dots,r\}$, define the set $S_k \coloneqq \{\mathcal{S}\subset r: |\mathcal{S}|< k \}$. Note that 
 $|S_k|= \sum_{j=1}^{k-1}\binom{r}{j}$. Thus, for $\mathcal{S} \sim \text{Unif}(\mathcal{P}([r]))$, $\mathbb{P}_{\mathcal{S}} (\mathcal{S}\in S_k) = 2^{-r}\sum_{j=1}^{k-1}\binom{r}{j} = 1 - 2^{-r} \sum_{j=k}^{r}\binom{r}{j}$. 
 In the following, we assume $\mathcal{S} \in S_k$ unless stated otherwise.

Next, recall that the algorithm $\mathcal{A}$ maps infinite training samples from a single training  task $f$ to a representation $\Psi: \mathcal{H}^d \rightarrow [-1,1]^{\hat{m}}$. Thus, the choice of training task (equivalently, the choice of $\mathbf{M}$ and $\mathcal{S}$) determines the resulting representation. Let $f_{\mathbf{M},\mathcal{S}}$ denote the task in $\mathcal{T}_{\text{s.p.}}(\mathbf{M})$ with support $\mathcal{S}$, and $\Psi_{\mathbf{M},\mathcal{S}} := \mathcal{A}(f_{\mathbf{M},\mathcal{S}})$ denote the resulting representation.


    The test task is taken to be the parity task on all $r$ bits specified by $\mathbf{M}$, i.e. $f_{\mathbf{M}, [r]}$. For ease of notation, we write this task as $\bar{f}_{\mathbf{M}}$, and for for  any  representation $ \Psi: \mathcal{H}^d\rightarrow [-1,1]^{\hat{m}} $, define
    \begin{align}
        L_{\mathbf{M}}(\mathbf{a},\Psi)\coloneqq \mathbb{E}_{\mathbf{v}\sim \mathcal{U}^d} [\ell( \mathbf{a}^\top \Psi(\mathbf{v}), \bar{f}_{\mathbf{M}}(\mathbf{v})  )]
    \end{align}
    where $\ell : \mathbb{R}\rightarrow\mathbb{R}_{\geq 0}$ is the hinge loss and $\mathcal{U}^d$ is the uniform distribution over the $d$-dimension Rademacher hypercube $\mathcal{H}^d:=\{-1,1\}^d$. We need to lower bound this loss for all $\mathbf{a}\in \mathbb{R}^{\hat{m}}:\|\mathbf{a}\|_2\leq B$ some representation resulting from single-task training.
    
 To do so, we now follow a similar argument as in Section 4.1 in \cite{malach2022hardness}. Consider any $\mathbf{M} \in \mathbb{O}^{r \times d}_{\{0,1\}}$,  $\mathcal{S} \in S_k$, and resulting representation $\Psi_{\mathbf{M}, \mathcal{S}}: \mathcal{H}^d\rightarrow [-1,1]^{\hat{m}}$.  Since the hinge loss is convex, for any $\mathbf{a}\in \mathbb{R}^{\hat{m}}$ such that $\|\mathbf{a}\|_2 \leq B$, we have:
    \begin{align}
        L_{\mathbf{M}}(\mathbf{a},\Psi_{\mathbf{M}, \mathcal{S}}) &\geq L_{\mathbf{M}}(\mathbf{0},\Psi_{\mathbf{M}, \mathcal{S}}) + \langle \nabla L_{\mathbf{M}}(\mathbf{0},\Psi_{\mathbf{M}, \mathcal{S}}) , \mathbf{a} \rangle  \nonumber\\
        &\geq 1 -  B \|\nabla L_{\mathbf{M}}(\mathbf{0},\Psi_{\mathbf{M}, \mathcal{S}})\|_2 \label{wpkl}
    \end{align}
    where \eqref{wpkl} follows by the Cauchy-Schwarz inequality and the fact that
$L_{\mathbf{M}}(\mathbf{0},\Psi) =1$ for all $\Psi$. 
Next, we motivate considering random $\mathbf{M}$.  We have 
\begin{align}
  \max_{\mathbf{M} \in \mathbb{O}^{d\times r}_{\{0,1\}} }  \min_{\mathbf{a}:\|\mathbf{a}\|_2 \leq B} L_{\mathbf{M}}(\mathbf{a},\Psi_{\mathbf{M}, \mathcal{S}})  
    &\geq \mathbb{E}_{\mathbf{M} \sim \mathcal{M}^d_r} \left[ \min_{\mathbf{a}:\|\mathbf{a}\|_2 \leq B} L_{\mathbf{M}}(\mathbf{a},\Psi_{\mathbf{M}, \mathcal{S}}) \right]    \label{95979} 
\end{align}
 where $\mathcal{M}^{d}_r$ denotes the uniform distribution over all $\binom{d}{r}$ possible choices of $\mathbf{M} \in \mathbb{O}^{d\times r}_{\{0,1\}}$. It remains to lower bound the RHS of \eqref{95979}. For ease of notation, we write $\mathbb{E}_{\mathbf{M}} := \mathbb{E}_{\mathbf{M} \sim \mathcal{M}^d_r}$.
Using \eqref{wpkl}, we obtain 
\begin{align}
    \mathbb{E}_{\mathbf{M}} \left[ \min_{\mathbf{a}:\|\mathbf{a}\|_2 \leq B} L_{\mathbf{M}}(\mathbf{a},\Psi_{\mathbf{M}, \mathcal{S}})\right] &\geq  1 - B \; \mathbb{E}_{\mathbf{M}}[\|\nabla_{\mathbf{a}} L_{\mathbf{M}}(\mathbf{0},\Psi_{\mathbf{M}, \mathcal{S}})\|_2]. \label{16676}
\end{align}
The crux of the proof is to upper bound $ \mathbb{E}_{\mathbf{M}}[\|\nabla_{\mathbf{a}} L_{\mathbf{M}}(\mathbf{0},\Psi_{\mathbf{M}, \mathcal{S}})\|_2]$. 
Note that 
\begin{align}
    \mathbb{E}_{\mathbf{M}}\left[\|\nabla_{\mathbf{a}} L_{\mathbf{M}}(\mathbf{0},\Psi_{\mathbf{M}, \mathcal{S}})\|_2^2\right] &= \mathbb{E}_{\mathbf{M}, \mathcal{S}} \left[\sum_{j=1}^{\hat{m}} ( \mathbb{E}_{\mathbf{v}\sim \mathcal{U}^d}  [\bar{f}_{\mathbf{M}}(\mathbf{v})\Psi_{\mathbf{M}, \mathcal{S}}(\mathbf{v})_j] )^2 \right] \nonumber \\
    &= \sum_{j=1}^{\hat{m}} \mathbb{E}_{\mathbf{M}} \mathbb{E}_{\mathbf{v}\sim \mathcal{U}^d} \mathbb{E}_{\mathbf{v}'\sim \mathcal{U}^d}\left[  \bar{f}_{\mathbf{M}}(\mathbf{v})\bar{f}_{\mathbf{M}}(\mathbf{v}')\Psi_{\mathbf{M}, \mathcal{S}}(\mathbf{v})_j\Psi_{\mathbf{M}, \mathcal{S}}(\mathbf{v}')_j  \right] \nonumber \\
    &= \sum_{j=1}^{\hat{m}} \mathbb{E}_{\mathbf{M}} \mathbb{E}_{\mathbf{v}, \mathbf{v}'} \left[ \left(\prod_{i=1}^r (\mathbf{Mv})_i (\mathbf{Mv}')_i  \right) \Psi_{\mathbf{M}, \mathcal{S}}(\mathbf{v})_j\Psi_{\mathbf{M}, \mathcal{S}}(\mathbf{v}')_j  \right] \label{1776} 
\end{align}
where $(\mathbf{Mv})_i$ and $\Psi(\mathbf{v})_j$ are the $i$-th and $j$-th elements of $\mathbf{Mv}$ and $\Psi(\mathbf{v})$, respectively.
Next, let $\mathbf{M}_{\mathcal{S}}$ denote the rows of $\mathbf{M}$ picked out by $\mathcal{S}$, and $\mathbf{M}_{\setminus \mathcal{S}}$ denote the remaining rows. 
Further, let
$\mathbf{v}_{\mathbf{M}, \mathcal{S}} = \mathbf{M}_{\mathcal{S}}\mathbf{v} \in \{-1,1\}^{|\mathcal{S}|}$ denote the bits in $\mathbf{v}$ specified by $\mathbf{M}_{\mathcal{S}}$,  let $\mathbf{v}_{\mathbf{M},\setminus \mathcal{S}}= \mathbf{M}_{\setminus \mathcal{S}}\mathbf{v} \in \{-1,1\}^{r - |\mathcal{S}|}$ denote  the bits specified by $\mathbf{M}_{\setminus \mathcal{S}}$. Also let $\mathbf{v}_{\setminus(\mathbf{M}, \mathcal{S})}  \in \{-1,1\}^{d-|\mathcal{S}|}$ denote the bits in $\mathbf{v}$ not specified by $\mathbf{M}_{\mathcal{S}}$. We have, for any $j \in [\hat{m}]$,
\begin{align}
    & \mathbb{E}_{\mathbf{M}} \mathbb{E}_{\mathbf{v}, \mathbf{v}'} \left[ \left(\prod_{i=1}^r (\mathbf{Mv})_i (\mathbf{Mv}')_i  \right) \Psi_{\mathbf{M}, \mathcal{S}}(\mathbf{v})_j\Psi_{\mathbf{M}, \mathcal{S}}(\mathbf{v}')_j  \right] \nonumber \\
    &= \mathbb{E}_{\mathbf{M}_{\mathcal{S}},\mathbf{v}_{\mathbf{M},\mathcal{S}}, \mathbf{v}'_{\mathbf{M},\mathcal{S}}} \Bigg[ \Bigg(\prod_{i=1}^{|\mathcal{S}|} ( \mathbf{v}_{\mathbf{M}, \mathcal{S}}  )_i (\mathbf{v}_{\mathbf{M}, \mathcal{S}}')_i  \Bigg) \nonumber \\
    &\quad \quad \quad \quad \quad \quad \quad \quad \quad\;\times \mathbb{E}_{\mathbf{M}_{\setminus\mathcal{S}}, \mathbf{v}_{\setminus(\mathbf{M}, \mathcal{S})}, \mathbf{v}'_{\setminus(\mathbf{M}, \mathcal{S})}} \Bigg[ \Bigg(\prod_{i=1}^{r-|\mathcal{S}|} (\mathbf{v}_{\mathbf{M},\setminus\mathcal{S}})_i (\mathbf{v}_{\mathbf{M},\setminus \mathcal{S}}')_i  \Bigg) \Psi_{\mathbf{M}, \mathcal{S}}(\mathbf{v})_j\Psi_{\mathbf{M}, \mathcal{S}}(\mathbf{v}')_j  \Bigg]\Bigg] \label{8890}
\end{align}
Note that for fixed $\mathcal{S}$, $\mathbf{M}_{\mathcal{S}}$, and $\mathbf{v}_{\mathbf{M},\mathcal{S}}$, $\Psi_{\mathbf{M}, \mathcal{S}}(\mathbf{v})_j$ is a function of  $\mathbf{v}_{\setminus(\mathbf{M},\mathcal{S})}$, namely $\Psi_{\mathbf{M}, \mathcal{S}, \mathbf{v}_{\mathbf{M},\mathcal{S}}}(\mathbf{v}_{\setminus(\mathbf{M},\mathcal{S})})_j$. For ease of notation, denote this function as $\psi(\mathbf{v}_{\setminus(\mathbf{M},\mathcal{S})})$. We have:
\begin{align}
&\mathbb{E}_{\mathbf{M}_{\setminus\mathcal{S}}, \mathbf{v}_{\setminus(\mathbf{M}, \mathcal{S})}, \mathbf{v}'_{\setminus(\mathbf{M}, \mathcal{S})}} \Bigg[ \Bigg(\prod_{i=1}^{r-|\mathcal{S}|} (\mathbf{v}_{\mathbf{M},\setminus\mathcal{S}})_i (\mathbf{v}_{\mathbf{M},\setminus \mathcal{S}}')_i  \Bigg) \Psi_{\mathbf{M}, \mathcal{S}}(\mathbf{v})_j\Psi_{\mathbf{M}, \mathcal{S}}(\mathbf{v}')_j  \Bigg] \nonumber \\
    &=\mathbb{E}_{\mathbf{M}_{\setminus\mathcal{S}}, \mathbf{v}_{\setminus(\mathbf{M}, \mathcal{S})}, \mathbf{v}'_{\setminus(\mathbf{M}, \mathcal{S})}} \Bigg[ \Bigg(\prod_{i=1}^{r-|\mathcal{S}|} (\mathbf{v}_{\mathbf{M},\setminus\mathcal{S}})_i (\mathbf{v}_{\mathbf{M},\setminus \mathcal{S}}')_i  \Bigg) \psi(\mathbf{v}_{\setminus(\mathbf{M},\mathcal{S})})\psi(\mathbf{v}'_{\setminus(\mathbf{M},\mathcal{S})})  \Bigg] \nonumber  \\
    &=\mathbb{E}_{\mathbf{M}_{\setminus\mathcal{S}}, \mathbf{v}_{\setminus(\mathbf{M}, \mathcal{S})}, \mathbf{v}'_{\setminus(\mathbf{M}, \mathcal{S})}} \Bigg[ h_{\mathbf{M},\setminus \mathcal{S}}(\mathbf{v}_{\setminus(\mathbf{M},\mathcal{S})})h_{\mathbf{M},\setminus \mathcal{S}}(\mathbf{v}'_{\setminus(\mathbf{M},\mathcal{S})})\psi(\mathbf{v}_{\setminus(\mathbf{M},\mathcal{S})})\psi(\mathbf{v}'_{\setminus(\mathbf{M},\mathcal{S})})  \Bigg] \nonumber  \\
    &= \mathbb{E}_{\mathbf{B} \sim \mathcal{M}^{d-|\mathcal{S}|}_{r - |\mathcal{S}|}, \mathbf{u}\sim \mathcal{U}^{d - |\mathcal{S}|}, \mathbf{u}'\sim \mathcal{U}^{d - |\mathcal{S}|}} \bigg[ h_{\mathbf{B}}(\mathbf{u})h_{\mathbf{B}}(\mathbf{u}')\psi(\mathbf{u})\psi(\mathbf{u}') \bigg] \nonumber \\
    &= \mathbb{E}_{\mathbf{B} \sim \mathcal{M}^{d-|\mathcal{S}|}_{r - |\mathcal{S}|}} \bigg[ \langle h_{\mathbf{B}},\psi \rangle_{\mathcal{U}^{d- |\mathcal{S}|}}^2 \bigg]  \label{889}
\end{align}
where $h_{\mathbf{M},\setminus \mathcal{S}}$ is the sparse parity task on input bits specified  by $\mathbf{M}_{\setminus \mathcal{S}}$, $h_{\mathbf{B}}$ is the sparse parity task on input bits specified by $\mathbf{B}\in \mathbb{O}^{(d-|\mathcal{S}|)\times (r - |\mathcal{S}|)}_{\{0,1\}}$, and 
\begin{align}
    \langle h_{\mathbf{B}},\psi \rangle_{\mathcal{U}^{d- |\mathcal{S}|}} := \mathbb{E}_{\mathbf{u}\sim \mathcal{U}^{d- |\mathcal{S}|}} [ h_{\mathbf{B}}(\mathbf{u}) \psi( \mathbf{u}) ]
\end{align}

Note that \eqref{889} is exactly the variance of the task distribution $\mathcal{M}^{d-|\mathcal{S}|}_{r - |\mathcal{S}|}$ with respect to the function $\psi$. Since $\mathcal{M}^{d-|\mathcal{S}|}_{r - |\mathcal{S}|}$ is a uniform distribution over $\binom{d-|\mathcal{S}|}{r - |\mathcal{S}|}$ orthonormal tasks, and $\sup_{\mathbf{u}} |\psi(\mathbf{u})|\leq 1$, we have by Parseval's identity:
\begin{align}
  \mathbb{E}_{\mathbf{B} \sim \mathcal{M}^{d-|\mathcal{S}|}_{r - |\mathcal{S}|}} \bigg[ \langle h_{\mathbf{B}},\psi \rangle_{\mathcal{U}^{d- |\mathcal{S}|}}^2 \bigg] = \frac{1}{\binom{d-|\mathcal{S}|}{r - |\mathcal{S}|}} \sum_{\mathbf{B}\in \mathbb{O}^{(d-|\mathcal{S}|)\times (r - |\mathcal{S}|)}_{\{0,1\}} } \langle h_{\mathbf{B}},\psi \rangle_{\mathcal{U}^{d- |\mathcal{S}|}}   \leq \frac{\sup_{\mathbf{u}} |\psi(\mathbf{u})|}{\binom{d-|\mathcal{S}|}{r - |\mathcal{S}|}} \leq \frac{1}{\binom{d-|\mathcal{S}|}{r - |\mathcal{S}|}} \label{parseval}
\end{align}
Please see Section 4.1 of \cite{malach2022hardness} for more details. Now combining \eqref{parseval} with \eqref{889} and \eqref{8890}, we obtain via Cauchy-Schwarz:
\begin{align}
    &\mathbb{E}_{\mathbf{M}} \mathbb{E}_{\mathbf{v}, \mathbf{v}'} \left[ \left(\prod_{i=1}^r (\mathbf{Mv})_i (\mathbf{Mv}')_i  \right) \Psi_{\mathbf{M}, \mathcal{S}}(\mathbf{v})_j\Psi_{\mathbf{M}, \mathcal{S}}(\mathbf{v}')_j  \right]^2  \nonumber \\
    &\leq \mathbb{E}_{\mathbf{M}_{\mathcal{S}},\mathbf{v}_{\mathbf{M},\mathcal{S}}, \mathbf{v}'_{\mathbf{M},\mathcal{S}}} \Bigg[ \Bigg(\prod_{i=1}^{|\mathcal{S}|} ( \mathbf{v}_{\mathbf{M}, \mathcal{S}}  )_i (\mathbf{v}_{\mathbf{M}, \mathcal{S}}')_i  \Bigg)^2\Bigg] \nonumber \\
    &\quad \quad \;\times \mathbb{E}_{\mathbf{M}_{\mathcal{S}},\mathbf{v}_{\mathbf{M},\mathcal{S}}, \mathbf{v}'_{\mathbf{M},\mathcal{S}}} \Bigg[ \Bigg(\mathbb{E}_{\mathbf{M}_{\setminus\mathcal{S}}, \mathbf{v}_{\setminus(\mathbf{M}, \mathcal{S})}, \mathbf{v}'_{\setminus(\mathbf{M}, \mathcal{S})}} \Bigg[ \Bigg(\prod_{i=1}^{r-|\mathcal{S}|} (\mathbf{v}_{\mathbf{M},\setminus\mathcal{S}})_i (\mathbf{v}_{\mathbf{M},\setminus \mathcal{S}}')_i  \Bigg) \Psi_{\mathbf{M}, \mathcal{S}}(\mathbf{v})_j\Psi_{\mathbf{M}, \mathcal{S}}(\mathbf{v}')_j  \Bigg]\Bigg)^2\Bigg]  \nonumber \\
    &\leq \mathbb{E}_{\mathbf{M}_{\mathcal{S}},\mathbf{v}_{\mathbf{M},\mathcal{S}}, \mathbf{v}'_{\mathbf{M},\mathcal{S}}} \Bigg[ \Bigg(\prod_{i=1}^{|\mathcal{S}|} ( \mathbf{v}_{\mathbf{M}, \mathcal{S}}  )_i (\mathbf{v}_{\mathbf{M}, \mathcal{S}}')_i  \Bigg)^2\Bigg] \times \frac{1}{\binom{d-|\mathcal{S}|}{r - |\mathcal{S}|}^2}  \nonumber \\
    &\leq \frac{1}{\binom{d-|\mathcal{S}|}{r - |\mathcal{S}|}^2} \nonumber
\end{align}
Therefore, returning to  \eqref{1776}, we obtain
\begin{align}
    \mathbb{E}_{\mathbf{M}}\left[\|\nabla_{\mathbf{a}} L_{\mathbf{M}}(\mathbf{0},\Psi_{\mathbf{M}, \mathcal{S}})\|_2^2\right] &\leq \frac{\hat{m}}{\binom{d-|\mathcal{S}|}{r - |\mathcal{S}|}^2}
\end{align}
thus 
\begin{align}
    \mathbb{E}_{\mathbf{M}}\left[\|\nabla_{\mathbf{a}} L_{\mathbf{M}}(\mathbf{0},\Psi_{\mathbf{M}, \mathcal{S}})\|_2\right] &\leq \frac{\sqrt{\hat{m}}}{\binom{d-|\mathcal{S}|}{r - |\mathcal{S}|}}\leq \frac{\sqrt{\hat{m}}}{\binom{d-k+1}{r - k+1}}
\end{align}
where the last inequality follows since $\mathcal{S}\in S_k$. Combining this with \eqref{95979} and \eqref{16676} yields that for any $\mathcal{S}\in S_k$,
\begin{align}
     \max_{\mathbf{M} \in \mathbb{O}^{d\times r}_{\{0,1\}} }  \min_{\mathbf{a}:\|\mathbf{a}\|_2 \leq B} L_{\mathbf{M}}(\mathbf{a},\Psi_{\mathbf{M}, \mathcal{S}})  \geq 1 -  \frac{\sqrt{\hat{m}}B}{\binom{d-k+1}{r - k+1}},
\end{align}
completing the proof.
\end{proof}

\newpage

\section{Distributions That Satisfy Assumption \ref{assump:diversity}}\label{app:sec:dists}

\begin{lemma}
    The task link function distribution $\mathcal{T}_{\text{all}}$
    satisfies Assumption \ref{assump:diversity}.
\end{lemma}

\begin{proof}
The set of all functions $\mathcal{T}_{\text{all}}:= \{g: \mathcal{H}^r\rightarrow \{-1,1\}\}$ has a bijection with the power set on 
$\mathcal{H}^r$,
denoted by $\mathcal{P}_{2^r}$, where each element of $\mathcal{P}_{2^r}$ is paired with the positive inverse set (i.e. $\{\mathbf{z}\in \mathcal{H}^r:g(\mathbf{z})=1\}$) for some function $g \in \mathcal{T}_{\text{all}}$. So, sampling uniformly from $\mathcal{T}_{\text{all}}$ is equivalent to sampling uniformly from $\mathcal{P}_{2^r}$, which is equivalent to the following procedure: for all  $\mathbf{z}\in \mathcal{H}^r$, independently assign $\mathbf{z}$ to Bin 1 (the `keep' set) with probability 0.5 and Bin 2 (the `discard' set) with probability 0.5. If $\mathbf{z}\neq \mathbf{z}',$ it is equally likely that $\mathbf{z}$ and $\mathbf{z}'$ are in the same bin as they are in different bins, completing the proof.
\end{proof}







\begin{lemma}\label{lem:sp}
    The task link function distribution $\mathcal{T}_{\text{s.p.}}$
    satisfies Assumption \ref{assump:diversity}.
\end{lemma}

\begin{proof}
   First note that there are ${r}\choose{0}$ subsets of $[r]$ of size 0, ${r}\choose{1}$ subsets of $[r]$ of size 1, and so on, thus there are ${r\choose 0} + {r\choose 1} + \dots +  {r \choose r} = 2^r$ sparse parity tasks in total. Let $d_H(\mathbf{v}, \mathbf{v}') = \sum_{i=1}^r \chi\{\mathbf{v}_i \neq \mathbf{v}_i'\}$ be the Hamming distance between two Boolean vectors of length $r$.  If $d_H(\mathbf{v},\mathbf{v}') = 0$, then clearly  ${g}_{i}(\mathbf{v}) = {g}_{i}(\mathbf{v}')$ for all ${g}_{i} \in \mathcal{T}$.

    On the other hand, if $d_H(\mathbf{v},\mathbf{v}') = \gamma$ for any $\gamma \in \{1,\dots, r\}$, then $\mathbf{v}$ and $\mathbf{v}'$ share the same values for $r - \gamma$ coordinates, so must share the same label on all sparse parity tasks on subsets of these coordinates, of which there are $2^{r - \gamma} = 2^{r - \gamma} {\gamma \choose 0}$. Next, there are $2^{r - \gamma}\times {\gamma \choose 1}$ sparse parity tasks on one coordinate on which $\mathbf{v}$ and $\mathbf{v}'$ differ and other coordinates on which $\mathbf{v}$ and $\mathbf{v}'$ agree. Since these tasks are sparse parities on a set of coordinates on which $\mathbf{v}$ and $\mathbf{v}'$ differ on an odd number of coordinates, $g_i(\mathbf{v}) \neq g_i(\mathbf{v}')$ for each of these $2^{r - \gamma}\times {\gamma \choose 1}$ tasks. Similarly, there are $2^{r - \gamma}\times {\gamma \choose 2}$ tasks on two coordinates on which $\mathbf{v}$ and $\mathbf{v}'$ differ, and since two is even, $g_i(\mathbf{v}) = g_i(\mathbf{v}')$ for all such tasks. 
    Extrapolating this argument, if $\gamma$ is even, then there are $2^{r - \gamma}({\gamma\choose 0} + {\gamma \choose 2} + \dots +  {\gamma \choose \gamma}) = 2^{r - \gamma} 2^{\gamma-1} = 2^{r-1}$ tasks for which $g_i(\mathbf{v}) = g_i(\mathbf{v}')$, and 
$2^{r - \gamma}({\gamma\choose 1} + {\gamma \choose 3} + \dots +  {\gamma \choose \gamma-1}) = 2^{r - \gamma} 2^{\gamma-1} = 2^{r-1}$ tasks for which $g_i(\mathbf{v}) \neq  g_i(\mathbf{v}')$. Likewise,  if $\gamma$ is odd, there are $2^{r - \gamma}({\gamma\choose 0} + {\gamma \choose 2} + \dots +  {\gamma \choose \gamma - 1})  = 2^{r-1}$ tasks for which $g_i(\mathbf{v}) = g_i(\mathbf{v}')$, and 
$2^{r - \gamma}({\gamma\choose 1} + {\gamma \choose 3} + \dots +  {\gamma \choose \gamma-1}) = 2^{r-1}$ tasks for which $g_i(\mathbf{v}) \neq  g_i(\mathbf{v}')$. 
\end{proof}


\newpage
\section{Informal Extension to Regression}\label{app:regression}

In this section, we show informally that our insights also apply to multi-task regression.
In the regression setting, the global loss is given by (consider infinite samples per task, and ignore the bias parameters for simplicity):

$$\mathcal{L}_{\text{reg}}(\mathbf{W}, \mathbf{a}_1, \dots, \mathbf{a}_T) := \frac{1}{2T}\sum_{i=1}^T \mathbb{E}_{\mathbf{x}} [(\mathbf{a}_i^\top \sigma(\mathbf{Wx}) - f_i(\mathbf{x}))^2] + \frac{\lambda_{\mathbf{W}}}{2}||\mathbf{W}||_F^2$$

For any fixed $\mathbf{W}$, the optimal $\mathbf{a}_i$ is $\mathbf{a}_i^*(\mathbf{W}) = \mathbb{E}_{\mathbf{x}} [\sigma(\mathbf{Wx} )\sigma(\mathbf{Wx} )^\top]^{-1} \mathbb{E}_{\mathbf{x}} [ f_i(\mathbf{x}) \sigma(\mathbf{Wx} )]$, which is an average of the current features weighted by labels (as in the classification case we study), that is now additionally multiplied by the normalizing matrix $ \mathbb{E}_{\mathbf{x}} [\sigma(\mathbf{Wx})\sigma(\mathbf{Wx} )^\top]^{-1}$. Note that this optimal $\mathbf{a}_i^*(\mathbf{W}) $ can be attained by one step of gradient descent.

Let $\boldsymbol{\Sigma}_{\mathbf{W}} := \mathbb{E}_{\mathbf{x}} [\sigma(\mathbf{Wx})\sigma(\mathbf{Wx} )^\top]$. Substituting the optimal $\mathbf{a}_i^*(\mathbf{W})$'s from above into the loss, we have \begin{eqnarray} \tilde{\mathcal{L}}_{\text{reg}}(\mathbf{W}) &:=& \frac{1}{2T}\sum_{i=1}^T \mathbb{E}_{\mathbf{x}} \bigg[\big( \mathbb{E}_{\mathbf{x}} [ f_i(\mathbf{x}) \sigma(\mathbf{Wx} )]^\top \boldsymbol{\Sigma}_{\mathbf{W}}^{-1} \sigma(\mathbf{Wx}) - f_i(\mathbf{x})\big)^2\bigg] + \frac{\lambda_{\mathbf{W}}}{2}||\mathbf{W}||_F^2 \nonumber \\ &=& \frac{1}{2T}\sum_{i=1}^T \mathbb{E}_{\mathbf{x}} [ f_i(\mathbf{x}) \sigma(\mathbf{Wx} )]^\top \boldsymbol{\Sigma}_{\mathbf{W}}^{-1} \mathbb{E}_{\mathbf{x}} [\sigma(\mathbf{Wx})\sigma(\mathbf{Wx})^\top] \boldsymbol{\Sigma}_{\mathbf{W}}^{-1} \mathbb{E}_{\mathbf{x}} [ f_i(\mathbf{x}) \sigma(\mathbf{Wx} )] \nonumber \\ && \quad \quad \quad \quad - 2 \mathbb{E}_{\mathbf{x}} [ f_i(\mathbf{x}) \sigma(\mathbf{Wx} )]^\top \boldsymbol{\Sigma}_{\mathbf{W}}^{-1}\mathbb{E}_{\mathbf{x}} [ f_i(\mathbf{x}) \sigma(\mathbf{Wx} )] + \mathbb{E}_{\mathbf{x}}[f_i^2(x)] + \frac{\lambda_{\mathbf{W}}}{2}||\mathbf{W}||_F^2 \nonumber \\ &=& \frac{1}{2T}\sum_{i=1}^T - \mathbb{E}_{\mathbf{x}} [ f_i(\mathbf{x}) \sigma(\mathbf{Wx} )]^\top \boldsymbol{\Sigma}_{\mathbf{W}}^{-1}\mathbb{E}_{\mathbf{x}} [ f_i(\mathbf{x}) \sigma(\mathbf{Wx} )] + \mathbb{E}_{\mathbf{x}}[f_i^2(\mathbf{x})] + \frac{\lambda_{\mathbf{W}}}{2}||\mathbf{W}||_F^2 \nonumber \\ &=& - \frac{1}{2}\mathbb{E}_{\mathbf{x}, \mathbf{x}'} \left[ \beta(\mathbf{x}, \mathbf{x}') \sigma(\mathbf{Wx} )^\top \boldsymbol{\Sigma}_{\mathbf{W}}^{-1} \sigma(\mathbf{Wx}' )\right] + \frac{\lambda_{\mathbf{W}}}{2}||\mathbf{W}||_F^2 + c \nonumber \end{eqnarray} where, as in the classification case, $\beta(\mathbf{x}, \mathbf{x}'):= \frac{1}{T}\sum_{i=1}^T f_i(\mathbf{x})f_i(\mathbf{x}') $, and here, $c:=\frac{1}{2T}\sum_{i=1}^T\mathbb{E}_{\mathbf{x}}[f_i^2(\mathbf{x})] $ is a constant independent of $W$. This loss is very similar in form to the pseudo-contrastive loss we derived in \eqref{tjjt} for the classification case: we again have the negative average of $\beta(\mathbf{x}, \mathbf{x}')$ times a proxy for the similarities between the representations of $\mathbf{x}$ and $\mathbf{x}'$. For $\tilde{\mathcal{L}}_{\text{reg}}$ to encourage learning the ground-truth features, $\beta(\mathbf{x}, \mathbf{x}')$ must be a proxy for the similarity of the ground-truth features of $\mathbf{x}$ and $\mathbf{x}'$.

This is a reasonable condition for the following reason: suppose the tasks are normalized such that $\mathbb{E}_{f\sim \mathcal{T}}[f(\mathbf{x})] = 0$ and $\mathbb{E}_{f\sim \mathcal{T}}[f^2(\mathbf{x})] = \nu^2$ for all $x$. Then in the limit $T \rightarrow \infty$, $\beta(\mathbf{x}, \mathbf{x}') = \nu^2\mathbb{E}_{f\sim \mathcal{T}}[f(\mathbf{x})f(\mathbf{x}')]$ is proportional to the correlation between the labels of $x$ and $x'$, which we expect to be a proxy for the similarity between the ground-truth features of $x$ and $x'$. Intuitively, inputs with similar ground-truth features should have more correlated labels (across tasks) than inputs with dissimilar ground-truth features.

Consider for example $f(\mathbf{x}) = \sin( \mathbf{h}^\top \frac{\mathbf{M} \mathbf{x}}{||\mathbf{Mx}||_2} )$ (the following argument would also hold analogously for $f(\mathbf{x}) = \sin( \mathbf{h}^\top \text{sign}({\mathbf{M} \mathbf{x}}) )$, where $f\sim \mathcal{T}$ is induced by drawing $\mathbf{h} \sim \mathcal{N}(\mathbf{0}_r, \mathbf{I}_r)$. Then for all $\mathbf{x}$, $\mathbb{E}_{f\sim \mathcal{T}}[f(\mathbf{x})] = 0$ and $\mathbb{E}_{f\sim \mathcal{T}}[f^2(\mathbf{x})] = \nu^2$ for some $\nu>0$, and, with $T=\infty$, \begin{eqnarray} \beta(\mathbf{x}, \mathbf{x}') &=& \nu^2\mathbb{E}_{f\sim \mathcal{T}}\left[\sin\left( \mathbf{h}^\top \frac{\mathbf{M} \mathbf{x}}{||\mathbf{Mx}||_2} \right)\sin\left( \mathbf{h}^\top \frac{\mathbf{M} \mathbf{x}'}{||\mathbf{Mx}'||_2} \right)\right] \nonumber \\ &=& \frac{\nu^2}{2}\mathbb{E}_{\mathbf{h}}\left[\cos\left( \mathbf{h}^\top \left(\frac{\mathbf{M} \mathbf{x}}{||\mathbf{Mx}||_2} - \frac{\mathbf{M} \mathbf{x}'}{||\mathbf{Mx}'||_2} \right) \right) - \cos\left( \mathbf{h}^\top \left(\frac{\mathbf{M} \mathbf{x}}{||\mathbf{Mx}||_2} + \frac{\mathbf{M} \mathbf{x}'}{||\mathbf{Mx}'||_2} \right) \right)\right] \nonumber \\ &=& \frac{\nu^2}{2}\mathbb{E}_{z \sim \mathcal{N}(0,1)}\left[\cos\left( \left\|\frac{\mathbf{M} \mathbf{x}}{||\mathbf{Mx}||_2} - \frac{\mathbf{M} \mathbf{x}'}{||\mathbf{Mx}'||_2} \right\|_2 z \right) - \cos\left( \left\|\frac{\mathbf{M} \mathbf{x}}{||\mathbf{Mx}||_2} + \frac{\mathbf{M} \mathbf{x}'}{||\mathbf{Mx}'||_2} \right\|_2 z \right)\right] \nonumber \\ &\stackrel{a}{=}& \frac{\nu^2}{4}\left(\exp\left( -\left\|\frac{\mathbf{M} \mathbf{x}}{||\mathbf{Mx}||_2} - \frac{\mathbf{M} \mathbf{x}'}{||\mathbf{Mx}'||_2} \right\|_2^2/2 \right) - \exp\left( - \left\|\frac{\mathbf{M} \mathbf{x}}{||\mathbf{Mx}||_2} + \frac{\mathbf{M} \mathbf{x}'}{||\mathbf{Mx}'||_2} \right\|_2^2/2 \right)\right) \nonumber \\ &\stackrel{b}{=}& \frac{\nu^2 e}{4}\left(\exp\left( \text{cossim}(\mathbf{M} \mathbf{x},\mathbf{Mx}') \right) - \exp\left( -\text{cossim}(\mathbf{M} \mathbf{x},\mathbf{Mx}')\right)\right) \nonumber \end{eqnarray} where $a$ follows by a Gaussian integral calculation.

Observe that the expression in the RHS of $b$ is monotonically increasing in the cosine similarity of the ground-truth features of $\mathbf{x}$ and $\mathbf{x}'$, as desired. Therefore, $\tilde{\mathcal{L}}_{\text{reg}}$ encourages aligning the normalized representations of pairs of inputs (i.e., making $\sigma(\mathbf{Wx} )^\top \boldsymbol{\Sigma}_{\mathbf{W}}^{-1} \sigma(\mathbf{Wx}' )$ large) that have similar ground-truth features ($\text{cossim}(\mathbf{M} \mathbf{x},\mathbf{Mx}') \approx 1$), and encourages the normalized representations of pairs of points with dissimilar ground-truth features ($\text{cossim}(\mathbf{M} \mathbf{x},\mathbf{Mx}') \approx -1$) to also be dissimilar (i.e., making $\sigma(\mathbf{Wx} )^\top \boldsymbol{\Sigma}_{\mathbf{W}}^{-1} \sigma(\mathbf{Wx}' )$ small). The same intuitions hold if $\mathbb{E}_{f\sim \mathcal{T}}[f(\mathbf{x})] = \mu \neq 0$ for all $x$. So just as in the classification setting, $\tilde{\mathcal{L}}_{\text{reg}}(\mathbf{W})$ again behaves as a pseudo-contrastive loss that encourages recovering the ground-truth representation. Here, since $ \beta(\mathbf{x}, \mathbf{x}')$ is smooth, we may refer to $\tilde{\mathcal{L}}_{\text{reg}}$ as a ``soft'' contrastive loss.

Please see Tables \ref{tab:1} and \ref{tab:2} in Section \ref{sec:sims} for empirical results verifying this conclusion.

\newpage
\section{Numerical Simulations}\label{sec:sims}

\begin{figure}
  \centering
  \includegraphics[width=0.44\linewidth]{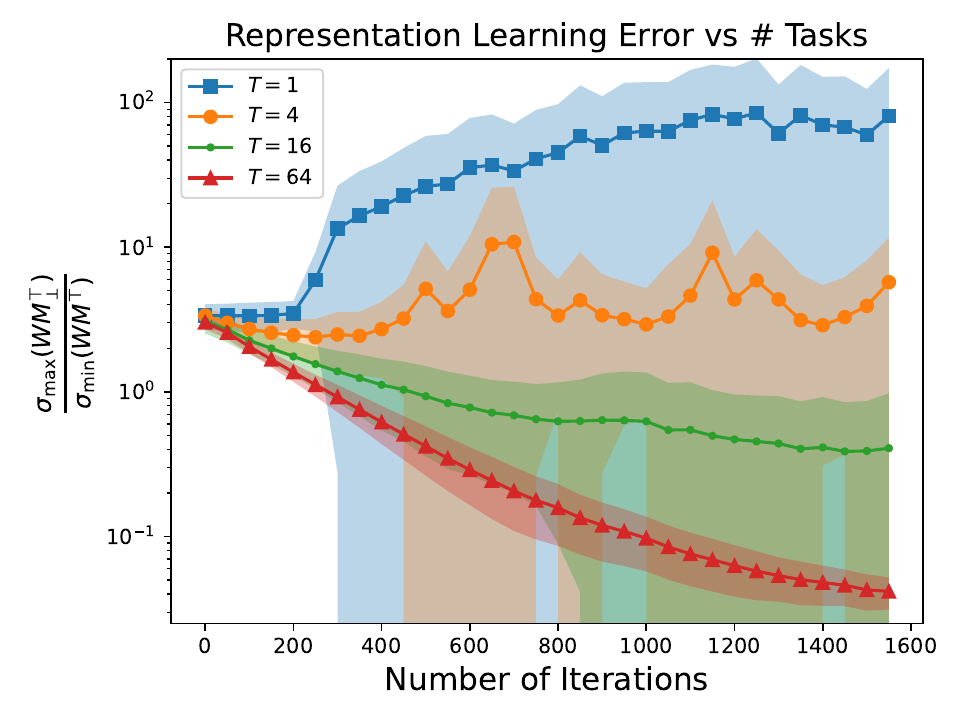}
  \includegraphics[width=0.48\linewidth]{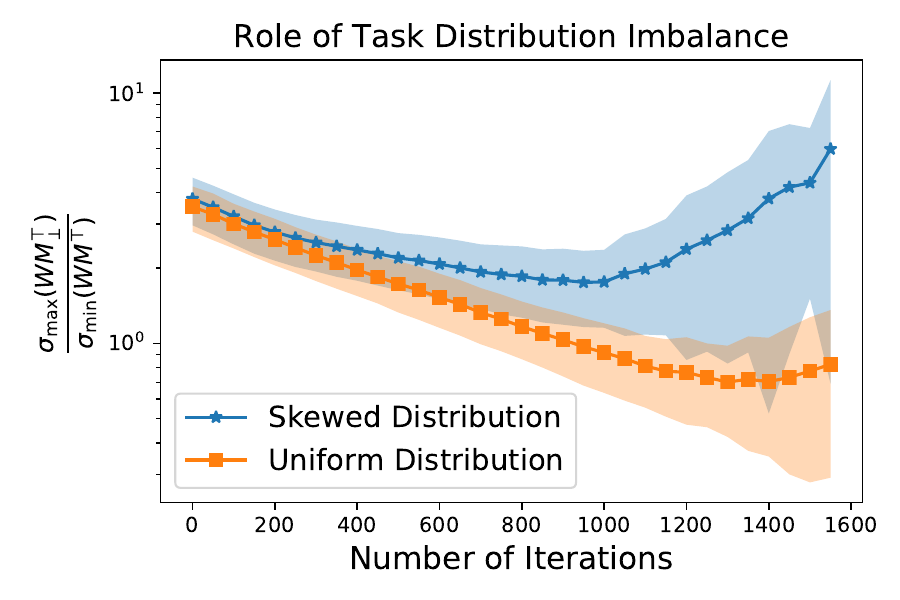}
  \caption{\textbf{Representation learning error.} (Left) Version of Figure \ref{fig:multi} showing the benefit of pretraining with additional tasks that includes the standard deviations (shaded regions) of each statistic around the plotted means over 10 trials. Note that $d=32, r=3$ and all cases use the same total number of samples. (Right) Representation learning error vs number of training iterations when tasks are sampled from either $\mathcal{T}_{\text{s.p.}}$ (`Uniform Distribution') or  a skewed distribution over the support of $\mathcal{T}_{\text{s.p.}}$ (`Skewed Distribution'). In this case $d=32, r=4$ and $T=32$.}\label{fig:std}
\end{figure}

In this section we verify our analysis with numerical simulations. 
We aim to both confirm that the alternating stochastic-gradient descent algorithm for multi-task pretraining that we study recovers the ground-truth representation and further explore the mechanisms by which it does so. To this end, all experiments are conducted on synthetic data generated according to the model described in Section \ref{sec:formulation}. To generate $\mathbf{M}$, we sample each of its elements independently from the standard normal distribution, then orthonormalize its rows via a QR decomposition. All experiments use $\mathcal{T}_{s.p.}$ as the distribution over task link functions. The pretraining algorithm is the pretraining algorithm described in Section \ref{sec:formulation} but repeated for many iterations. 


\textbf{Benefit of training with many tasks.} \textit{Figure \ref{fig:multi} in Section \ref{sec:intro}}  shows that increasing the number of pretraining tasks improves representation learning, even though all cases in this experiment use the same total number of samples. In particular, for each number of pretraining tasks $T$, gradients for each task are computed with $n_1=n_2=1024/T$ fresh samples per iteration, so the more tasks, the fewer samples per task.  Representation learning error is measured using the metric from Theorem \ref{thm:1}: $ \rho(\mathbf{M},\mathbf{W}):= \frac{\sigma_{1}(\mathbf{W}\mathbf{M}_{\perp}^\top)}{\sigma_{r}(\mathbf{WM}^\top)}$. This  metric captures the extent to which the row space of $\mathbf{W}$ covers that of $\mathbf{M}$ (measured by $\sigma_{r}(\mathbf{WM}^\top)$) {\em and} the extent to which the row space of $\mathbf{W}$ lies only in that of $\mathbf{M}$ (measured by $\sigma_{1}(\mathbf{W}\mathbf{M}_{\perp}^\top)$). Figure \ref{fig:multi} shows the mean values of $\rho(\mathbf{W}^t,\mathbf{M})$ across 10 independent random trials, including  independently sampled sets of pretraining tasks; \textit{Figure \ref{fig:std}(Left)} plots the same results plus shaded regions indicating $\pm$ one standard deviation across the 10 trials. Here $d=32,$ $r=3$, and $m=16$.

\textbf{Role of task distribution diversity.} Figure \ref{fig:std}(Right) motivates Assumption \ref{assump:diversity} by demonstrating that the quality of the learned representation degrades with the diversity, or balancedness, of the task distribution. Here we use $d=32$, $r=4$, $T=32$, $m=16$ and $n_1=n_2=16$ (so larger $r$ and fewer total samples than in \textit{Figure \ref{fig:std}(Left)}. `Uniform Distribution' means the task link functions are sampled from $\mathcal{T}_{\text{s.p.}}$ as usual, and `Skewed Distribution' means the link functions are sampled from a non-uniform distribution over the set sparse parity tasks on $r$ inputs as follows: (1) sample $|\mathcal{S}_i|$ from $\{0,1,...,r\}$, weighted by the number of sparse parity tasks on support sets of that size, i.e. proportionally to the binomial coefficients (same as in sampling from $\mathcal{T}_{\text{s.p.}}$), (2) sample $|\mathcal{S}_i|$ elements without replacement from $\{1,\dots,r\}$ weighted by $[0.3, 0.3, 0.3, 0.1]$, noting that $r=4$ (this step differs from sampling from $\mathcal{T}_{\text{s.p.}}$, which would apply uniform weights to the $r$ features). {\em We see that sampling tasks from the uniform distribution leads to much smaller representation learning error than sampling from the skewed distribution, since the skewed distribution de-emphasizes one feature ground-truth feature.} Again the experiment is repeated over 10 independent random trials and means and standard deviations are shown. 

\begin{figure}
  \centering
  \includegraphics[width=0.44\linewidth]{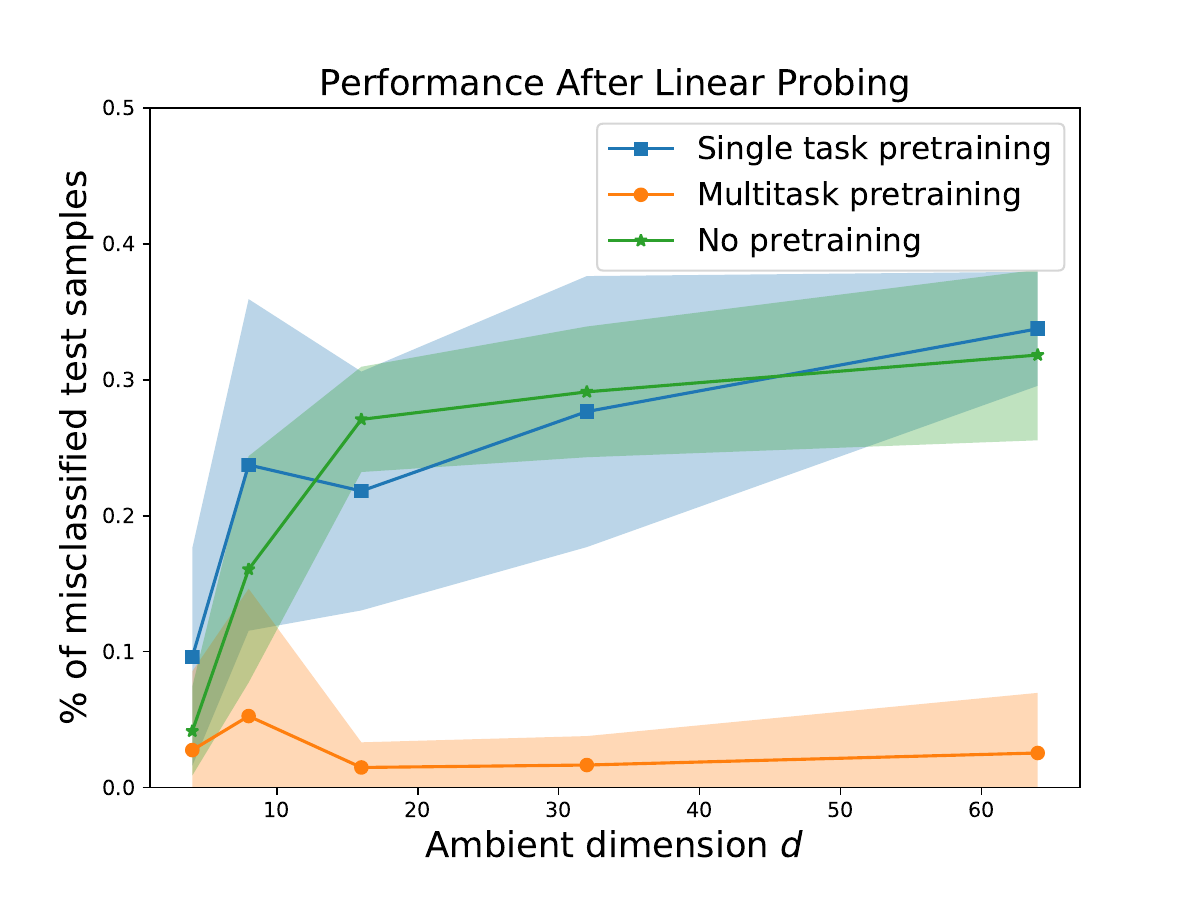}
  \includegraphics[width=0.44\linewidth]{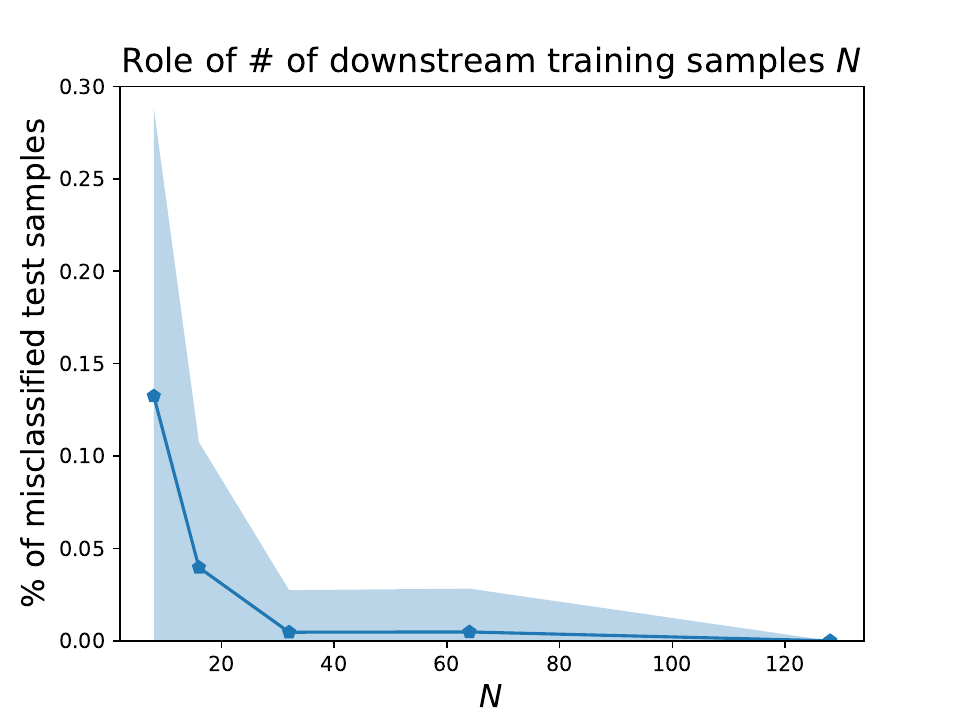}
  \caption{\textbf{Downstream task performance.} (Left) Downstream task performance for multi-task pretrained, single task, and random (`No pretrained') representations $\mathbf{W}$ with varying dimension $d$. Unlike single task pretrained and the non-pretrained representations, the downstream performance of representations trained with multiple tasks does not degrade with $d$. Note  that for multi-task, $T=16+d$ and $n_1=n_2=16$ and for single task, $n_1=n_2=16\times(16+d)$, and $r=4$ and $m=16$ in all cases. (Right) Downstream task performance for multi-task-trained representation with $T=32, d=32,r=3,n_1=n_2=16$ and $m=16$, and with $\hat{m}=32$ for downstream linear probing, with varying number of downstream training samples $N$.
  }\label{fig:gen}
\end{figure}

\textbf{Generalization to downstream tasks.} We also investigate whether learning an approximation of the ground-truth representation leads to strong downstream performance. To evaluate the downstream performance of a learned representation, we follow the same procedure from Section \ref{sec:formulation} by first randomly sampling $\hat{m}$ neuron weights $\hat{\mathbf{W}}$ from the multivariate standard normal distribution and $\mathbf{b}$ and $\hat{\mathbf{b}}$ from the uniform distribution on $[-10^{-3}, 10^{-3}]$. Then, we run gradient descent on the regularized empirical hinge loss on a fixed set of $N$ samples to learn the last layer head, i.e. linear probing. We run this gradient descent with step size $\eta=0.1$ and $\ell_2-$regularizer $\hat{\lambda}_{\mathbf{a}}=0.01$. After this linear probing on $N$ samples, we evaluate the performance of the returned classifier on a distinct set of $1000$ test samples for each task.

\textit{Figure  \ref{fig:gen}(Left)} plots the number of misclassified test samples after this linear probing with $N=32$ training samples averaged across 10 randomly drawn tasks from $\mathcal{T}_{\text{s.p.}}(\mathbf{M})$ with varying $d$. In particular, for each value of $d$, we randomly generate an $\mathbf{M}$, execute multi-task and single task pretraining on task(s) drawn from $\mathcal{T}_{\text{s.p.}}(\mathbf{M})$ to learn  $\mathbf{W}$, then execute   linear probing with $N=32$ samples for $100$ iterations on the head of  the random ReLU network with $\hat{m}=128$ second-layer neurons generated from these trained $\mathbf{W}$'s, as well as a randomly generated $\mathbf{W}$ (`No pretraining'), on each of 10 new downstream tasks sampled from $\mathcal{T}_{\text{s.p.}}$,  and save the average percentage of misclassified samples. We repeat this process end-to-end 10 times, and plot mean and standard deviations across these 10 trials. Again we use $m=16$ neurons and execute pretraining for 1600 rounds.
To mitigate the effect of representation learning error, we scale $T$ with $d$ for multi-task pretraining, specifically $T=16+d$. For fair comparison with single task pretraining, we scale $n_1$ and $n_2$ with $d$ for the single task case, specifically $n_1=n_2=16\times(16+d)$ for single task, whereas $n_1=n_2=16$ for multi-task. 
{\em While the percentage of misclassified samples grows with $d$ for single and no pretraining, it does not for multi-task pretraining.} This confirms that multi-task pretraining reduces  the effective dimension of the downstream task from $d$ to $r$, unlike single task pretraining, which effectively confers no downstream benefit as it performs similarly to no pretraining.

\textit{Figure \ref{fig:gen}(Right)} explores the role
of 
$N$ in downstream performance. Here we pretrain a single $\mathbf{W}$ on $T=32$ tasks from $\mathcal{T}_{\text{s.p.}}$ for 1600 rounds with $d=32,r=3,n_1=n_2=16,$ and $m=16$. Then we execute linear probing for 200 iterations on the random three-layer ReLU network with first-layer weights $\mathbf{W}$. We fix either $\hat{m}=32$ 
and vary $N$.
The results shown are the mean and standard deviation of the percentage of misclassified test samples across 25 tasks drawn from $\mathcal{T}_{\text{s.p.}}$, with 5000 test samples used per task. { The classification accuracy improves with $N$, as predicted by Theorem \ref{thm:downstream}, and { even reaches perfect test classification accuracy} (when $N=128$).
}




\begin{figure*}[t]
\begin{center}
\centerline{\includegraphics[width=0.90\textwidth]{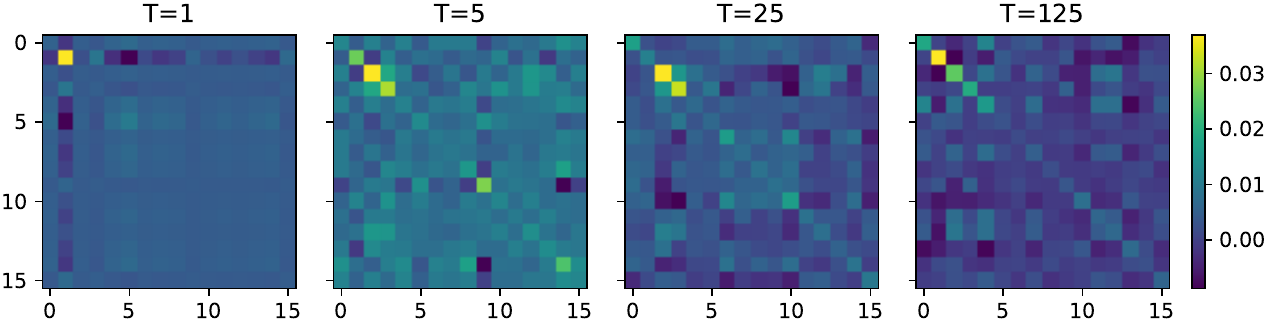}
}
\vspace{0mm}
\caption{\textbf{More tasks isolate important features.} From the discussion in Section \ref{sec:sketch}, the loss induced by multi-task training with task-specific heads as a function of the representation  is approximately $\mathcal{L}(\mathbf{W})\approx -\mathbb{E}_{\mathbf{x},\mathbf{x}'}[\beta(\mathbf{x},\mathbf{x}') \sigma(\mathbf{Wx})^\top \sigma(\mathbf{Wx}') ]$, where $
\beta(\mathbf{x},\mathbf{x}')=\mathbb{E}_i[f_i(\mathbf{x})f_i(\mathbf{x}')]$ is the average product of the labels of two points across tasks.  This loss is pseudo-contrastive in  that it encourages representations of two points to be similar if and only if they share the same label on most tasks ($\beta(\mathbf{x},\mathbf{x}')\approx 1$), which is equivalent to saying that they share important features. Here we consider the gradient of $\mathcal{L}(\mathbf{W})$ with respect to one neuron weight $\mathbf{w}_j$. The gradient takes the form $-\mathbf{A}\mathbf{w}_j$, and we plot finite-task and finite-sample approximations of  $\mathbf{A}$. We set $d=16$ and the ground-truth features to be the first $r=4$ coordinates of the data, i.e. $\mathbf{M} = [\mathbf{I}_4, \mathbf{0}_{4\times 12}]$. Roughly speaking, if the finite-task approximation of $\beta(\mathbf{x},\mathbf{x}')$, namely $\frac{1}{T}\sum_{i=1}^T f_i(\mathbf{x})f_i(\mathbf{x}')$, serves as a proxy for whether $\mathbf{x}$ and $\mathbf{x}'$ share ground-truth features, as does $\mathbb{E}_i[f_i(\mathbf{x})f_i(\mathbf{x}')]$, then the terms with $\mathbf{x}$ and $\mathbf{x}'$ having the same ground-truth features will dominate the loss, and these features themselves will dominate $\mathbf{A}$. {\em The above plots confirm this;  as the number of tasks $T$ increases and $\frac{1}{T}\sum_{i=1}^T f_i(\mathbf{x})f_i(\mathbf{x}')$ approaches $\mathbb{E}_i[f_i(\mathbf{x})f_i(\mathbf{x}')]$, $\mathbf{A}$ becomes dominated by its top 4-by-4 submatrix, i.e. $\mathbf{A} \approx c \mathbf{M}^\top \mathbf{M}$ for a scalar $c$.} So, $\mathbf{A}$ behaves more like a projection onto the row space of $\mathbf{M},$ as desired.
}\label{fig:num_tasks}
\end{center}
\end{figure*}

\textbf{The role of head updates.} Next, we explore {\em why} multi-tasking leads to better feature learning. We are motivated by our discussion in Section \ref{sec:sketch} regarding the similarity of the multi-task loss induced from updating the task-specific heads  to a constrastive loss \cite{chen2020simple}, which encourages representations that align points sharing semantic meanings and dis-align arbitrary points. Recall that in the population setting, the multi-task loss is approximately of the form of $-  \E_{\mathbf{x}, \mathbf{x}' }  \left[  \beta({\mathbf{x},\mathbf{x}'}) \sigma(\mathbf{W}^0 \mathbf{x}')^\top \sigma( \mathbf{W}^0\mathbf{x}) \right]$, where $\beta({\mathbf{x},\mathbf{x}'})= \mathbb{E}_i[f_i(\mathbf{x})f_i(\mathbf{x}')]$ is a scalar that either encourages the representation to align $\mathbf{x}$ and   $\mathbf{x}'$ (if  $\beta({\mathbf{x},\mathbf{x}'})\approx 1$) or not (if  $\beta({\mathbf{x},\mathbf{x}'})\ll 1$). The intuition is that  $\beta({\mathbf{x},\mathbf{x}'})\approx 1$ if and only if $\mathbf{x}$ and $\mathbf{x}'$ share the same label on most tasks and thereby share important features.  The  gradient of $ \E_{\mathbf{x}, \mathbf{x}' }  \left[  \beta({\mathbf{x},\mathbf{x}'}) \sigma(\mathbf{W}^0 \mathbf{x}')^\top \sigma( \mathbf{W}^0\mathbf{x}) \right]$ with respect to one neuron weight $\mathbf{w}_j$ is of the form $\mathbf{A}\mathbf{w}_j$ 
where 
\begin{align}
    \mathbf{A} = \mathbb{E}_{\mathbf{x},\mathbf{x}'}[ \beta(\mathbf{x},\mathbf{x}') \sigma'(\mathbf{w}_j^\top \mathbf{x})  \sigma'(\mathbf{w}_j^\top \mathbf{x}') \mathbf{x}(\mathbf{x}')^\top]\nonumber.
\end{align}
See Appendix \ref{app:thrm1} for a rigorous derivation. In \textit{Figure \ref{fig:num_tasks}} we plot finite-sample estimates of $\mathbf{A}$ with varying numbers of tasks $T$ drawn from $\mathcal{T}_{\text{s.p.}}(\mathbf{M})$, where here $M = [\mathbf{I}_4, \mathbf{0}_{4\times 12}]$ for ease of visualization. We use $d=16, r=4$ and $n_1\!=\!n_2\!=\!100$. We repeated each computation 10 times for each value $T$, each with independent draws of $\mathbf{w}_j$, $T$ tasks, and $n_1\!+\!n_2$ samples per task, and plotted the matrix $\mathbf{A}$ matrix that achieved the smallest value of $\rho(\mathbf{A}, \mathbf{M}) = \frac{\sigma_{1}(\mathbf{A}\mathbf{M}_{\perp}^\top)}{\sigma_{r}(\mathbf{AM}^\top)}$ among these 10 trials.
\textit{\em Figure \ref{fig:num_tasks} shows that as the number of tasks increases, the finite-task approximation of $\beta({\mathbf{x},\mathbf{x}'})$ increasingly acts like an indicator for whether $\mathbf{x}$ and $\mathbf{x}'$ share the same ground-truth features, evidenced by $\mathbf{A}$ approaching $\mathbf{M}^\top \mathbf{M} = [\mathbf{I}_4,\mathbf{0}_{4\times 12};
\mathbf{0}_{12\times 4},\mathbf{0}_{12\times 12}]$ }. Thus, $\mathbf{A}$ acts increasingly like a projection onto the row space of $\mathbf{M}$ as $T$ increases.

To further assess  the importance of adapting the heads to each task, \textit{Figure \ref{fig:neuronss}(Left)} compares the representation learning performance of the multi-task pretraining algorithm we study along with a modified version that learns only one shared head across all tasks. In particular, $\mathbf{W}$ and $\mathbf{a}$ are updated simultaneously on each iteration by averaging the task-specific gradients. {\em Since this algorithm does not involve task-specific head adaptation prior to the update of the representation, it does not induce a feature-learning-encouraging contrastive loss, and therefore does not lead to learning the ground-truth features}. In this case $d=16,r=2,m=8,n_1=n_2=16$, and $\nu_{\mathbf{w}}=0.001$ (note increasing $\nu_{\mathbf{w}}$ does not improve the performance of single-head training).

Finally, \textit{Figure \ref{fig:neuronss}(Center)} plots the dynamics of $\mathbf{Mw}_j$ for four neuron weights $\mathbf{w}_j$ during multi-task pretraining with $T=25$, $r=2$, $d=8$, $m=4$ (and task-specific heads). \textit{The projections $\mathbf{Mw}_j$ fan outwards from the origin and remain roughly isotropic in the row space of $\mathbf{M}$.} Conversely, \textit{Figure \ref{fig:neuronss}(Right) shows that the projections of $\mathbf{w}_j$ onto the first two rows of $\mathbf{M}_{\perp}$ contract towards the origin for each of the four neurons,} as desired.


\begin{figure}
  \centering
   \includegraphics[width=0.37\linewidth]{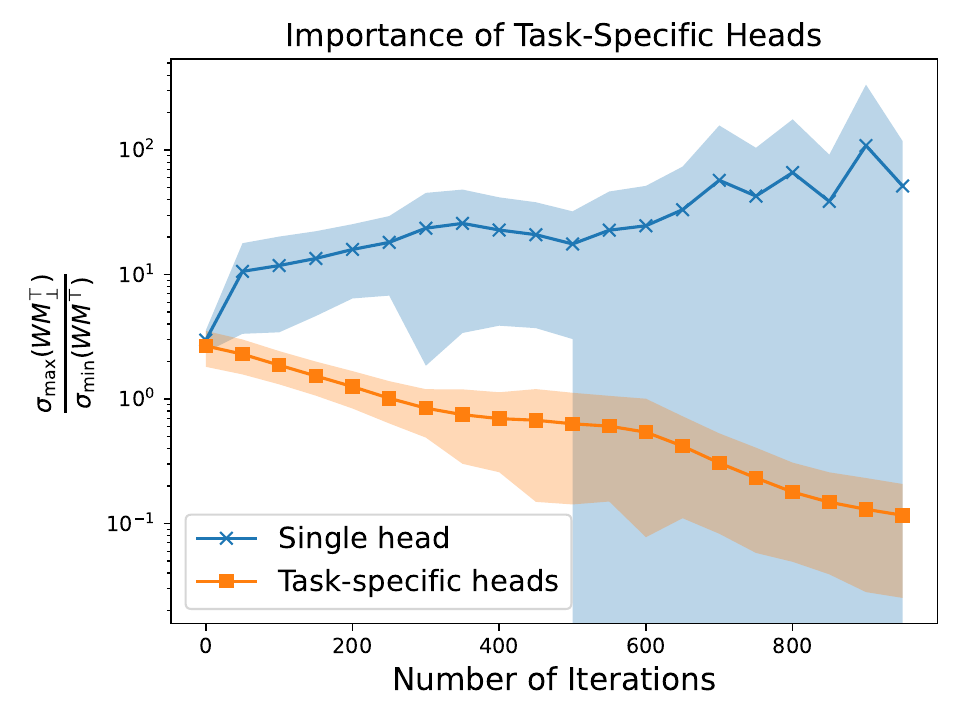}
   \hspace{2mm}
   \includegraphics[width=0.29\linewidth]{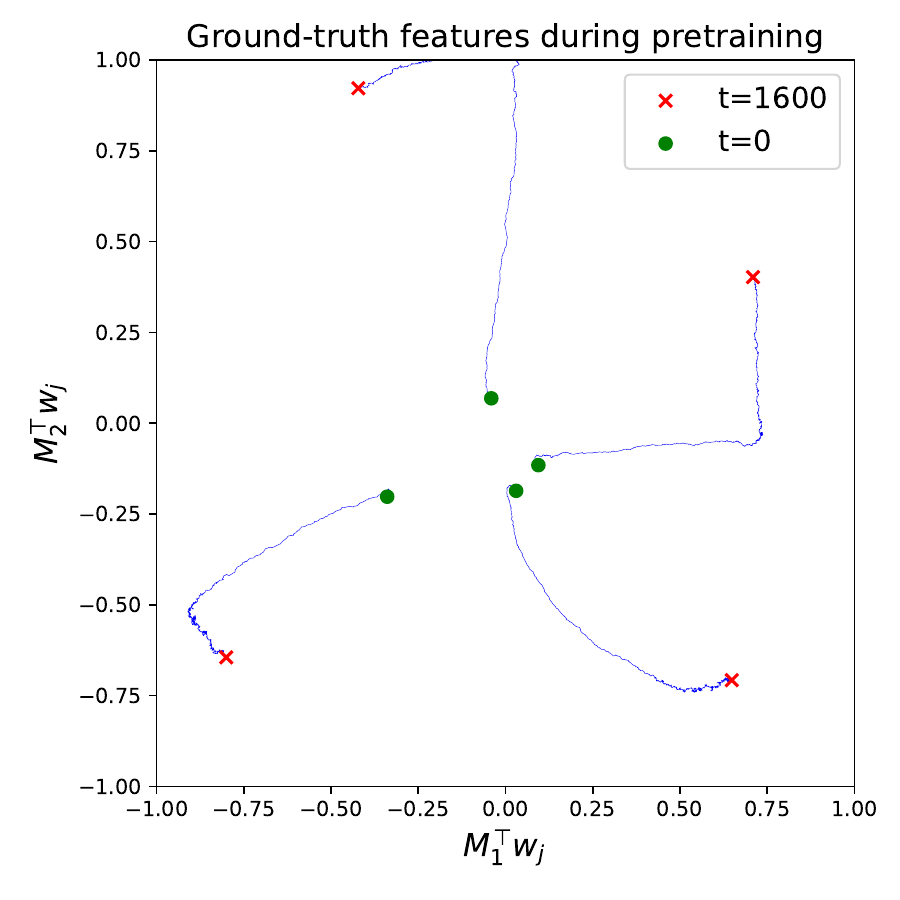}
   \includegraphics[width=0.29\linewidth]{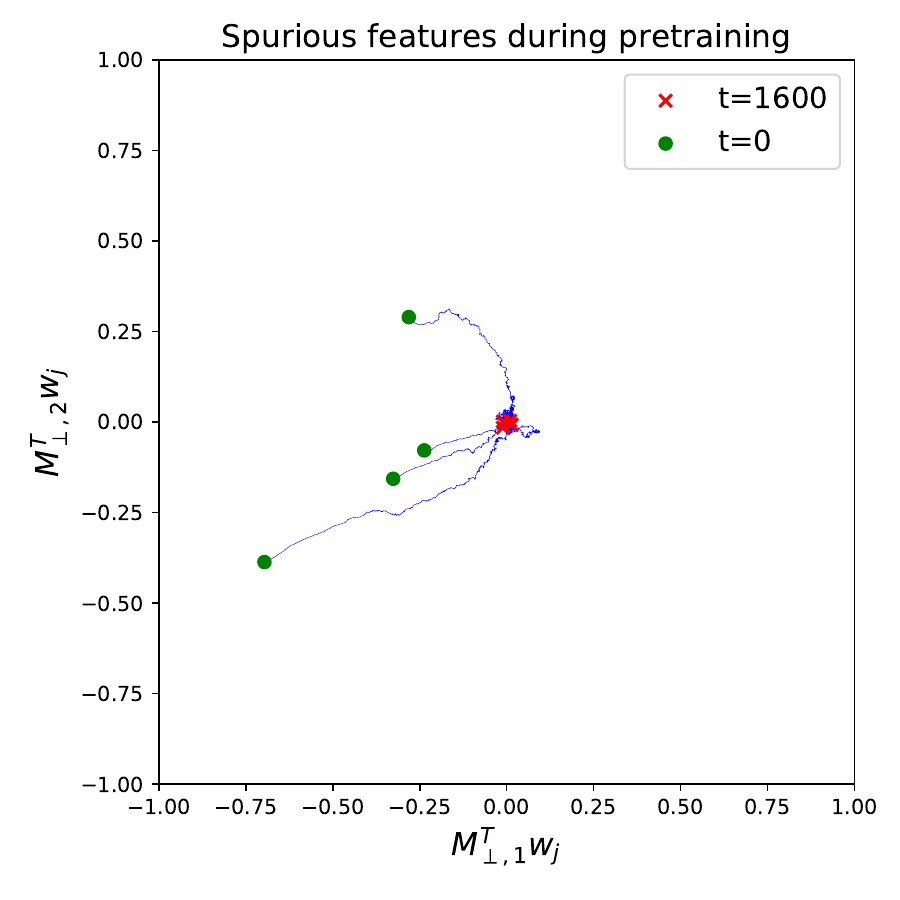}
  \caption{(Left) Training with a single head, i.e. $\mathbf{a}_1=\mathbf{a}_2=...=\mathbf{a}_T=\mathbf{a}$, fails to recover the ground-truth representation, as this does not induce an appropriate contrastive loss. (Center) During multi-task pretraining with task-specific heads, the projection of four neurons onto the $r\!=\!2$-dimensional ground-truth subspace fan outwards from the origin such that they remain large and isotropic in this space, whereas  (Right) their projections onto the spurious subspace contract towards the origin.}\label{fig:neuronss}
\end{figure}

\textbf{Extension to relaxed version of Assumption \ref{assump:diversity} and regression.} We empirically verify that the special cases of the new, weaker condition discussed above result in learning the ground-truth features. We also evaluate whether optimizing the loss derived in the regression setting also leads to recovering the features. All cases consider $T=\infty$ for simplicity. In particular, the settings we consider are:

\begin{enumerate}
    \item 
The standard hinge-loss classification setting with \begin{eqnarray}\beta(\mathbf{x},\mathbf{x}')&=&\begin{cases} 1 & \text{if} \quad \text{sign}(\mathbf{Mx}) = \text{sign}(\mathbf{Mx}')\\ \delta & \text{o/w}\\ \end{cases},\end{eqnarray} for various choices of $\delta$.

\item The standard hinge-loss classification setting with $$\beta(\mathbf{x},\mathbf{x}')=1 - \frac{2}{\pi} \arccos\left(\text{cossim}(\mathbf{Mx},\mathbf{Mx}') \right).$$

\item The regression setting from Section \ref{app:regression} with $$\beta(\mathbf{x},\mathbf{x}')=\exp\left( \text{cossim}(\mathbf{Mx},\mathbf{Mx}') \right) - \exp\left( -\text{cossim}(\mathbf{Mx},\mathbf{Mx}')\right).$$

\end{enumerate}

To focus on the role of $\beta(\mathbf{x},\mathbf{x}')$, we run SGD with a batch size of 10 $(\mathbf{x},\mathbf{x}')$ pairs on the loss $\tilde{\mathcal{L}}$ for the classification cases, and SGD with a batch size of 10 $(\mathbf{x},\mathbf{x}')$ pairs on the loss $\tilde{\mathcal{L}}_{\text{reg}}$ for the regression case. We set $d=10$, $r=2$, $\mathbf{M} = [\mathbf{I}_2, \mathbf{0}_{2\times 8}]$, and $m=6$. In the regression case, we used 10 i.i.d. samples each round to approximate $\boldsymbol{\Sigma}_{\mathbf{W}}$. We do not differentiate through functions of $\mathbf{W}$ that arise from solving for the optimal $\mathbf{a}_i^*(\mathbf{W})$'s. We use the same hyperparameters in all cases (learning rate $=0.005$, regularization parameter $=0.1$). We evaluate $\frac{\sigma_1(\mathbf{W}^t(\mathbf{I}_d-\mathbf{M}^\top \mathbf{M}))}{\sigma_r(\mathbf{W}^t \mathbf{M}^\top \mathbf{M})}$ and $\frac{\sigma_1(\mathbf{W}^t \mathbf{M}^\top \mathbf{M})}{\sigma_r(\mathbf{W}^t \mathbf{M}^\top \mathbf{M})}$ every 5000 rounds of training over 20000 rounds. All values are means plus or minus standard deviation over 5 independent random trials.

We can see that in all cases, $\mathbf{W}^t$ becomes approximately a rank-$r$ matrix whose row space aligns with the row space of $\mathbf{M}$, and whose projection onto the row space of $\mathbf{M}$ is well-conditioned, confirming recovery of the ground-truth features. As expected, the convergence is slower for larger values of $\delta$ in Case 1, since the loss $\tilde{\mathcal{L}}(\mathbf{W})$ puts lets of an incentive on increasing the representation similarity of positive pairs ($(\mathbf{x},\mathbf{x}'): \text{sign}(\mathbf{Mx})= \text{sign}(\mathbf{Mx}')$) relative to the similarity of negative pairs. Nevertheless, all values of $\delta$ result in a representation tending towards the ground-truth.

{\em We can see that in all cases, $\mathbf{W}^t$ becomes approximately a rank-$r$ matrix whose row space aligns with the row space of $\mathbf{M}$, and whose projection onto the row space of $\mathbf{M}$ is well-conditioned, confirming recovery of the ground-truth features}. As expected, the convergence is slower for larger values of $\delta$ in Case 1, since the loss $\tilde{\mathcal{L}}(\mathbf{W})$ puts lets of an incentive on increasing the representation similarity of positive pairs ($(\mathbf{x},\mathbf{x}'): \text{sign}(\mathbf{Mx})= \text{sign}(\mathbf{Mx}')$) relative to the similarity of negative pairs. Nevertheless, all values of $\delta$ result in a representation tending towards the ground-truth.



\begin{table}[t] 
\caption{Subspace learning error ($\frac{\sigma_1(\mathbf{W}^t(\mathbf{I}_d-\mathbf{M}^\top \mathbf{M}))}{\sigma_r(\mathbf{W}^t \mathbf{M}^\top \mathbf{M})}$) vs number of training iterations $t$. All values are means plus or minus standard deviation over 5 independent random trials.}
\label{tab:1}
\vskip 0.15in
\begin{center}
\begin{small}
\begin{sc}
\begin{tabular}{lccccc}
\toprule
              & $t=0$           & $t=200$         & $t=400$         & $t=600$         & $t=800$          \\
\midrule
(1) $\delta=0$   & $2.29\pm 0.72$  & $0.784\pm 0.26$   & $0.242\pm 0.089$   & $0.0947\pm 0.029$ & $0.796\pm 0.025$  \\
 (1) $\delta=0.1$ & $2.29 \pm 0.72$ & $0.879 \pm 0.29$ & $0.308 \pm 0.11$   & $0.116\pm 0.031$  & $0.0825\pm 0.018$ \\
 (1) $\delta=0.5$ & $2.29 \pm 0.72$ & $1.35 \pm 0.38$   & $0.807 \pm 0.23$   & $0.450 \pm 0.12$  & $0.250\pm 0.60$   \\
(2)  Linear Tasks           & $2.29 \pm 0.72$ & $0.323\pm 0.11$   & $0.0761 \pm 0.021$ & $0.0665\pm 0.031$ & $0.0581\pm 0.012$ \\
(3)  Regression               & $2.95 \pm 0.89$ & $0.382 \pm 0.23$  & $0.268 \pm 0.080$  & $0.218 \pm 0.038$ & $0.421 \pm 0.266$ \\
\bottomrule
\end{tabular}
\end{sc}
\end{small}
\end{center}
\vskip -0.1in
\end{table}



\begin{table}[t] 
\caption{Condition number of $\mathbf{W}^t$ in ground-truth subspace ($\frac{\sigma_1(\mathbf{W}^t \mathbf{M}^\top \mathbf{M})}{\sigma_r(\mathbf{W}^t \mathbf{M}^\top \mathbf{M})}$) vs number of training iterations $t$. All values are means plus or minus standard deviation over 5 independent random trials. The closer the condition number is to 1, the better.}
\label{tab:2}
\vskip 0.15in
\begin{center}
\begin{small}
\begin{sc}
\begin{tabular}{lccccc}
\toprule
              & $t=0$           & $t=200$         & $t=400$         & $t=600$         & $t=800$          \\
\midrule
(1) $\delta=0$   & $1.45 \pm 0.28$ & $1.37 \pm 0.23$   & $1.34 \pm 0.22$  & $1.30 \pm 0.25$  & $1.29 \pm 0.33$ \\
(1) $\delta=0.1$ & $1.45 \pm 0.28$ & $1.38  \pm 0.23$  & $1.35 \pm 0.30$  & $1.30 \pm 0.20$  & $1.24 \pm 0.22$\\
(1) $\delta=0.5$ & $1.45 \pm 0.28$ & $1.40 \pm 0.21$   & $1.41 \pm 0.14$  & $1.34 \pm 0.14$  & $1.25 \pm 0.12$ \\
(2)  Linear Tasks                 & $1.45 \pm 0.28$ & $1.43 \pm 0.25$   & $1.44 \pm 0.24$  & $1.45 \pm 0.29$  & $1.45 \pm 0.29$        \\
(3)  Regression                 & $1.83 \pm 0.41$ & $1.05  \pm 0.041$ & $1.07 \pm 0.031$ & $1.03 \pm 0.011$ & $1.11 \pm 0.11$ \\
\bottomrule
\end{tabular}
\end{sc}
\end{small}
\end{center}
\vskip -0.1in
\end{table}

\textbf{Separation between training with a fully-informative single task and multi-tasking.} Finally, we empirically verify the improved sample complexity of multi-tasking vs single-tasking in the same setting as Figure \ref{fig:multi} (whose full version is Figure \ref{fig:std} in this Appendix \ref{sec:sims}) with $T\in \{1,16\}$, but always using the full sparse parity task as the single training task in the $T=1$ case, unlike the figure in in the paper, in which tasks were drawn from $\mathcal{T}_{\text{s.p.}}$ in all cases. In each case we vary $n = n_1 = n_2$, where $n_1$ is the number of samples used per batch to compute the gradient with respect to the head $a$ and $n_2$ is the number of samples used per batch compute the gradients with respect the neuron weights $\mathbf{W}$ and bias $b$. We use $d=32,r=3$ and $m=16$ neurons. All cases use Gaussian initialization. We alternate between updates of the head and representation, as we did not observe any significant change in performance by running simultaneous updates of the head and representation for the single-task case. Learning rates and regularization parameters were tuned separately for $T=1$ and $T=16$, resulting in $(\eta=0.01, \lambda_{\mathbf{w}}=0.05, \lambda_{\mathbf{a}}=0.5)$ for $T=1$ and $(\eta=0.5, \lambda_{\mathbf{w}}=0.05, \lambda_{\mathbf{a}}=0.5)$ for $T=16$. We run 5 independent random trials for 800 iterations and plot means plus or minus standard deviations. As in Tables \ref{tab:1} and \ref{tab:2}, we evaluate the subspace learning error $\frac{\sigma_1(\mathbf{W}^t(\mathbf{I}_d-\mathbf{M}^{\top} \mathbf{M}))}{\sigma_r(\mathbf{W}^t \mathbf{M}^\top \mathbf{M})}$ in Table \ref{tab:3} and the condition number of the learned representation in the ground-truth subspace $\frac{\sigma_1(\mathbf{W}^t \mathbf{M}^\top \mathbf{M})}{\sigma_r(\mathbf{W}^t \mathbf{M}^\top \mathbf{M})}$ in Table \ref{tab:4}.
{\em We can see that for all $n$, multi-tasking leads to a representation that is much closer to a projection onto the ground-truth subspace, achieving $10-100\times$ smaller subspace learning error and approximately $5\times$ smaller condition number in the ground-truth subspace than single-task training on the full parity task.}


\begin{table}[t]
\caption{Subspace learning error ($\frac{\sigma_1(\mathbf{W}^t(\mathbf{I}_d-\mathbf{M}^\top \mathbf{M}))}{\sigma_r(\mathbf{W}^t \mathbf{M}^\top \mathbf{M})}$) vs number of training iterations $t$. All values are means plus or minus standard deviation over 5 independent random trials.}
 \label{tab:3}
\vskip 0.15in
\begin{center}
\begin{small}
\begin{sc}
\begin{tabular}{lccccc}
\toprule
              & $t=0$           & $t=200$         & $t=400$         & $t=600$         & $t=800$          \\
\midrule
$T=1,n=8$    & $2.89 \pm 0.36$ & $2.86 \pm 0.36$ & $2.80 \pm 0.39$ & $2.81 \pm 0.38$ & $2.86 \pm 0.37$  \\
$T=1,n=64$   & $2.89 \pm 0.36$ & $2.87 \pm 0.37$ & $2.81 \pm 0.37$ & $2.80 \pm 0.37$ & $2.84 \pm 0.34$\\
$T=1,n=512$  & $2.89 \pm 0.36$ & $2.86 \pm 0.36$ & $2.80 \pm 0.37$ & $2.79 \pm 0.37$ & $2.83 \pm 0.35$\\
$T=16,n=8$   & $2.89 \pm 0.36$ & $0.50 \pm 0.33$ & $0.27 \pm 0.03$ & $0.26 \pm 0.02$ & $0.27 \pm 0.03$      \\
$T=16,n=64$  & $2.89 \pm 0.36$ & $0.23 \pm 0.14$ & $0.08 \pm 0.01$ & $0.08 \pm 0.01$ & $0.08 \pm 0.01$  \\
$T=16,n=512$ & $2.89 \pm 0.36$ & $0.11 \pm 0.07$ & $0.03 \pm 0.01$ & $0.03 \pm 0.01$ & $0.03 \pm 0.01$ \\
\bottomrule
\end{tabular}
\end{sc}
\end{small}
\end{center}
\vskip -0.1in
\end{table}


\begin{table}[t] 
\caption{Condition number of $\mathbf{W}^t$ in ground-truth subspace ($\frac{\sigma_1(\mathbf{W}^t \mathbf{M}^\top \mathbf{M})}{\sigma_r(\mathbf{W}^t \mathbf{M}^\top \mathbf{M})}$) vs number of training iterations $t$. All values are means plus or minus standard deviation over 5 independent random trials. The closer the condition number is to 1, the better.}
\label{tab:4}
\vskip 0.15in
\begin{center}
\begin{small}
\begin{sc}
\begin{tabular}{lccccc}
\toprule
              & $t=0$           & $t=200$         & $t=400$         & $t=600$         & $t=800$          \\
\midrule
$T=1,n=8$    & $1.72 \pm 0.25$ & $1.91 \pm 0.13$ & $2.60 \pm 0.36$ & $3.94 \pm 1.08$ & $6.72 \pm 2.81$ \\
$T=1,n=64$   & $1.72 \pm 0.25$ & $1.92 \pm 0.13$ & $2.59 \pm 0.35$ & $3.98 \pm 1.08$ & $6.70 \pm 2.61$ \\
$T=1,n=512$  & $1.72 \pm 0.25$ & $1.92 \pm 0.13$ & $2.57 \pm 0.36$ & $3.93 \pm 1.08$ & $6.61 \pm 2.63$ \\
$T=16,n=8$   & $1.72 \pm 0.25$ & $2.63 \pm 1.89$ & $1.31 \pm 0.22$ & $1.27 \pm 0.18$ & $1.27 \pm 0.19$       \\
$T=16,n=64$  & $1.72 \pm 0.25$ & $3.38 \pm 2.68$ & $1.31 \pm 0.28$ & $1.25 \pm 0.21$ & $1.23 \pm 0.18$\\
$T=16,n=512$ & $1.72 \pm 0.25$ & $3.38 \pm 2.69$ & $1.32 \pm 0.30$ & $1.26 \pm 0.23$ & $1.24 \pm 0.20$\\
\bottomrule
\end{tabular}
\end{sc}
\end{small}
\end{center}
\vskip -0.1in
\end{table}

\textbf{Additional hyperparameters.}
Unless otherwise noted, we used  $\lambda_{\mathbf{w}} = 0.05$, $\lambda_{\mathbf{a}}=0.5$ and $\eta =0.1$ (learning rate for both $\mathbf{a}_i$ and $\mathbf{W}$).
after tuning each parameter in the set $ \{0.01,0.05,0.1,0.5,1\}$, separately for single task and multi-task cases, unless $r=4$. We tuned $\nu_{\mathbf{w}} \in \{0.001,0.01,0.1,1\}$, and used $\nu_{\mathbf{w}} = 0.01$ for $r\leq 3$, unless otherwise noted. For $r=4$, we found that setting $\lambda_{\mathbf{w}}=0.1$ and $\nu_{\mathbf{w}}=0.001$ improved performance, but did not see improvement by changing the other parameters, so kept them the same. We used a smaller learning rate of 0.001 for the bias in all cases, although we reset the bias randomly before downstream evaluation.
\end{document}